%% file: arxiv.tex
\documentclass{article}

\usepackage[a4paper,margin=1in]{geometry}

\usepackage{amsmath, amssymb, amsthm}
\usepackage{mathtools}
\usepackage{algorithm}
\usepackage{algpseudocode}
\usepackage{tikz}

\usepackage{mathtools}
\usepackage{graphicx}
\usepackage{caption}
\usepackage{booktabs}
\usepackage{wrapfig}
\usepackage{subcaption}
\usepackage{tabularx}
\usepackage{siunitx} 
\usepackage{multirow}
\usepackage{makecell}
\usepackage{url}
\usepackage{xcolor}

\newcolumntype{Y}{>{\centering\arraybackslash}X}
\usepackage{natbib}
 \bibpunct[, ]{(}{)}{,}{a}{}{,}%
\newtheorem{theorem}{Theorem}
\newtheorem{proposition}{Proposition}
\newtheorem{lemma}{Lemma}

\theoremstyle{definition}

\theoremstyle{remark}
\newtheorem{remark}{Remark}
\newtheorem{assumption}{Assumption}

\title{End-to-End Fairness Optimization with Fair Decision-Focused Learning}
\author{
  Yu Wang \\
  School of Business \\
  Stevens Institute of Technology \\
 \texttt{ywang88@stevens.edu} \\
    \and
  Violet (Xinying) Chen \\
  School of Business \\
  Stevens Institute of Technology \\
  \texttt{vchen3@stevens.edu} \\
}
\date{}

\begin{document}

\maketitle

\begin{abstract}
Many real-world systems rely on predictive models to inform decisions, and fairness concerns arise in both the prediction and decision stages. 
We introduce end-to-end fairness optimization (E2EFO) as a unifying framework that integrates fairness across the prediction-to-decision pipeline.
We focus on resource allocation with group-based fairness: the prediction task estimates allocation impacts while limiting accuracy disparity across groups, and the decision task distributes those impacts equitably by optimizing a group-based alpha-fairness measure. 
Within this framework, we propose fair decision-focused learning (FDFL), a training paradigm that jointly accounts for prediction accuracy, prediction fairness, and decision regret -- the loss in decision fairness due to imperfect predictions. 
FDFL trains the predictor by gradient descent, combining the objective gradients through multi-task learning techniques.
The core computational challenge is the decision Jacobian with respect to the predictor parameters: we derive exact closed-form formulas for a tractable class of fair allocation and apply a differentiable optimization layer in the general case. We further establish a finite-sample generalization bound for the scalarized FDFL objective.
Numerical experiments on a healthcare-based single resource allocation and a synthetic multiple resource allocation 
illustrate the value of jointly accounting for prediction fairness and decision fairness in prediction-informed decision-making.
\end{abstract}

\section{Introduction}\label{sec:Intro}
Data-driven predictions increasingly inform real-world decisions in critical domains. In healthcare, forecasts of medical needs guide the allocation of scarce care resources \citep{doi:10.1126/science.aax2342}; in social services, predicted risk informs screening for intervention \citep{chouldechova2018case}; in finance, loan processing decisions account for the expected creditworthiness of applicants. These processes typically consist of two stages: a \emph{prediction} task estimates unknown parameters, and a subsequent \emph{decision} task uses these estimates as input.

Many such applications require fairness, and it is worth distinguishing between fairness in how predictions are obtained and fairness in how decisions are made \citep{paulus2020predictably,scantamburlo2024prediction}. \emph{Prediction fairness} seeks to eliminate undesirable disparities in predictions, reflecting concerns about statistical properties of the predictor. \emph{Decision fairness} aims at equitable decision outcomes, reflecting moral concerns about the justice of decisions. \cite{kuppler2022fair} argue that both aspects should be jointly addressed.

A stylized example illustrates the potential insufficiency of pursuing either fairness alone. A decision maker allocates a fixed amount of resources to recipients in an advantaged group $A$ and a disadvantaged group $D$; each recipient derives a benefit from the resource, and a predictive model estimates these benefits from historical data carrying systemic bias against $D$. A standard accuracy-maximizing predictor tends to be more accurate
on $A$ and underestimate benefits in $D$. Suppose the decision maker responds with a fair decision policy that equalizes total benefits across the two groups. Because the policy takes the imperfect predictions as input, it may compensate $D$ by allocating disproportionately more resources than their true benefits justify, creating unfairness against $A$. Alternatively, suppose the decision maker responds with a fair predictor that 
reduces the underestimation at the cost of noisier estimates, and allocates to maximize the total predicted benefit. The allocation now favors $A$, which contains more high-benefit recipients, leaving $D$ under-served. Neither intervention suffices for fairness throughout the prediction-to-decision pipeline. 


We study this joint problem, which we call \emph{End-to-End Fairness Optimization} (E2EFO). The prediction task pursues accuracy for reliability while accounting for prediction fairness to mitigate unfair bias; the decision task, formalized as an optimization model that maximizes a decision fairness objective subject to feasibility constraints, seeks equitable outcomes.

We focus on resource allocation as the decision task. An \emph{allocation instance} is a set of stakeholders, partitioned into groups, with observable features and unknown allocation impacts. A predictor estimates the impacts, and the decision model distributes limited resources among the stakeholders. We specify each stakeholder's utility to be linear in the resources received, the feasible allocation set to be convex and compact, and the decision fairness objective to be a group-based $\alpha$-fairness measure of the induced utilities. With these specifications, the decision Jacobian with respect to the input prediction is computable (Section \ref{sec:grad-compute}).

E2EFO can be solved under two data-driven paradigms. Predict-then-optimize (PTO) methods address the two tasks separately \citep{bertsimas2020predictive,kannan2024residuals}: predictions are made first to achieve desirable accuracy and fairness, then input into a fairness-embedded optimization model. Decision-focused learning (DFL) methods instead train the predictor to directly optimize downstream decision quality \citep{donti2017task,wilder2019melding,spo2020}. 
In both paradigms, decision quality is measured by \emph{regret}, the loss in the decision objective from acting on predictions rather than the ground-truth parameters (formalized in Section~\ref{sec:prob}).
DFL improves regret over PTO when the two tasks are misaligned, that is, when good predictions do not guarantee good decisions \citep{spo2020,mandi2024decision}.
We mainly explore the DFL perspective because adopting different fairness perspectives in prediction and decision can expose E2EFO to such misalignment; our experiments identify the conditions under which it does.

E2EFO evaluates a predictor along three dimensions: prediction error, prediction disparity measuring the gap in prediction error across groups, and decision regret, all to be minimized. Existing DFL methods target decision regret alone to pursue decision quality. Optimizing all three jointly is a multi-objective problem, for which multi-task learning (MTL) \citep{sener2018multi,chen2024three} offers a principled paradigm. MTL trains one machine learning model against several objectives by combining their gradients, via static scalarization \citep{kendall2018multi} or dynamic conflict-avoidant combination \citep{desideri2012multiple,yu2020gradient}. We propose \emph{Fair Decision-Focused Learning} (FDFL), a family of training algorithms for E2EFO that draws on DFL methods to make training decision-aware and on MTL methods to balance the three objectives.

We summarize our main contributions as follows. 
\begin{enumerate}
    \item We propose E2EFO as a unifying framework for incorporating fairness into prediction-informed decision-making. Within this framework, we focus on group-based fairness: prediction fairness measures the disparity in prediction accuracy using the mean absolute deviation (MAD) of group-specific prediction errors, and decision fairness measures the equity of the utility distribution induced by allocations with a two-level group-based $\alpha$-fairness function. 
    \item We develop FDFL algorithms to train predictors with the three E2EFO objectives.
    Under gradient-based training, we use MTL techniques, including static scalarization and dynamic conflict-avoidant combinations, to combine the objective gradients. The decision-regret gradient requires differentiating through the decision model: we derive an exact closed-form decision Jacobian for the single-budget $\alpha$-fair allocation, and apply differentiable convex optimization layers from \citet{agrawal2019differentiable} for general convex feasible sets where no closed form is available.
    \item We establish a finite-sample generalization guarantee for scalarized FDFL training: with high probability, the composite excess risk of the empirical minimizer is $O\big(E_\Theta\sqrt{q/N}\big)$, under boundedness and Lipschitz regularity conditions. We verify the required conditions for our $\alpha$-fairness optimizing allocation tasks.
    \item We evaluate FDFL methods and fairness-embedded PTO baselines on a single-resource allocation built from a healthcare context, and on a synthetic multiple resource allocation with group imbalance varied by construction. Experimental results demonstrate the value of multi-objective training and identify the regimes in which each fairness objective delivers its gains. 
\end{enumerate}
The remainder of the paper is organized as follows. Section~\ref{sec:literature} reviews related work. Section~\ref{sec:prob} formulates the E2EFO problem and specifies the fairness measures. Section~\ref{sec:FDFL-alg} presents the FDFL training algorithms, the decision Jacobian computation, and a generalization guarantee. Section~\ref{sec:exp} reports experimental results. Section~\ref{sec:conclusion} concludes.

\section{Related Work} \label{sec:literature}
We review three lines of literature: (1) prediction-informed decisions under the PTO or DFL paradigms; (2) multi-task learning for handling multiple training objectives; and (3) fairness in prediction and in optimization. The first two lines provide the methodological foundation for our framework, while the third motivates it.

\subsection{Prediction-Informed Decisions}
Prediction-informed decisions belong to the broad field of data-driven prescriptive decision-making, 
also known as data-driven optimization or contextual optimization. We adopt the term prediction-informed decisions to emphasize the combination of prediction techniques from machine learning with the optimization methodology. The survey by \citet{sadana2025survey} organizes this literature under three paradigms: decision rule optimization, sequential learning and optimization, and integrated learning and optimization. The latter two share a structure in which a prediction model is learned prior to optimizing decisions, which aligns with our setting, so we review these two paradigms. 

Sequential learning and optimization, also referred to as predict-then-optimize (PTO), first trains a prediction model to estimate uncertain parameters and then plugs the estimates into a decision optimization model. The predictor may produce a point estimate, yielding a deterministic view of the downstream optimization, or a conditional distribution, yielding a stochastic view that accounts for prediction uncertainty. In this sequential paradigm, the stochastic view is commonly adopted \citep{bertsimas2020predictive,kannan2024residuals}. As \citet{sadana2025survey} note, the two views coincide when the decision objective is linear in the predictions, since estimating the conditional distribution reduces to
estimating its expectation. We adopt the deterministic view in this work.

Integrated learning and optimization, also referred to as decision-focused learning (DFL), trains the predictor to directly optimize decision performance, enabling the prediction component to anticipate its impact on the decisions. DFL has been studied for a variety of decision tasks, including linear programming \citep{spo2020}, quadratic programming \citep{amos2017optnet, agrawal2019differentiable}, and general nonlinear optimization \citep{shah2022decision}. Recent work develops scalable, general-purpose techniques, such as directional gradients \citep{huang2024decision} and landscape surrogates \citep{zharmagambetov2023landscape}. These methods differ primarily in the training loss and in how gradients are computed for backpropagation; we refer to \citet{mandi2024decision} for a comprehensive survey.

Theoretical comparisons under the stochastic view show that PTO dominates DFL in regret 
when the predictor model class is well-specified and data are sufficient, with the reverse result holding in the misspecified setting \citep{elmachtoub2023estimate}.
\citet{elmachtoub2025dissecting} refined this characterization with finite-sample regret bounds depending on the degree of misspecification.
Empirically, DFL has demonstrated gains over PTO in multi-period inventory management \citep{qi2023practical} and healthcare inventory allocation \citep{chung2022decision}. 
Together, these findings indicate that the integrated view is valuable when prediction and decision performance are misaligned.

\subsection{Multi-Task Learning}
Multi-task learning (MTL), also referred to as multi-objective learning, trains a single shared model against several objectives. \citet{sener2018multi} cast MTL explicitly as multi-objective optimization and formalizes two targets: Pareto optimality across all objectives, and the weaker Pareto stationarity that holds when no common descent direction improves all objectives simultaneously.

MTL training algorithms weigh objectives either statically or dynamically. Static weighting applies a fixed scalarization, replacing the vector-valued loss with a linear combination of the objectives \citep{kendall2018multi}, which reduces MTL to standard scalar loss minimization. \citet{hu2023revisiting} showed that scalarization is in general incapable of tracing out the full Pareto front. Dynamic weighting adapts the weights on the objective gradients during training to reduce conflicts; examples include MGDA \citep{desideri2012multiple}, PCGrad \citep{yu2020gradient}, and Nash-MTL \citep{navon2022multi}. There is mixed empirical evidence comparing the weighting schemes: simple scalarization with standard regularization has been found competitive with, or superior to, more elaborate dynamic rules \citep{xin2022current,kurin2022defense}.

Closest to our work, \citet{jeon2025plg} applies MTL within DFL, treating the prediction loss as a secondary objective alongside the decision-quality loss and combining the two gradients into an update rule whose iterates converge to a Pareto-stationary point. We extend this idea to three objectives by adding prediction fairness. Our work contributes no new gradient-combination method; we treat MTL techniques as interchangeable components of FDFL for balancing the training objectives.

\subsection{Fairness in Prediction and Optimization}
Fairness has been studied extensively in both machine learning and optimization, though the two communities pursue different goals. Fair machine learning seeks parity across groups or individuals in predictive models to reduce discriminatory bias; \citet{mehrabi2021survey} survey the fairness definitions and techniques. Fairness in optimization instead emphasizes the fairness of decision outcomes, typically measured by how utilities are distributed among individuals or groups; \citet{Chen2023} survey utility-based fairness metrics and their use in fairness-aware optimization.

Recent work in both communities increasingly emphasizes outcome-aware fairness, 
motivated by the finding that fair predictions do not guarantee equitable outcomes \citep{liu2018delayed,corbett2023measure}. This perspective supports embedding fairness throughout the decision pipeline rather than in the prediction stage alone \citep{scantamburlo2024prediction}. One line of research pursues outcome-aware fairness by directly learning fair decision policies \citep{chohlas2024learning}, aligning with the decision-rule-optimization paradigm of \citet{sadana2025survey}.

A related line extends DFL to fairness-aware settings by embedding fairness in the decision model. \citet{kotary2022end} imposes fairness constraints in the learning-to-rank problem, leveraging the linearity of the problem to develop a DFL method for fair rankings. \citet{dinh2024learning} also studies fair ranking, defining fairness through an ordered weighted average (OWA) of group exposures and developing a DFL algorithm with customized forward and backward passes; \citet{dinh2024end} generalizes the OWA approach to fairness among multiple decision objectives under uncertainty. These extensions address decision fairness alone. A persistent gap remains: no integrated framework has addressed fairness in both prediction and decision stages. Our work fills this gap, with fair resource allocation as the decision problem.

\section{Problem Formulation: End-to-End Fairness Optimization} \label{sec:prob}
We formalize E2EFO for resource allocation. A decision maker distributes limited resources among stakeholders who derive beneficial impacts from the resources, such as improved access, health, or opportunity. These impacts are not observable before allocation and must be predicted.

In an \emph{allocation instance}, a decision maker considers $m$ stakeholders, indexed by $i \in [m] = \{1, \ldots, m\}$, for the allocation of $R$ resource types. In the single-resource case ($R=1$), we let $\mathbf{d} \in \mathbb{R}^m_{\geq 0}$ denote the allocation, where $d_i$ is the amount given to stakeholder $i$, and let $\mathbf{r} \in \mathbb{R}^m_{> 0}$ denote the allocation outcomes, where $r_i$ represents the positive impact of one unit of resource on $i$.
For the general case with $R > 1$, the allocation becomes $\mathbf{d} \in \mathbb{R}^{m \times R}_{\geq 0}$ with entries $d_{ij}$, and each stakeholder carries an impact vector
$\mathbf{r}_i = (r_{i1},\ldots,r_{iR}) \in \mathbb{R}^{R}_{>0}$, where $r_{ij}$
is the impact of one unit of resource $j$ on $i$. In this section, we present the single-resource case ($R=1$) for notational ease. The experiments in Section~\ref{sec:exp} consider both a single-resource setting and a multiple-resource setting ($R=3$).

Each stakeholder $i$ is described by a feature vector $\mathbf{x}_i$. The $m$ feature vectors of an instance form the matrix $X = [\mathbf{x}_1,\ldots,\mathbf{x}_m]^{\top}$. The \textbf{prediction task} estimates each stakeholder's impact from its features, $\hat{r}_i = f_\theta(\mathbf{x}_i)$ parameterized by $\theta \in \Theta$. Applying the predictor to a full instance gives the estimate $\hat{\mathbf{r}} = f_\theta(X) \in \mathbb{R}^m$. Since true impacts are positive, we require positive predictor output. 

Prediction quality is evaluated along two dimensions. A \emph{prediction accuracy} loss $L(\hat{\mathbf{r}};\mathbf{r})$ measures the discrepancy between predicted and true impacts. We adopt the mean squared error $L(\hat{\mathbf{r}};\mathbf{r}) = \frac{1}{m}\sum_{i=1}^{m}(r_i-\hat r_i)^2$. 
To mitigate potential biases in historical data or structural inequities in the feature-to-impact mapping, we consider a \emph{prediction fairness} criterion $F(\hat{\mathbf{r}};\mathbf{r})$, which quantifies disparities in prediction errors across stakeholders. We specify $F$ in Section~\ref{sec:fairness-def}. 
For both metrics, lower values are desirable. We abbreviate them as $L(\hat{\mathbf{r}})$ and $F(\hat{\mathbf{r}})$ to simplify notation. 

Given predictions $\hat{\mathbf{r}}$, each stakeholder derives utility $u_i = U_i(d_i;\hat{r}_i)$ from the received allocation.
Thus a decision $\mathbf{d}$ induces a utility distribution $\mathbf{u} = (u_1,\ldots,u_m)$. A \emph{decision fairness} measure $W: \mathbb{R}^m \to \mathbb{R}$ evaluates the fairness of this distribution, $W(\mathbf{d};\hat{\mathbf{r}}) = W(U_1(d_1;\hat{r}_1),\ldots,U_m(d_m;\hat{r}_m))$, with higher values reflecting more equitable outcomes. The \textbf{decision task} solves
\begin{equation} \label{eq:opt_prob}
\max_{\mathbf{d}} \quad W(\mathbf{d};\hat{\mathbf{r}}) ~ \text{s.t.} \quad \mathbf{d} \in \mathcal{S}.
\end{equation}
We assume that the feasible set $\mathcal{S} \subseteq \mathbb{R}^{m}$ is nonempty, compact, convex, and independent of $\hat{\mathbf{r}}$. $\mathcal{S}$ may vary across instances through instance-specific data such as resource budgets. 
Let $\mathbf{d}^*(\hat{\mathbf{r}})$ denote an optimal solution to \eqref{eq:opt_prob}, and 
$\mathbf{d}^*(\mathbf{r})$ the solution under the true impacts $\mathbf{r}$.

We refer to this prediction-informed decision process as the E2EFO problem. The predictor $f_\theta$ is trained on historical instances $\{X_s, \mathbf{r}_s, \mathcal{S}_s\}_{s=1}^{N}$ drawn from a population $\mathcal{P}$. In deployment, facing a new instance $X_{\mathrm{test}} \sim \mathcal{P}$ with unknown impacts, the decision maker uses $\hat{\mathbf{r}} = f_\theta(X_{\mathrm{test}})$ in \eqref{eq:opt_prob}. Prediction and decision jointly shape the outcome, which motivates evaluating the predictor on decision quality as well as prediction quality. We measure decision quality by \emph{decision regret}, the gap in the decision-fairness objective between the prediction-based decision and the full-information decision when both are evaluated under the true impacts:
\begin{equation} \label{eq:regret}
\mathrm{Regret}(\hat{\mathbf{r}};\mathbf{r}) := W\!\left(\mathbf{d}^*(\mathbf{r});\mathbf{r}\right)
- W\!\left(\mathbf{d}^*(\hat{\mathbf{r}});\mathbf{r}\right).
\end{equation}
Since $\mathbf{d}^*(\mathbf{r})$ maximizes $W(\cdot\,;\mathbf{r})$ over $\mathcal{S}$ and $\mathbf{d}^*(\hat{\mathbf{r}}) \in \mathcal{S}$, the regret is nonnegative and vanishes when the prediction-based decision attains the same fairness as the full-information decision. 
The first term of \eqref{eq:regret} is the optimal value of \eqref{eq:opt_prob} and is unique; the second term is unambiguous when $\mathbf{d}^*(\cdot)$ is single-valued, which holds under the fairness measure specified in Section~\ref{sec:fairness-def}. 
E2EFO thus evaluates a predictor along three dimensions, all to be minimized: $L(\hat{\mathbf{r}})$, $F(\hat{\mathbf{r}})$, and $\mathrm{Regret}(\hat{\mathbf{r}};\mathbf{r})$.

\begin{remark} \label{rem:deterministic}
Our formulation adopts a deterministic view of the decision task: the predictor produces a point estimate $\hat{\mathbf{r}}$, which is plugged into \eqref{eq:opt_prob}, and regret is measured against the impacts realized after the decision. An alternative stochastic view models $\mathbf{r}$ as a random vector and estimates its conditional distribution, requiring a stochastic or robust decision model. We leave the extension of E2EFO to this stochastic view for future work.
\end{remark}

\subsection{Fairness Definitions} \label{sec:fairness-def}
E2EFO accommodates a wide range of fairness definitions. We specify the fairness metrics used in our study. All metrics are defined \emph{within a single instance}: prediction fairness over the prediction errors of the $m$ stakeholders, and decision fairness over their utility distribution. We focus on \emph{group-based} fairness, where groups $G_1, \ldots, G_K$ partition the $m$ stakeholders and fairness is assessed across these groups in both the prediction and decision tasks. 

\paragraph{Prediction Fairness.}
The prediction task should not systematically disadvantage any group. We measure prediction fairness with an accuracy-disparity metric, which applies a dispersion measure to the group-level prediction errors, taking smaller values when the errors are more evenly distributed. Such measures are widely used in fair regression \citep{berk2017convex,agarwal2019fair}. Our running example is the mean absolute deviation (MAD) of group-level mean squared error (MSE). With $m_k = |G_k|$, the group-$k$ error is $\mathrm{MSE}_k = \frac{1}{m_k}\sum_{i \in G_k} (r_i-\hat{r}_i)^2$. The average error over all groups is $\overline{\mathrm{MSE}} = \frac{1}{K}\sum_{k=1}^K \mathrm{MSE}_k$. 
The group-based accuracy disparity is:
\begin{equation} \label{def:group-acc-parity}
    F(\hat{\mathbf{r}}) = \frac{1}{K} \sum_{k=1}^K |\mathrm{MSE}_k - \overline{\mathrm{MSE}}|.
\end{equation}

\paragraph{Decision Fairness.}
We suppose $r_i$ captures $i$'s utility gain per unit of resource, and specify the stakeholder utility to be linear in the allocation: $U_i(d_i; r_i) = r_i d_i$, 
with $U_i(d_i;\hat{r}_i) = \hat{r}_i d_i$ when utilities are evaluated under predicted impacts. With multiple resources ($R > 1$), utility aggregates across resource types: $U_i(\mathbf{d}_{i}; \mathbf{r}_{i}) = \sum_{j=1}^{R} r_{ij} d_{ij}$.

A fairness measure $W$ aggregates utility values into a scalar score. We adopt the widely used $\alpha$-fairness \citep[e.g.,][]{mo2000fair}, a family of measures offering a tunable trade-off between efficiency and equity: $W_{\alpha}(\mathbf{u}) =
\begin{cases}
    \sum_{i=1}^{m} \frac{u_i^{1 - \alpha}}{1-\alpha}, & \text{if } \alpha \geq 0 \text{ and } \alpha \neq 1; \\
    \sum_{i=1}^{m} \log(u_i), & \text{if } \alpha = 1
\end{cases} $.
At $\alpha = 0$, this measure simplifies to the utility sum, a pure efficiency goal. As $\alpha$ increases, the measure places greater emphasis on equity, with $\alpha \to \infty$ recovering the Rawlsian max-min criterion.
Applied directly to the utility vector $\mathbf{u}$, this measure captures individual-level decision fairness. To assess fairness across groups, we extend it to a two-level construction that applies the $\alpha$-fair kernel within and across groups.

Let $g_k(\mathbf{u})$ denote the raw $\alpha$-fair score of group $G_k$, that is, $W_{\alpha}$ applied to the utility distribution in $G_k$. The outer aggregation applies the $\alpha$-fair kernel to these group scores, which requires each to be \emph{strictly positive}. The raw score is guaranteed to be positive only for $0 \leq \alpha < 1$. When $\alpha \geq 1$, we replace $g_k$ by a strictly increasing positive transformation, which preserves the ordering of groups by intra-group fairness while making the outer aggregation well defined. This yields the intra-group fairness score: 
\begin{equation} \label{eq:alpha-group-g}
h_k(\mathbf{u}) =
\begin{cases}
    \sum_{i \in G_k} \frac{u_i^{1 - \alpha}}{1-\alpha} \;=\; g_k(\mathbf{u}), & \text{if } 0 \leq \alpha < 1; \\[2pt]
    \prod_{i \in G_k} u_i \;=\; \exp\big(g_k(\mathbf{u})\big), & \text{if } \alpha = 1; \\[2pt]
    \frac{\alpha - 1}{\sum_{i \in G_k} u_i^{1-\alpha}} \;=\; -1/g_k(\mathbf{u}), & \text{if } \alpha > 1,
\end{cases}
\end{equation}
where a higher $h_k(\mathbf{u})$ indicates a more equitable utility distribution within $G_k$. Aggregating the intra-group scores with the same $\alpha$-fair kernel gives group-based $\alpha$-fairness:
\begin{equation} \label{eq:alpha-group-W}
    W^g_{\alpha}(\mathbf{u}) =
\begin{cases}
    \sum_{k=1}^{K} \frac{h_k^{1 - \alpha}(\mathbf{u})}{1-\alpha}, &\text{if } \alpha \geq 0,\ \alpha \neq 1; \\
    \sum_{k=1}^{K} \log(h_k(\mathbf{u})), &\text{if } \alpha = 1.
\end{cases}
\end{equation}
This formulation accounts for both intra-group fairness, via $h_k$, and inter-group fairness, via the outer aggregation. Like the individual measure $W_{\alpha}$, $W^g_{\alpha}$ reduces to the utility sum at $\alpha = 0$ and weighs equity more heavily as $\alpha$ grows; we therefore require $\alpha > 0$ in the decision objective, excluding the pure-efficiency case. To our knowledge, this two-level $\alpha$-fair composition has not appeared in prior work.

The group-based measure $W^g_{\alpha}(\mathbf{u})$ is strictly concave in $\mathbf{u}$ for every $\alpha > 0$ (Lemma~\ref{lem:alpha-fair-concavity}). In the single-resource case, the utility map $u_i = \hat r_i d_i$ is injective, so the decision objective is strictly concave in $\mathbf{d}$ and $\mathbf{d}^*(\hat{\mathbf{r}})$ is unique. For $R>1$, ithe map $u_i = \sum_{j} \hat r_{ij} d_{ij}$ is not injective; the uniqueness of $\mathbf{d}^*(\hat{\mathbf{r}})$ requires the conic nondegeneracy condition of Assumption~\ref{assp:gen-conic} in Appendix \ref{app:verify}.

\begin{remark}[Coincidence with the individual measure] \label{rem:alpha-coincide} The two-level measure \eqref{eq:alpha-group-W} coincides with the individual measure $W_{\alpha}$ at two values of $\alpha$. At $\alpha = 1$, $h_k = \prod_{i \in G_k} u_i$ gives $W^g_1(\mathbf{u}) = \sum_{k} \sum_{i \in G_k} \log u_i = W_1(\mathbf{u})$. At $\alpha = 2$, $h_k = \big(\sum_{i \in G_k} u_i^{-1}\big)^{-1}$ gives $W^g_2(\mathbf{u}) = -\sum_{k}\sum_{i \in G_k} u_i^{-1} = W_2(\mathbf{u})$. At all other $\alpha > 0$, the two measures differ, and the group structure affects the allocation. \end{remark}

\section{Fair Decision-Focused Learning Algorithms} \label{sec:FDFL-alg}
To solve an E2EFO problem, we first learn a predictor that produces the estimate $\hat{\mathbf{r}}$ and then compute $\mathbf{d}^*(\hat{\mathbf{r}})$ by solving \eqref{eq:opt_prob}. The decision task is a convex program that off-the-shelf solvers handle directly. The predictor training should account for the three E2EFO evaluation dimensions of Section~\ref{sec:prob}. We formalize this training problem as \emph{Fair Decision-Focused Learning} (FDFL). 

We consider a differentiable parametric predictor $f_{\theta}$ (e.g., a neural network) shared across instances, generating $\hat{\mathbf{r}}_s = f_{\theta}(X_s)$ and informing the decision $\mathbf{d}^*(\hat{\mathbf{r}}_s)$.
The three training objectives are empirical averages over the $N$ training instances: the prediction-accuracy loss $L_{\mathrm{pred}}(\theta) = \frac{1}{N}\sum_{s=1}^N L(f_{\theta}(X_s))$, the prediction disparity $F(\theta) = \frac{1}{N}\sum_{s=1}^N F(f_{\theta}(X_s))$, and the decision regret $L_{\mathrm{regret}}(\theta) =  \frac{1}{N}\sum_{s=1}^N \mathrm{Regret}(f_{\theta}(X_s); \mathbf{r}_s)$.
They form the vector-valued FDFL objective:
\begin{equation} \label{eq:fdfl-vector-loss}
\mathcal{L}_{\mathrm{FDFL}}(\theta) = \bigl( L_{\mathrm{pred}}(\theta),\; F(\theta),\; L_{\mathrm{regret}}(\theta) \bigr).
\end{equation}
FDFL generalizes the two established training paradigms: PTO methods train on the prediction objectives and omit $L_{\mathrm{regret}}$, whereas DFL methods train on $L_{\mathrm{regret}}$ alone.

To enable gradient-based training, we apply the chain rule to decompose the three objective gradients. The instance subscript $s$ is dropped since the gradient chain is identical across instances.
\begin{equation}  \label{eq:grad-chain-pred}
\nabla_\theta L_{\mathrm{pred}}(\theta) = \frac{\partial L(\hat{\mathbf{r}})}{\partial \hat{\mathbf{r}}}\cdot\frac{\partial \hat{\mathbf{r}}}{\partial \theta},
\qquad
\nabla_\theta F(\theta) = \frac{\partial F(\hat{\mathbf{r}})}{\partial \hat{\mathbf{r}}}\cdot\frac{\partial \hat{\mathbf{r}}}{\partial \theta} 
\end{equation}
\begin{equation}  \label{eq:grad-chain-regret}
\nabla_{\theta} L_{\text{regret}}(\theta) = -\frac{\partial W(\mathbf{d}^*(\hat{\mathbf{r}});\mathbf{r})}{\partial \mathbf{d}^*(\hat{\mathbf{r}})} \cdot \frac{\partial \mathbf{d}^*(\hat{\mathbf{r}})}{\partial \hat{\mathbf{r}}} \cdot \frac{\partial \hat{\mathbf{r}}}{\partial \theta}
\end{equation}

These components differ in computational difficulty. The objective-gradient terms, $\frac{\partial L(\hat{\mathbf{r}})}{\partial \hat{\mathbf{r}}}$, $\frac{\partial F(\hat{\mathbf{r}})}{\partial \hat{\mathbf{r}}}$, $\frac{\partial W(\mathbf{d}^*(\hat{\mathbf{r}});\mathbf{r})}{\partial \mathbf{d}^*(\hat{\mathbf{r}})}$, differentiate explicit metrics: $L$ and $W$ are differentiable on their domains, and the MAD-based disparity $F$ is piecewise differentiable admitting closed-form subgradients at the kinks of its absolute-value terms. The predictor Jacobian $\frac{\partial \hat{\mathbf{r}}}{\partial \theta}$ is available by automatic differentiation. The decision Jacobian $\frac{\partial \mathbf{d}^*(\hat{\mathbf{r}})}{\partial \hat{\mathbf{r}}}$ poses the main technical challenge, since $\mathbf{d}^*(\hat{\mathbf{r}})$ generally lacks an explicit form and need not be differentiable at every $\hat{\mathbf{r}}$; Section~\ref{sec:grad-compute} addresses its computation in our setup.

FDFL is an MTL problem in which the three objective gradients are combined into a single update direction, either by static scalarization or by dynamic weighting that adapts the combination in each iteration to avoid conflicts. 
Section \ref{sec:mtl-combine} provides details for implementing both MTL strategies. Algorithm~\ref{alg:dfl} summarizes the training procedure.

\begin{algorithm}[ht]
\caption{Gradient-based Training for FDFL}
\label{alg:dfl}
\small
\setlength{\abovedisplayskip}{3pt}
\setlength{\belowdisplayskip}{3pt}
\begin{algorithmic}[1]
\Require features $\{X_s\}_{s=1}^N$; true impacts $\{\mathbf{r}_s\}_{s=1}^N$; decision feasible sets $\{\mathcal{S}_s\}_{s=1}^N$; prediction-accuracy loss $L$; prediction disparity $F$; decision fairness objective $W$; learning rate $\eta$; combination rule $\mathrm{Combine}$.
\State Initialize predictor parameters $\theta \in \Theta$ for $f_\theta$.
\For{each training epoch}
    \For{each instance $s$ in a sampled batch $B$}
    \State $\hat{\mathbf{r}}_s \gets f_\theta(X_s)$. \Comment{Predict impacts}
    \State $\mathbf{d}^*(\hat{\mathbf{r}}_s) \gets \arg\max_{\mathbf{d}\in\mathcal{S}_s}\,W(\mathbf{d};\hat{\mathbf{r}}_s)$. \Comment{Solve prediction-informed decisions}
    \EndFor
    \State $L_{\mathrm{pred}}(\theta) \gets \frac{1}{|B|}\sum_{s \in B} L(\hat{\mathbf{r}}_s;\mathbf{r}_s)$, $F(\theta) \gets \frac{1}{|B|}\sum_{s \in B} F(\hat{\mathbf{r}}_s;\mathbf{r}_s)$,
    \State $L_{\mathrm{regret}}(\theta) \gets \frac{1}{|B|}\sum_{s \in B} \left[W(\mathbf{d}^*(\mathbf{r}_s);\mathbf{r}_s) - W(\mathbf{d}^*(\hat{\mathbf{r}}_s);\mathbf{r}_s) \right]$. \Comment{Evaluate objectives}
    \State $\mathbf{g}_{\mathrm{pred}} \gets \nabla_\theta L_{\mathrm{pred}}(\theta)$,  $\mathbf{g}_{\mathrm{fair}} \gets \nabla_\theta F(\theta)$, $\mathbf{g}_{\mathrm{dec}} \gets \nabla_\theta L_{\mathrm{regret}}(\theta)$. \Comment{Compute gradients}
    \State $\mathbf{g} \gets \mathrm{Combine}(\mathbf{g}_{\mathrm{pred}}, \mathbf{g}_{\mathrm{fair}}, \mathbf{g}_{\mathrm{dec}})$. \Comment{Static or dynamic combination}
    \State Update $\theta \gets \theta - \eta\,\mathbf{g}$.
\EndFor
\State \Return $\theta$
\end{algorithmic}
\end{algorithm}
\begin{remark} 
The decision-regret loss is generally nonconvex in $\theta$, so gradient-based training targets stationary points rather than global optima. Section~\ref{sec:mtl-combine} states the convergence properties of each combination rule and Section~\ref{sec:generalization} gives statistical guarantees for the scalarized objective. A structural regret ordering between decision-focused and prediction-focused methods in the spirit of \citet{elmachtoub2023estimate} remains open for the E2EFO setup. 
\end{remark}

\subsection{Decision Jacobian Computation} \label{sec:grad-compute}
The decision Jacobian is computed in the same manner for each training instance:
each instance $s$ contributes $\frac{\partial \mathbf{d}^*(\hat{\mathbf{r}}_s)}{\partial \hat{\mathbf{r}}_s}$ to the average regret gradient; we continue to drop the instance subscript $s$ for notational ease. 
When the optimal decision is available in closed form, the Jacobian follows by direct differentiation. Otherwise, the general technique is to differentiate the optimality conditions of the decision problem \citep{mandi2024decision}. Section~\ref{sec:closed} derives the closed-form Jacobian for $\alpha$-fairness maximization over a nonnegative knapsack, and Section~\ref{sec:cvxpylayers} applies differentiable convex optimization layers \citep{agrawal2019differentiable} to a general convex feasible region.

\subsubsection{Closed-Form Jacobian} \label{sec:closed}
We consider allocation under a single budget constraint. The decision maker has budget $Q > 0$ to distribute among $m$ stakeholders, partitioned into $K$ groups. Each unit of resource allocated to stakeholder $i$ incurs a known cost $c_i > 0$ and leads to impact $r_i$ (or its prediction $\hat{r}_i$). 
The feasible set $\mathcal{S} = \{\mathbf{d} \geq \mathbf{0} : \sum_{i=1}^{m} c_i d_i \leq Q\}$ is a nonnegative knapsack, varying across instances through the instance-specific $(\mathbf{c}, Q)$. The decision task maximizes the group-based $\alpha$-fairness of the induced utilities:
\begin{equation} \label{eqsys:alpha-g}
    \max_{\mathbf{d}} ~  W_{\alpha}^g (\mathbf{u}) \quad
    \text{ s.t. } u_i = \hat{r}_id_i, ~ d_i \geq 0, ~ \forall i; \sum_{i=1}^m c_id_i \leq Q.
\end{equation}

In Proposition~\ref{prop:closed_form_solution_group-1}, we use the Karush-Kuhn-Tucker (KKT) conditions to derive the closed-form optimal solution to \eqref{eqsys:alpha-g}. 
Differentiating the solution yields the exact decision Jacobian, whose components we state in Proposition \ref{prop:analytical_gradient_group-1} of Appendix~\ref{app:proofs}. We give all formulas with the true impacts $\mathbf{r}$ as input. During training, the Jacobian $\frac{\partial \mathbf{d}^*(\hat{\mathbf{r}})}{\partial \hat{\mathbf{r}}}$ is obtained by substituting $\hat{\mathbf{r}}$ for $\mathbf{r}$. We present the general case $\alpha \in (0,1) \cup (1,\infty)$; the cases $\alpha \in \{0, 1\}$ and $\alpha \to \infty$ and all proofs appear in Appendix~\ref{app:proofs}.


\begin{proposition} [Closed-Form Decisions]
\label{prop:closed_form_solution_group-1}
The optimal solution to \eqref{eqsys:alpha-g} with parameter $\mathbf{r}$ is given by: for all $i \in [m]$, $d_i^* = \left ( Q c_i^{-\frac{1}{\alpha}}r_i^{\frac{1-\alpha}{\alpha}}S_{k(i)}^{\frac{1}{\alpha-2}} \right) / \sum_{k=1}^K S_k^{1 +\frac{1}{\alpha-2}}$ at $0 < \alpha < 1$, or $d_i^* = \left(Q c_i^{-\frac{1}{\alpha}}r_i^{\frac{1-\alpha}{\alpha}} S_{k(i)}^{\frac{-\alpha+2}{-\alpha^2 + 2\alpha-2}} \right) / \sum_{k=1}^K S_k^{1 + \frac{-\alpha+2}{-\alpha^2 + 2\alpha-2}}$ at $\alpha > 1$,
where $k(i)$ is the group of stakeholder $i$, and for each group $k \in [K]$, $S_k = \sum_{i \in G_k} \left(c_i^{-1/\alpha} r_i^{1/\alpha} \right)^{1-\alpha}$. 
\end{proposition}

\subsubsection{Differentiation through Optimality Conditions} \label{sec:cvxpylayers}
For a general convex, compact feasible region $\mathcal{S}$, the optimal allocation $\mathbf{d}^*(\hat{\mathbf{r}})$ admits no closed form, but the decision Jacobian remains computable by implicitly differentiating the optimality conditions of the decision model. The DFL literature offers several such methods \citep{mandi2024decision} that differentiate the KKT conditions of a quadratic program \citep[OptNet;][]{amos2017optnet}, the optimality conditions of a conic program \citep[cvxpylayers;][]{agrawal2019differentiable}, or the fixed-point condition of an iterative algorithm \citep[FoldOpt;][]{kotaryfoldopt2023}. We adopt cvxpylayers, which differentiates through any convex program expressible as a \emph{disciplined parametrized programming} (DPP) form.

DPP restricts how parameters enter an optimization model, so that the canonicalization into a cone program depends affinely on the parameters. 
Cvxpylayers computes the decision Jacobian by implicitly differentiating the optimality conditions of the resulting conic program, returning a least-squares approximation when the conic program is not differentiable \citep[Appendix B,][]{agrawal2019differentiable}.
Our decision model is DPP-compliant with $\mathbf{d}$ and $\mathbf{u}$ as the variables and $\hat{\mathbf{r}}$ as the only parameter. The objective $W_{\alpha}^{g}(\mathbf{u})$ is concave in $\mathbf{u}$ and parameter-free. The utility coupling $u_i = \hat{r}_i d_i$ or $u_i = \sum_j \hat r_{ij} d_{ij}$ is a sum of products, each pairing a factor affine in the parameter with a parameter-free variable. $\mathcal{S}$ is convex and independent of $\hat{\mathbf{r}}$. Hence cvxpylayers applies.

\subsection{Gradient Combination} \label{sec:mtl-combine}
At each training iteration, FDFL combines the three objective gradients into one update direction, $\mathbf{g} = \mathrm{Combine}(\mathbf{g}_{\mathrm{pred}}, \mathbf{g}_{\mathrm{fair}}, \mathbf{g}_{\mathrm{dec}})$. Two objectives \emph{conflict} at $\theta$ when their gradients form an obtuse angle, $\langle \mathbf{g}_i, \mathbf{g}_j\rangle < 0$, so that descending one objective increases the other. MTL aims to attain Pareto stationarity, that is, learn a predictor parameter $\theta$ at which no common descent direction improves all three objectives. We next describe three combination rules that we implement in FDFL.

\paragraph{Static scalarization.}
The simplest rule is a fixed weighted sum. Normalizing the decision-regret weight to one, we place $\mu \geq 0$ on prediction accuracy and $\lambda \geq 0$ on prediction disparity. We refer to FDFL adopting the following update direction as \emph{FDFL-Scal}. 
\begin{equation} \label{eq:fdfl-scal-rule}
\mathbf{g}_{\mathrm{FDFL\text{-}Scal}}
  = \mu\,\mathbf{g}_{\mathrm{pred}} + \lambda\,\mathbf{g}_{\mathrm{fair}} + \mathbf{g}_{\mathrm{dec}},
\end{equation}
FDFL-Scal trains a predictor by gradient descent on the scalarized loss $\mu L_{\mathrm{pred}}(\theta) + \lambda F(\theta) + L_{\mathrm{regret}}(\theta)$, targeting convergence to a first-order stationary point of the combined loss, which is Pareto stationary by definition. Scalarization is simple and interpretable, but fixed weights may lead to undesirable trade-offs when the gradients conflict during training.

\paragraph{Dynamic conflict-avoidant combination.} 
The alternative strategy re-computes a combination at each iteration from the geometry of the objective gradients. From the MTL literature, we select dynamic rules with convergence guarantee to Pareto stationarity.

\emph{FDFL-PCGrad} applies the technique from \citet{yu2020gradient} to remove conflict by projection: each gradient is projected onto the normal plane of any other gradient with which it has a negative inner product, and the de-conflicted gradients are summed into the update direction.
Following \cite{yu2020gradient}, at each iteration and for each objective $i \in \{\text{pred}, \text{fair}, \text{dec}\}$, we initialize $\mathbf{g}_i^{\mathrm{PC}} = \mathbf{g}_i$ and traverse the other two objectives $j$ in a random order, projecting the running gradient onto the normal plane of $\mathbf{g}_j$ whenever the two conflict:
\begin{equation} \label{eq:pcgrad}
  \mathbf{g}_i^{\mathrm{PC}} \leftarrow \mathbf{g}_i^{\mathrm{PC}}
  - \frac{\min\bigl\{\langle \mathbf{g}_i^{\mathrm{PC}}, \mathbf{g}_j\rangle,\, 0\bigr\}}{\lVert \mathbf{g}_j\rVert^2}\,\mathbf{g}_j;
  \qquad
  \mathbf{g} = \mathbf{g}_{\mathrm{pred}}^{\mathrm{PC}}
  + \mathbf{g}_{\mathrm{fair}}^{\mathrm{PC}}
  + \mathbf{g}_{\mathrm{dec}}^{\mathrm{PC}}.
\end{equation}

\emph{FDFL-NashMTL} applies the Nash bargaining combination from \citet{navon2022multi}. With combination treated as a bargaining game among the objectives, we use the update direction  $\mathbf{g} = \omega_1 \mathbf{g}_{\mathrm{pred}} + \omega_2 \mathbf{g}_{\mathrm{fair}} + \omega_3 \mathbf{g}_{\mathrm{dec}}$ where the weights $\boldsymbol{\omega} > \mathbf{0}$
solve
\begin{equation} \label{eq:nashmtl}
  G^{\top} G\,\boldsymbol{\omega} = [\frac{1}{\omega_1}, \frac{1}{\omega_2}, \frac{1}{\omega_3}]^{\top},
  \qquad
  G = [\,\mathbf{g}_{\mathrm{pred}},\ \mathbf{g}_{\mathrm{fair}},\ \mathbf{g}_{\mathrm{dec}}\,].
\end{equation}

Two additional rules -- multiple-gradient-descent algorithm \citep{desideri2012multiple} taking the minimum-norm combination from the convex hull of the objective gradients, and a prediction-disparity-regularized extension of prediction-loss-guided DFL \citep{jeon2025plg} -- are implemented and evaluated in Section~\ref{sec:exp}. Appendix~\ref{app:method-details} explains both.


\subsection{Generalization Theory} \label{sec:generalization}\label{sec:fdfl-theory}
We now provide a finite-sample guarantee for FDFL-Scal. We collect the data of one allocation instance as $z=(X,\mathbf{a},\mathbf{r},D_\mathcal{S})$, where $X$ and $\mathbf{r}$ are the feature matrix and true impacts, $\mathbf{a}\in[K]^m$ records group memberships,
and $D_\mathcal{S}$ specifies $\mathcal{S}$ (e.g., the costs $\mathbf{c}$ and budget $Q$ in Section~\ref{sec:closed}). The training set $\mathcal{D}=\{z_s\}_{s=1}^N$ contains independent and identically distributed (i.i.d.) samples from a population $\mathcal{P}$ on the instance space $\mathcal{Z}$. This sampling formulation and the notations $v_0$ and $\hat v_0$ used below follow the integrated estimation-optimization framework of \citet{elmachtoub2025dissecting}.

The per-instance composite loss aggregates three terms,
$\ell^{\mathrm{full}}(\theta; z) \;=\; \ell^{\mathrm{reg}}(\theta; z) \;+\; \lambda\, \ell^{\mathrm{fair}}(\theta;z) \;+\; \mu\, \ell^{\mathrm{pred}}(\theta;z)$,
with $\ell^{\mathrm{reg}}(\theta; z):=\mathrm{Regret}(f_\theta(X);\mathbf{r})$, $\ell^{\mathrm{fair}}(\theta;z) := F(f_\theta(X))$, and $\ell^{\mathrm{pred}}(\theta;z) := L(f_\theta(X))$. 
We denote the population loss as $v_0^{\mathrm{full}}(\theta):=\mathbb{E}_{z \sim \mathcal{P}}[\ell^{\mathrm{full}}(\theta; z)]$, and its empirical counterpart as $\hat v_0^{\mathrm{full}}(\theta):=\frac{1}{N}\sum_{s} \ell^{\mathrm{full}}(\theta; z_s)$. We analyze the empirical loss minimizer $\hat\theta\in\arg\min_{\theta\in\Theta}\hat v_0^{\mathrm{full}}(\theta)$, the target of FDFL-Scal training, by benchmarking it against $\theta^\ast\in\arg\min_{\theta\in\Theta}v_0^{\mathrm{full}}(\theta)$ through the composite \emph{excess risk} $R^{\mathrm{full}}(\hat\theta):=v_0^{\mathrm{full}}(\hat\theta)-v_0^{\mathrm{full}}(\theta^\ast)$.

We require two standard conditions in the analysis: the loss components are bounded and Lipschitz in $\theta$. Appendix~\ref{app:proofs-gen} verifies these conditions for the prediction losses and for the single-resource decision oracle, and gives a sufficient conic-nondegeneracy condition for decision-oracle Lipschitzness in the non-injective multi-resource setting.


\begin{assumption}[Parameter-space geometry] \label{assp:gen-geometry}
The parameter space $\Theta\subset\mathbb{R}^q$ is compact and contained in a Euclidean ball of radius $E_\Theta:=\inf_{\theta_0\in\mathbb{R}^q}\sup_{\theta\in\Theta}\|\theta-\theta_0\|$ centered at some $\theta_0\in\mathbb{R}^q$. 
\end{assumption}

\begin{assumption}[Boundedness of losses] \label{assp:gen-bounded}
The per-instance loss components are uniformly bounded: $|\ell^{\mathrm{reg}}(\theta;z)|\leq B_{\mathrm{reg}}$, $|\ell^{\mathrm{fair}}(\theta;z)|\leq B_F$, and $|\ell^{\mathrm{pred}}(\theta;z)|\leq B_L$ for all $(\theta,z)$.
\end{assumption}

\begin{assumption}[Lipschitzness of decision fairness measure and decision oracle]\label{assp:gen-oracle-lip}
There exists $L_W >0$ such that, for every instance $z$ and all $\mathbf d_1, \mathbf d_2$ in a convex set containing the oracle-reachable decisions, the decision fairness measure satisfies $\lvert W(\mathbf d_1;\mathbf r)-W(\mathbf d_2;\mathbf r)\rvert \leq L_W\,\|\mathbf d_1-\mathbf d_2\|$.
There exists $L_{\mathrm{or}}>0$ such that, for every instance $z$ and all $\theta_1,\theta_2\in\Theta$, $\|\mathbf{d}^*_z(f_{\theta_1}(X))-\mathbf{d}^*_z(f_{\theta_2}(X))\|\leq L_{\mathrm{or}}\,\|\theta_1-\theta_2\|$.
\end{assumption}

\begin{assumption}[Lipschitzness of prediction-related losses and predictor] \label{assp:gen-predictor-lip}
There exist $L_f, L_L,L_F>0$ such that, for every instance $z$ and all $\theta_1,\theta_2\in\Theta$, the predictor satisfies $\lVert f_{\theta_1}(X) - f_{\theta_2}(X) \rVert \leq L_f\|\theta_1-\theta_2\|$, and the prediction-related losses satisfy $\lvert \ell^{\mathrm{pred}}(\theta_1;z)-\ell^{\mathrm{pred}}(\theta_2;z) \rvert \leq L_L\|\theta_1-\theta_2\|$ and $\lvert \ell^{\mathrm{fair}}(\theta_1;z)-\ell^{\mathrm{fair}}(\theta_2;z) \rvert \leq L_F\|\theta_1-\theta_2\|$. 
\end{assumption}

Assumptions~\ref{assp:gen-oracle-lip} and~\ref{assp:gen-predictor-lip} are the analogs of the standard regularity conditions of the DFL literature \citep{elmachtoub2025dissecting,donti2017task}, where oracle Lipschitzness typically follows from strong convexity of the decision objective. Under all four assumptions, $\Theta$ is compact and both $\hat v_0^{\mathrm{full}}$, $v_0^{\mathrm{full}}$ are Lipschitz in $\theta$ (Appendix~\ref{app:proofs-gen}), so the minimizers $\hat\theta$ and $\theta^\ast$ exist.

\begin{theorem}[Composite-loss generalization for FDFL-Scal]\label{thm:gen-main}
Under Assumptions~\ref{assp:gen-geometry}--\ref{assp:gen-predictor-lip}, for any $\delta\in(0,1)$, there exists a constant $C_{\mathrm{abs}}$ such that, with probability at least $1-\delta$:
\begin{equation} \label{eq:gen-main-bound}
R^{\mathrm{full}}(\hat\theta) \;\leq\;
2\,C_{\mathrm{abs}}\,\Big(L_W\,L_{\mathrm{or}} + \lambda\,L_F + \mu\,L_L\Big)\,E_\Theta\,\sqrt{\tfrac{q}{N}}
\;+\;
2\,(B_{\mathrm{reg}} + \lambda\,B_F + \mu\,B_L)\sqrt{\tfrac{2\log(2/\delta)}{N}}.
\end{equation}
\end{theorem}

Theorem~\ref{thm:gen-main} is related to the generalization bound of \citet[Theorem~1]{elmachtoub2025dissecting}, which bounds the decision regret in a stochastic setting that fits a distribution to estimate unknown parameters. Because our predictor feeds a point estimate to the decision oracle instead (Remark~\ref{rem:deterministic}), we control the loss class through Lipschitzness of the decision map and Dudley chaining on the parameter space, whereas their analysis uses total variation between fitted distributions and a vector-contraction inequality. We incorporate prediction-quality components by standard generalization arguments for bounded Lipschitz loss classes. Theorem~\ref{thm:gen-main} is proved in Appendix ~\ref{app:proofs-gen}.

We note three limitations. First, the bound in \eqref{eq:gen-main-bound} is specific to a fixed scalarization with constants depending on the scalarization weights, and does not hold uniformly over all $\lambda,\mu$. Second, the complexity term $\sqrt{q/N}$ scales with the parameter count $q$ of the predictor, so the bound is informative in the under-parametrized regime (e.g., linear predictor). Once $q$ exceeds $N$, the bound becomes vacuous and sharper guarantees require different analysis.
Third, we do not examine whether the $\sqrt{q/N}$ rate can be improved or certified as optimal by a matching minimax lower bound. These refinements are possible directions for future work. 

\section{Numerical Experiments}
\label{sec:exp}
We evaluate FDFL on two resource allocation problems in which predicted benefits inform allocation under budget constraints. The first setup allocates a single healthcare resource among patients, using the patient-level data of \citet{doi:10.1126/science.aax2342}; the second setup allocates three resources among stakeholders, using synthetic data in which group imbalance is varied by construction. Section \ref{sec:exp-setup} specifies the experimental setups. Section \ref{sec:methods} describes all methods and the training configuration. Section \ref{sec:exp-results} presents experimental results and discusses key findings. 

\subsection{Experimental Setups}
\label{sec:exp-setup}
\subsubsection{Single Resource Allocation} \label{subsec:healthcare}
This setup is motivated by healthcare applications in which predictions about patients guide the allocation of care. An allocation instance $s$ is a set of patients considered together for one type of care resource in a single period (e.g., week, month): assigning one unit to patient $i$ incurs a known cost $c_i$ and delivers a benefit (impact) $r_i$, subject to budget $Q_s$. Patients are partitioned into $K=2$ groups by race.

\paragraph{Data and instances.}
We use the dataset of \citet{doi:10.1126/science.aax2342}, which contains $48784$ patients with features about demographics, comorbidities, prior care utilization, and biomarkers, among whom $5582$ ($11.4\%$) are Black and the remaining $43202$ are non-Black. We construct the ground-truth benefit $r_i$ from two features, the number of chronic conditions and the avoidable medical expenses from receiving care, and derive the allocation cost $c_i$ from the healthcare spending feature; Appendix~\ref{app:datagen} gives the construction details. 
Excluding the three features used, there are $147$ patient-level features for training a benefit predictor. We partition patients into training, validation, and test pools, each preserving the full dataset's Black-patient proportion. Instances of $m=5000$ patients are drawn from the corresponding pool. We train on $N=50$ instances and use $30$ validation instances for model selection and $30$ test instances for evaluation.

\paragraph{E2EFO tasks.}
The prediction task estimates each patient's benefit $\hat r_i$ from the $147$ features, evaluated by MSE and by the MAD \eqref{def:group-acc-parity} of the two group-level MSEs. Given the predicted benefits of an instance, the decision task allocates a continuous amount $d_i \geq 0$ of the resource to each patient, consuming $c_i d_i$ of the budget and yielding utility $u_i = \hat r_i d_i$. The allocation maximizes the group-based $\alpha$-fairness \eqref{eq:alpha-group-W} subject to the budget constraint $\sum_{i} c_i d_i \leq Q$, where the instance budget $Q$ equals $30\%$ of the instance's aggregate cost $\sum_i c_i$. Decision quality is measured by the regret \eqref{eq:regret} of $\mathbf{d}^*(\hat{\mathbf{r}})$ evaluated under the constructed benefits $\mathbf{r}$.
This decision problem is the single-budget allocation \eqref{eqsys:alpha-g} from Section \ref{sec:closed}, so its decision Jacobian is available in closed form (Proposition~\ref{prop:analytical_gradient_group-1}). 

\subsubsection{Multiple Resource Allocation} \label{subsec:md-knapsack}
We next consider allocation of $R=3$ resource types under a fully synthetic design that allows the inherent imbalance across groups to be controlled. In an instance, the three resources are distributed among $m=200$ stakeholders partitioned into $K$ equal-size groups, subject to one budget per resource; stakeholder $i$ gains benefit (impact) $r_{ij}$ and incurs known cost $c_{ij}$ per unit of resource $j$.

\paragraph{Data and instances.}
Each stakeholder has features $\mathbf{x}_i$ drawn from a $5$-dimensional standard Gaussian distribution and a group label $g_i \in \{0,\ldots,K-1\}$, with smaller $g_i$ denoting a more advantaged group. We construct benefits by combining a degree-two polynomial of the features, which captures the noiseless signal shared by all groups, with additive group biases and per-group noise. The costs are generated by adding biases and noises to a fixed feature-independent baseline. A single scalar $\ell \in [0,1]$ sets both the magnitude of the group biases and the noise inflation of the disadvantaged groups: at $\ell=0$, there is no imbalance and groups are exchangeable; increasing $\ell$ separates the group means while inflating the noise of the disadvantaged groups. Because the baseline signal-to-noise ratio is held fixed, imbalance is not confounded with signal strength. Appendix~\ref{app:datagen} gives the generation formulas. 
We generate a pool of $4000$ stakeholders and partition it into stakeholder-disjoint training ($65\%$), validation ($15\%$), and test ($20\%$) subpopulations. We sample instances of $m=200$ stakeholders without replacement from the corresponding subpopulation. We use $N=50$ training, $30$ validation, and $30$ test instances. 

\paragraph{E2EFO tasks.}
The prediction task estimates the per-resource benefits $\hat r_{ij}$, again evaluated by MSE and MAD. The decision task chooses $\mathbf{d} \in \mathbb{R}^{m \times R}_{\geq 0}$ maximizing the group-based $\alpha$-fairness \eqref{eq:alpha-group-W} of the utilities $u_i = \sum_j \hat r_{ij} d_{ij}$, subject to $\sum_i c_{ij} d_{ij} \leq Q_{sj}$ for each resource $j$, where $Q_{sj}$ is $35\%$ of the instance's aggregate cost for that resource. This model is DPP-compliant (Section~\ref{sec:cvxpylayers}), so its decision Jacobian is computed with cvxpylayers.

\subsection{Methods and Implementation}
\label{sec:methods}
All methods share the same downstream allocation model and differ only in how the predictor is trained. A method is specified by which of the three objectives -- decision regret, prediction accuracy (MSE), and prediction fairness (MAD) -- are included, and how their gradients are combined when more than one objective is included. Table~\ref{tab:methods} lists the method pool. The names in its first column are used throughout the paper, including in all figures.

\begin{table}[htbp]
\centering
\footnotesize
\caption{Method pool, grouped by the number of included training objectives. Entries give the weight on an included objective, or $\checkmark$ where an objective is included without a fixed weight. } \label{tab:methods}
\begin{tabular}{@{}llccc l@{}}
\toprule
Name & & Regret & MSE & MAD & Description \\
\midrule
PTO             & & --- & $1$ & --- & Predict-then-optimize baseline \\
SAA             & & --- & \checkmark & --- & Sample average approximation, feature-free baseline estimating $\hat{\mathbf{r}}$ with the  \\
 &&&&& mean impact of training samples instead of fitting a predictor from data
\\
WDRO            & & --- & \checkmark & --- & Wasserstein distributionally robust optimization, \\
 &&&&& robust regularization of the MSE objective \citep{gao2024wasserstein} \\
DFL             & & $1$ & --- & --- & Decision-focused learning baseline \\
\midrule
FPTO            & & --- & $1$ & $1$ & Prediction-fairness-embedded PTO \\
Regret-and-MAD  & & $1$ & --- & $1$ & Decision fairness with prediction fairness \\
Regret-and-MSE  & & $1$ & $\{0.1, 0.5, 1\}$ & --- & Decision fairness with prediction accuracy \\
\midrule
FDFL-Scal       & & $1$ & $\{0.1, 0.5, 1\}$ & $1$ & All three objectives; static scalarization \eqref{eq:fdfl-scal-rule} \\
FDFL-PCGrad     & & \checkmark & \checkmark & \checkmark & All three objectives; gradient projection \eqref{eq:pcgrad} \\
FDFL-NashMTL    & & \checkmark & \checkmark & \checkmark & All three objectives; Nash bargaining weights \eqref{eq:nashmtl} \\
FDFL-MGDA       & & \checkmark & \checkmark & \checkmark & All three objectives; minimum-norm combination \\
FDFL-FPLG        & & \checkmark & \checkmark & $\{0, 1\}$ & All three objectives; hybrid static-dynamic combination \\
\bottomrule
\end{tabular}
\end{table}

\paragraph{Training configuration.} 
We use three predictor architectures of increasing \emph{capacity}, which refers to their power to represent the feature-to-impact mapping:
a linear model, a neural network model with two 16-unit hidden layers (MLP-16), and a neural network model with two 64-unit hidden layers (MLP-64), all with a Softplus output layer enforcing nonnegative predictions. Predictors are trained with Adam, with the learning rate and stopping step selected on the validation split, by validation MSE for methods without the regret objective and by validation regret for all other methods that include the regret objective. Each configuration is run with five seeds, and we report means and standard deviations across seeds. 
In both setups, we vary the decision fairness parameter $\alpha \in \{0.5, 1.5, 2\}$ and the number of training instances $N \in \{10, 20, 50\}$. In the synthetic multi-resource setup, we additionally vary the imbalance level $\ell \in \{0, 0.2, 0.4, 0.6, 0.8\}$ and the number of groups $K \in \{2,4\}$.

\paragraph{Evaluation.} We report three performance metrics on test instances: prediction MSE, the MAD of group prediction errors, and normalized decision regret $\mathrm{Regret}(\hat{\mathbf r};\mathbf r) / \lvert W(\mathbf{d}^*(\mathbf r);\mathbf r)\rvert$. 

\paragraph{Implementation environment.} All methods are implemented in PyTorch, with automatic differentiation computing every component of the gradient chain except the decision Jacobian. In the single-resource setup, the decision model is both solved and differentiated using the closed-form formulas; in the multi-resource setup, the decision model is solved by MOSEK (with CLARABEL/SCS fallbacks) and differentiated with cvxpylayers. All experiments run on a Google Colab standard CPU / high-RAM runtime. 
All data, code and results are available in \url{https://github.com/DennisWang2488/fair-dfl}.

\subsection{Results and Findings} \label{sec:exp-results}
We report the test performance of selected methods
in Table~\ref{tab:sec53-hc-main} and 
Table \ref{tab:sec53-md-main}. The full results, covering all methods and the complete grid over predictor capacity, $\alpha$, $\ell$, $N$, $K$, $\mu$ (the weight on the MSE objective) appear in Appendix~\ref{app:extended}.

\input{Tables/tab_sec53_hc_main}
\input{Tables/tab_sec53_md_main}

Our results verify the expected pattern that including an objective in training improves the metric it targets: the decision-regret objective lowers regret relative to prediction-focused training and each prediction-side objective improves the prediction metric it captures, demonstrating the value of multi-objective training (Section~\ref{exp-subsec:value-fdfl}). We further observe complementary roles of the two fairness objectives: 
the prediction fairness objective reduces prediction disparity more effectively under larger group imbalances (Section~\ref{exp-subsec:imbalance}), while the decision regret objective leads to greater regret advantages under weaker predictors and larger $\alpha$ (Section~\ref{exp-subsec:predictor-alpha}). Robustness of these findings is discussed in Appendix ~\ref{app:robustness}.



\subsubsection{The Value of Multi-Objective Training in FDFL} \label{exp-subsec:value-fdfl}
In the single resource setup at $\alpha=2$ with an MLP-64 predictor, Table~\ref{tab:sec53-hc-main} shows that training on decision regret alone (DFL) attains the lowest regret in the method pool, 
but the highest prediction disparity (MAD $71.0$) and the lowest prediction accuracy (MSE $270.3$). 
Adding the prediction-fairness objective reduces the disparity at no decision cost. Regret-and-MAD lowers MAD to $27.2$ while holding regret at $0.129$. 
In this setup, the prediction accuracy objective has the same effect: Regret-and-MSE reaches a MAD of $27.3$ without targeting disparity. Moreover, the FDFL methods, which integrate all three objectives, closely track the performance of Regret-and-MSE. These observations show that the contributions of the two prediction objectives cannot be separated in this single-resource allocation problem. 

In the multiple resource setup, we observe a clear separation between prediction accuracy and prediction fairness.
We examine the results at imbalance $\ell=0.6$ with $\alpha=2$ and an MLP-64 predictor (Table~\ref{tab:sec53-md-main}). Adding the prediction fairness objective to PTO reduces MAD from $0.228$ to $0.142$ in FPTO. Compared to DFL with MAD $0.176$, adding prediction fairness lowers MAD to $0.102$ in Regret-and-MAD, while adding prediction accuracy only lowers MAD to $0.150$ in Regret-and-MSE. Therefore, prediction fairness is not redundant with prediction accuracy. Section~\ref{exp-subsec:imbalance} shows that the prediction fairness advantage is more salient at higher imbalance between groups. 
We also note that PTO and DFL attain similarly low regrets ($0.215$ and $0.216$) in this setting;
Section~\ref{exp-subsec:predictor-alpha} identifies when a decision-focused regret advantage emerges.

\subsubsection{Group Imbalance} \label{exp-subsec:imbalance}
Table~\ref{tab:sec53-md-main} reports performance at varying imbalance levels $\ell$ in the multiple resource setting with $\alpha=2$, an MLP-64 predictor and $K=2$ groups. 
Recall from Section~\ref{subsec:md-knapsack} that as $\ell$ rises, the group means separate more and the disadvantaged group (represented by $g=1$) experiences larger noise inflation, which leads to greater disparities intrinsic to the allocation. The PTO performance illustrates this pattern: as $\ell$ increases from $0$ to $0.8$, MAD rises steeply from $0.042$ to $0.311$, while regret rises modestly from $0.135$ to $0.216$.

The prediction-fairness gain widens monotonically with imbalance in both training paradigms. Among prediction-focused methods, the MAD reduction from PTO to FPTO grows from $0.006$ (the gap between $0.078$ and $0.072$) at $\ell=0.2$ to $0.137$ at $\ell=0.8$; among decision-focused methods, the MAD reduction from DFL to Regret-and-MAD grows from $0.010$ at $\ell=0$ to $0.261$ at $\ell=0.8$. The right panel of Figure \ref{fig:sec53-md-imbalance} visualizes this trend. 

\begin{figure}[htbp]
\centering
\includegraphics[width=0.9\textwidth]{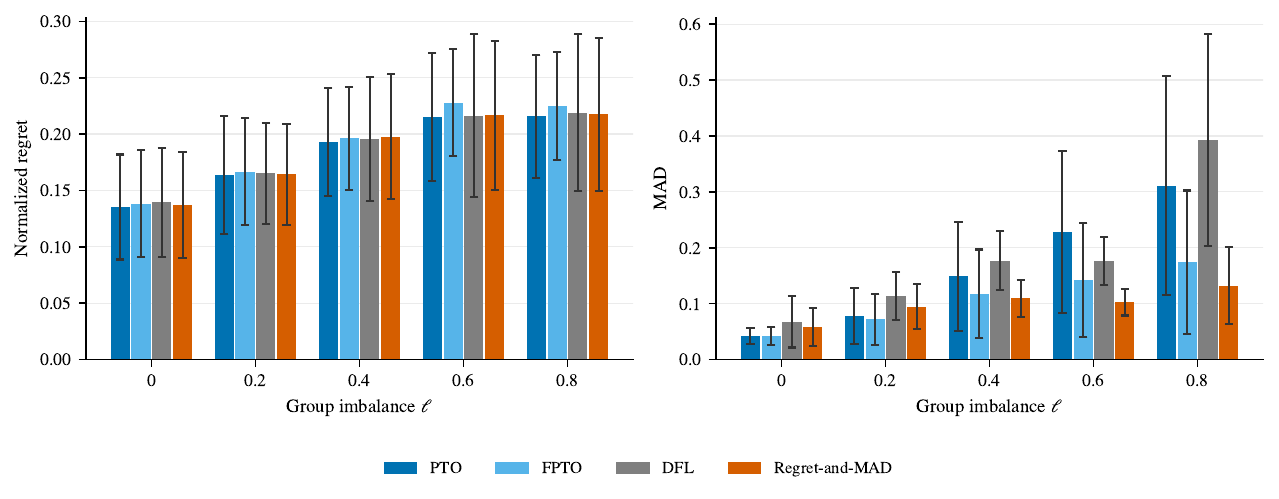}
\caption{Normalized regret (left) and MAD (right) across imbalance levels $\ell$ and $\alpha=2$, MLP-64, $K=2$ in multiple resource allocation setup. 
Bars are means over five seeds with 95\% confidence intervals. 
}
\label{fig:sec53-md-imbalance}
\end{figure}

The decision-regret objective has stable impacts across imbalance levels: with the MLP-64 predictor configuration, Table \ref{tab:sec53-md-main} shows similar regrets from all feature-based methods (only SAA is excluded) at every $\ell$. With our chosen seeds, decision-focused methods can improve over the PTO baseline regret slightly, specifically, FDFL-NashMTL attains the lowest mean regret at four of the five imbalance levels and FDFL-PCGrad at $\ell=0.4$, but every improvement margin falls within the across-seed standard deviation. 

Therefore, in our setting, imbalance shifts the prediction fairness advantage without affecting the regret comparison among methods. With the MLP-64 predictor, prediction accuracy and decision quality are aligned, leaving little regret gap between accuracy-only training and the decision-focused methods. In the next subsection, we demonstrate that the predictor capacity and the decision fairness parameter $\alpha$ shape the decision regret advantage. 

\subsubsection{Predictor Capacity and Decision Fairness Parameter} \label{exp-subsec:predictor-alpha}
Whether prediction accuracy and decision quality align depends on how well the predictor represents the feature-to-impact mapping and how sharply the allocation reacts to prediction errors. The predictor capacity affects the former, while the $\alpha$ choice affects the latter: 
at larger $\alpha$, the $\alpha$-fairness objective places more emphasis on prioritizing low-utility recipients, and small mispredictions can change who is prioritized. 

We compare three predictor types under $\alpha = 2$ in the single resource setting in Table~\ref{tab:sec53-hc-main}. PTO's MSE falls from $136.5$ with a linear predictor to $118.9$ with an MLP-64 predictor, consistent with the capacity ordering \footnote{The across-seed mean MSE of PTO method with an MLP-16 predictor is inflated by high across-seed variance.}. The regret gap between PTO and the best decision-focused method contracts as capacity grows: regret reduction of $32.6\%$ with a linear predictor ($0.193$ to $0.130$), $16.8\%$ with an MLP-16 predictor ($0.155$ to $0.129$), and $7.2\%$ with an MLP-64 predictor ($0.139$ to $0.129$). We observe similar regret-gap contracting patterns in multiple resource allocation results from Table~\ref{tab:app-md-capacity} in Appendix~\ref{app:extended}. 
With limited predictor capacity, MSE-only training leaves prediction errors that are large enough to carry over to affect decision quality, creating the need for a decision regret objective to mitigate such error propagation. Notably, the regret objective alone may be insufficient at low capacity: with the linear predictor, DFL attains a worse regret than PTO (Table ~\ref{tab:sec53-hc-main}). 

We next examine the role of $\alpha$ under a fixed MLP-64 predictor in both allocation settings using the results in Tables~\ref{tab:app-hc-alpha} and~\ref{tab:app-md-alpha} in Appendix~\ref{app:extended}. 
At $\alpha=2$, the decision-focused methods reduce regret by $7.2\%$ (single resource setting, Table \ref{tab:sec53-hc-main}) and $2.8\%$ (multiple resource setting at $\ell=0.6$, Table \ref{tab:sec53-md-main}) relative to PTO. When $\alpha$ is $0.5$ or $1.5$, the regret advantage disappears, and some decision-focused methods deliver worse regrets than PTO: the allocation responds weakly to prediction errors, so an accurate predictor already induces low-regret decisions. At larger $\alpha$, the allocation reacts more sharply to prediction errors, creating large decision regrets without an explicit decision objective.

Our findings show that the two fairness objectives are useful in complementary regimes. The decision-regret advantage emerges under a capacity-limited predictor and a sufficiently large $\alpha$, which cause prediction errors to propagate into decision regret, while the prediction-fairness advantage is driven by group imbalance. Training with the full FDFL objective is therefore desirable to cover both regimes. 
Appendix~\ref{app:robustness} reports supporting analyses: the effect of the prediction accuracy weight $\mu$, a comparison of gradient-combination rules, and confirmation that these findings persist at $\alpha=0.5$, across $N \in \{10,20,50\}$, and at $K=4$.

\section{Conclusion and Discussion}\label{sec:conclusion}
This work studies E2EFO as a unifying framework for fair prediction-informed decision-making, focusing on resource allocation among groups of stakeholders.
We measure prediction disparity by the MAD of prediction errors across groups and decision fairness by the group-based $\alpha$-fairness of allocation utilities. To train a predictor, we propose FDFL, a family of gradient-based algorithms that treat the three objectives---prediction error, prediction disparity, and decision regret---as a multi-task learning problem.
The challenging gradient component, the decision Jacobian, is computed exactly from the closed-form solution we derive for the single-budget $\alpha$-fair allocation and numerically through cvxpylayers for general convex feasible sets. We also establish a finite-sample generalization bound for the scalarized FDFL objective, with constants verified for the allocation tasks.

We evaluate a pool of FDFL algorithms on 
a single-resource healthcare allocation 
and a synthetic multiple-resource allocation where group imbalance is varied by construction. 
Experimental results show that training on decision regret alone attains low regret but can leave large disparity in prediction errors; adding a prediction-side objective removes most of this disparity at no observed regret cost. 
The two fairness objectives serve complementary roles: the prediction-fairness objective reduces prediction disparity more effectively as group imbalance grows, while the decision-regret objective yields greater regret advantages under weaker predictors and larger $\alpha$. Because either regime can arise in practice, these results support training on the full FDFL objective and monitoring all three metrics in deployment for end-to-end fairness.

Several directions remain for future work. We restrict the scope of this work to specific fairness definitions and a deterministic resource allocation as the decision task. Future research could explore E2EFO and FDFL with alternative fairness criteria, in other decision contexts, or under the stochastic view of the decision task (Remark~\ref{rem:deterministic}). On the empirical side, our experiments rely on constructed outcomes: the single resource setup builds benefits from patient records, and the multi-resource setup is fully synthetic. Formulating E2EFO problems with observed real-world outcomes and examining FDFL performance in such settings are natural next steps. On the theoretical side, the generalization bound holds for a fixed scalarization, but it does not explain when decision-focused training outperforms prediction-focused training under fairness considerations or characterize the Pareto frontier the dynamic combinations trace.

\bibliographystyle{plainnat}
\bibliography{sample-base} 

\appendix
\section{Closed-Form Decisions and Decision Jacobian} \label{app:proofs}
\begin{proposition} [Decision Jacobian]
\label{prop:analytical_gradient_group-1}
Let $\mathbf{d}^*$ denote the optimal solution to \eqref{eqsys:alpha-g} with parameter $\mathbf{r}$, and define $\beta = \begin{cases} \frac{1}{\alpha-2}, \text{ if } 0 < \alpha < 1; \\
\frac{-\alpha+2}{-\alpha^2 + 2\alpha-2}, \text{ if } \alpha > 1
\end{cases}$. For $k \in [K]$, $S_k = \sum_{i \in G_k} \left(c_i^{-1/\alpha}r_i^{1/\alpha}\right)^{1-\alpha}$. For $i \in [m]$, $k(i)$ denote the group it belongs to. 

The Jacobian $\frac{\partial \mathbf{d^*(\mathbf{r})}}{\partial \mathbf{r}}$ consists of the following components: for all $i \in [m]$,
\begin{itemize}
    \item  Derivative of decision $d_i$ with respect to $i$'s own impact:
\[ \frac{\partial d_i^*}{\partial r_i} = \frac{1-\alpha}{\alpha}Qc_i^{-\frac{1}{\alpha}}r_i^{\frac{1}{\alpha}-2} \frac{S_{k(i)}^{\beta}}{\sum_{k=1}^K S_k^{1+\beta}} + Qc_i^{-\frac{1}{\alpha}}r_i^{\frac{1}{\alpha}-1} \frac{\partial S_{k(i)}}{\partial r_i} \left( \frac{\beta S_{k(i)}^{\beta-1}}{\sum_{k=1}^K S_k^{1+\beta}} - \frac{(1+\beta)S_{k(i)}^{2\beta}}{(\sum_{k=1}^K S_k^{1+\beta})^2}\right).  \]
     \item  Derivative of decision $d_i$ with respect to the impact of another member $\ell$ of the same group:
\[
\frac{\partial d_i^*}{\partial r_\ell}
= Qc_i^{-\frac{1}{\alpha}}r_i^{\frac{1}{\alpha}-1}\frac{\partial S_{k(i)}}{\partial r_\ell}
\left( \frac{\beta\, S_{k(i)}^{\,\beta - 1}}{\sum_{k=1}^K S_k^{1+\beta}} - \frac{(1+\beta)\, S_{k(i)}^{\,2\beta}}{(\sum_{k=1}^K S_k^{1+\beta})^2} \right), ~\forall \ell \neq i, \ell \in G_{k(i)}.
\]
     \item Derivative of decision $d_i$ with respect to the impact of a member $j$ of a different group: 
     \[
     \frac{\partial d_i^*}{\partial r_j} = -Qc_i^{-\frac{1}{\alpha}}r_i^{\frac{1}{\alpha}-1}S_{k(i)}^{\beta}\frac{\partial S_{k(j)}}{\partial r_j} \frac{(\beta+1)S_{k(j)}^{\beta}}{(\sum_{k=1}^K S_k^{1+\beta})^2}, ~\forall j \notin G_{k(i)}. 
     \]
\end{itemize}
In these formulas, the gradient of $S_k$ is given by
\[
\frac{\partial S_k}{\partial r_i} = \frac{1-\alpha}{\alpha}c_i^{1-\frac{1}{\alpha}}r_i^{\frac{1}{\alpha}-2}, ~\forall i \in G_k;
~ \frac{\partial S_k}{\partial r_i} = 0, ~\forall i \notin G_k.
\]
\end{proposition}

\begin{proposition} [Closed-Form Decisions and Decision Jacobian, Special Cases of $\alpha$]
\label{prop:closed_form_solution_group_special}\label{prop:analytical_gradient_special_group}
For special cases of $\alpha = 0, 1, \infty$, the optimal solution $\mathbf{d}^*$ to \eqref{eqsys:alpha-g} with parameter $\mathbf{r}$ and its Jacobian $\frac{\partial \mathbf{d^*(\mathbf{r})}}{\partial \mathbf{r}}$ are, respectively, given by:
\begin{itemize}
    \item $\alpha = 0$: $d_i^* = \frac{Q}{c_i}$ if $i = \arg\max_j \frac{r_j}{c_j}$ and $d_i^* = 0$ otherwise, allocating the entire budget to the stakeholder with the highest reward-to-cost ratio. $\mathbf{d^*(\mathbf{r})}$ is not differentiable at the maximizer $i^* = \arg\max_j \frac{r_j}{c_j}$, and has $\frac{\partial d_i^*}{\partial r_j} = 0$ for all $i \neq i^*$ and all $j$.
    \item $\alpha = 1$: 
    $d_i^* = \frac{Q}{mc_i}, ~ \forall i \in [m]$; $\frac{\partial d_i^*}{\partial r_k} = 0$ for all $i,k \in [m]$ as $\mathbf{d}^*(\mathbf{r})$ does not depend on $\mathbf{r}$.
    \item $\alpha \to \infty$: $d_i^* = \frac{Q}{r_i \sum_{j=1}^m \frac{c_j}{r_j}}, ~ \forall i \in [m]$, equalizing the utilities of all stakeholders, with $\frac{\partial d_i^*}{\partial r_k} = \begin{cases} -\frac{Q}{r_k^2(\sum_{j=1}^m c_j/r_j)} + \frac{Qc_k}{r_k^3 (\sum_{j=1}^m c_j/r_j)^2}, & \text{ if } i=k; \\ \frac{Qc_k}{r_i r_k^2 (\sum_{j=1}^m c_j/r_j)^2}, & \text{ if } i \neq k\end{cases}$.
\end{itemize}
\end{proposition}

\textbf{Proof for closed-form decisions and decision Jacobian:}

\begin{itemize}
    \item \textbf{Special case $\alpha = 0$: } The objective function simplifies to the linear utility sum, $W_{\alpha}^g(\mathbf{u}) = \sum_{k=1}^K \sum_{i \in G_k} r_i d_i$.
    Maximizing a linear function under a knapsack constraint yields a greedy solution: allocate the full budget to the stakeholder with the highest reward-to-cost ratio, $\frac{r_i}{c_i}$. Therefore, the optimal solution assigns $d^*_i = \frac{Q}{c_i}$ if $\frac{r_i}{c_i} = \max_{j \in [m]} \{\frac{r_j}{c_j}\}$, and $d^*_i = 0$ for all other stakeholders. 

    This function $\mathbf{d}^*(\mathbf{r})$ is piecewise constant, so its derivative is zero everywhere except at points where the identity of $i^*$ changes, at those non-differentiable points the zero matrix is a valid subgradient selection for optimization algorithms.

    \item \textbf{General case $\alpha > 0$ and $\alpha \neq 1$:}  When $\alpha > 0$, the objective function is concave in $\mathbf{d}$, and we need a different strategy to conclude the optimal solution. For notation ease, we denote $W_{\alpha}^g(\mathbf{u})$ with $W(g(\mathbf{d}))$, where $g(\mathbf{d}) = (g_1(\mathbf{d}),\ldots,g_K(\mathbf{d}))$ computes all $g_k(\mathbf{u})$ as defined in \eqref{eq:alpha-group-g}. The Lagrangian function of \eqref{eqsys:alpha-g} is
$\mathcal{L}(\mathbf{d},\nu,\boldsymbol{\xi}) = -W(g(\mathbf{d})) + \nu(\sum_{i=1}^m c_i d_i - Q) - \sum_{i=1}^m \xi_i d_i$,
where $\nu \in \mathbb{R}_{\geq 0}$ and $\xi \in \mathbb{R}_{\geq 0}^m$ are Lagrangian multipliers. The Karush-Kuhn-Tucker (KKT) stationarity condition requires $\frac{\partial \mathcal{L}}{\partial d_i} = -\frac{\partial W}{\partial g_k} \cdot \frac{\partial g_k}{\partial d_i} + \nu c_i - \xi_i = 0$ for each $i \in G_k$ (by the chain rule), with $\frac{\partial W}{\partial g_k} = g_k(\mathbf{d})^{-\alpha}$. 

We compute $\frac{\partial g_k}{\partial d_i} = A_k\, r_i^{1-\alpha}d_i^{-\alpha}$ 
with $A_k = 1$ for $0 < \alpha \leq 1$ and $A_k = \bigl(\frac{\alpha-1}{\sum_{i \in G_k} (r_id_i)^{1-\alpha}}\bigr)^2$ for $\alpha > 1$. Complementary slackness ($\xi_i d_i = 0$) requires $\xi_i = 0, d_i > 0$ or $\xi_i > 0, d_i = 0$, and the latter conflicts with stationarity; hence $\xi_i = 0$ for all $i$, and solving the stationarity condition $g_k^{-\alpha} A_k\, r_i^{1-\alpha}d_i^{-\alpha} = \nu c_i$ for $d_i$ gives, with $k(i)$ the group index of $i$,
\begin{equation} \label{pf-group:d-eq-2}
    d_i = D_{k(i)}\, c_i^{-\frac{1}{\alpha}}r_i^{\frac{1-\alpha}{\alpha}},
\qquad
D_k := \bigl(\nu g_k^{\alpha}/A_k\bigr)^{-1/\alpha},
\end{equation}
where $D_k$ collects the factors common to all members of group $k$.

Substituting \eqref{pf-group:d-eq-2} into the definition of $g_k(\mathbf{d})$ gives $g_k = \frac{1}{1-\alpha}D_k^{1-\alpha}S_k$ for $0<\alpha<1$ and $g_k = (\alpha-1)(D_k^{1-\alpha}S_k)^{-1}$ for $\alpha>1$, where $S_k := \sum_{i \in G_k} (c_i^{-1/\alpha} r_i^{1/\alpha})^{1-\alpha}$. Substituting these, together with the regime-specific $A_k$, back into the definition of $D_k$ in \eqref{pf-group:d-eq-2} and rearranging (the resulting exponents $\alpha^2-2\alpha$ and $-\alpha^2+2\alpha-2$ are nonzero on the respective ranges of $\alpha$) yields, in both regimes, the common form $D_k = \nu^{e} P_k$, where the exponent $e$ and the group prefactor $P_k$ depend only on $\alpha$ and $S_k$:
\[
(e,\,P_k) =
\begin{cases}
\Bigl(\dfrac{1}{\alpha^2-2\alpha},\ \bigl(\tfrac{S_k}{1-\alpha}\bigr)^{\alpha e}\Bigr), & \text{if } 0<\alpha<1;\\[6pt]
\Bigl(\dfrac{1}{-\alpha^2+2\alpha-2},\ \bigl(\tfrac{S_k}{\alpha-1}\bigr)^{(2-\alpha)e}\Bigr), & \text{if } \alpha>1.
\end{cases}
\]

Lastly, we plug the $d_i$ formula \eqref{pf-group:d-eq-2} into the budget constraint $\sum_{i=1}^m c_i d_i = Q$. For each group $k$, recall $S_k = \sum_{i \in G_k} c_i^{\frac{\alpha-1}{\alpha}} r_i^{\frac{1-\alpha}{\alpha}}$, so $\sum_{i \in G_k} c_id_i = D_k S_k$, and the constraint eliminates $\nu$ in one step, identically in both regimes:
\[
\sum_{k=1}^K D_k S_k = \nu^{e}\sum_{k=1}^K S_k P_k = Q
\;\Rightarrow\; D_k = \frac{Q\, P_k}{\sum_{m=1}^K S_m P_m}
\;\Rightarrow\; d_i = \frac{Q\, P_{k(i)}\, c_i^{-\frac{1}{\alpha}}r_i^{\frac{1-\alpha}{\alpha}}}{\sum_{m=1}^K S_m P_m},
\]
which, after substituting the regime-specific $P_k$, is the stated closed form.

The decision Jacobian formula in Proposition~\ref{prop:analytical_gradient_group-1} follow from direct differentiation of the decision formulas. 

    \item \textbf{Special case $\alpha=1$:} With $h_k = \prod_{i \in G_k} u_i$ from \eqref{eq:alpha-group-g}, the outer aggregation gives $W_1^g(\mathbf{u}) = \sum_{k \in [K]} \log h_k = \sum_{i=1}^m \log u_i = W_1(\mathbf{u})$, so the objective is $\max \sum_{i=1}^m \log(r_i d_i)$, whose logarithmic nature ensures an interior solution $d_i^* > 0$. The KKT stationarity condition $\frac{1}{d_i} - \nu c_i = 0$ implies $d_i = 1/(\nu c_i)$, and the budget constraint gives $\sum_i c_i/(\nu c_i) = m/\nu = Q$, so $d_i^* = Q/(mc_i)$, which is independent of $\mathbf{r}$ with zero gradient.
    \item \textbf{Special case $\alpha \to \infty$:}  The objective is $\max \min_k g_k(\mathbf{d})$, and for large $\alpha$ each harmonic-mean-like $g_k(\mathbf{d})$ is dominated by the smallest utility $r_id_i$ within the group, so the optimum equalizes all individual utilities across groups: $r_i d_i = C$, i.e., $d_i = C/r_i$. The budget constraint gives $\sum_i c_i (C/r_i) = C \sum_i (c_i/r_i) = Q$, so $C = Q / \sum_j (c_j/r_j)$ and $d_i = \frac{Q}{r_i \sum_j(c_j/r_j)}$. The decision Jacobian formula follows from direct differentiation of the decisions.  
\end{itemize}

\subsection{Individual-Based Fairness}
The individual-based counterpart of \eqref{eqsys:alpha-g} replaces the group welfare $W_{\alpha}^g$ with the individual measure $W_{\alpha}$ of \eqref{eq:alphafairness}:
\begin{equation} \label{eqsys:alpha-ind}
    \max_{\mathbf{d}} ~  W_{\alpha} (\mathbf{u}) \quad
    \text{ s.t. } u_i = \hat{r}_id_i, ~ d_i \geq 0, ~ \forall i; \sum_{i=1}^m c_id_i \leq Q.
\end{equation}
\begin{proposition} [Individual-based Closed-Form Decisions and Gradients for Special Cases of $\alpha$]
\label{prop:closed_form_solution_special}\label{prop:analytical_gradient_special}
For the individual-based problem \eqref{eqsys:alpha-ind} with parameter $\mathbf{r}$, the optimal solutions and Jacobian for $\alpha = 0$ and $\alpha \to \infty$ coincide with the group-based formulas of Proposition~\ref{prop:closed_form_solution_group_special}. For $\alpha = 1$, the optimal solution is $d_i^* = \frac{Q}{mc_i}$ for all $i \in [m]$, with $\frac{\partial d_i^*}{\partial r_k} = 0$ for all $i,k \in [m]$ as $\mathbf{d}^*(\mathbf{r})$ does not depend on $\mathbf{r}$.
\end{proposition}

\textbf{Proof for closed-form decisions:}
For $\alpha = 0$ and $\alpha \to \infty$, the individual-based objective coincides with the group-based one---both reduce to $\max_{\mathbf{d}} \sum_i r_i d_i$ and to $\max_{\mathbf{d}} \min_i r_i d_i$, respectively---so the group-based derivations and gradients apply verbatim, and it remains to treat the following two cases.
\begin{itemize}
    \item \textbf{General case $\alpha >0 $ and $\alpha \neq 1$:} The Lagrangian of problem \eqref{eqsys:alpha-ind} is
$\mathcal{L}(\mathbf{d}, \nu, \boldsymbol{\xi}) = \sum_{i=1}^m \frac{(r_id_i)^{1-\alpha}}{1-\alpha} - \nu (\sum_{i=1}^m c_i d_i - Q) + \sum_{i=1}^m \xi_i d_i$.
For any $\alpha > 0$, the objective penalizes zero allocations with either an infinite marginal utility ($0 < \alpha < 1$) or an infinite penalty ($\alpha \geq 1$), ensuring an interior solution $d_i^* > 0$ given $Q, c_i, r_i > 0$; by complementary slackness, $\xi_i=0$ for all $i$. The KKT stationarity condition $\frac{\partial \mathcal{L}}{\partial d_i}=0$ reads
$(r_i)^{1-\alpha} d_i^{-\alpha} - \nu c_i = 0$, i.e., $d_i = (\nu c_i)^{-1/\alpha} r_i^{\frac{1-\alpha}{\alpha}}$.
At the optimum, the budget constraint holds with equality. Substituting $d_i$ into the constraint gives
\[
\sum_{j=1}^m c_j \left( (\nu c_j)^{-1/\alpha} r_j^{\frac{1-\alpha}{\alpha}} \right) = \nu^{-1/\alpha} \sum_{j=1}^m c_j^{1-\frac{1}{\alpha}} r_j^{\frac{1-\alpha}{\alpha}} = Q,
\]
so $\nu^{-1/\alpha} = Q\big/\sum_{j=1}^m c_j^{1-\frac{1}{\alpha}} r_j^{\frac{1-\alpha}{\alpha}}$, and substituting this back into the expression for $d_i$ gives the closed form solution formula
\[
d_i = \frac{c_i^{-\frac{1}{\alpha}} \cdot r_i^{\frac{1}{\alpha} - 1} \cdot Q}{\sum_{j=1}^m c_j^{1-\frac{1}{\alpha}} \cdot r_j^{\frac{1}{\alpha} - 1}}, \quad \forall i \in [m].
\]
The corresponding Jacobian follows from direct differentiation of this closed form.

\item \textbf{Special case $\alpha=1$:} The objective becomes $\max \sum_{i=1}^m \log(r_i d_i)$, whose logarithmic nature ensures an interior solution $d_i^* > 0$. The KKT stationarity condition $\frac{1}{d_i} - \nu c_i = 0$ implies $d_i = 1/(\nu c_i)$, and the budget constraint gives $\sum_i c_i/(\nu c_i) = m/\nu = Q$, so $d_i^* = Q/(mc_i)$, which is independent of $\mathbf{r}$ with zero gradient.
\end{itemize}

\section{Additional Experiment Details} \label{app:details}

\subsection{Data Generation Details} \label{app:datagen}

\subsubsection{Single Resource Allocation.}
We construct a benefit score $r_i$ for each patient $i$ by combining two components: (1) potential for health improvement, and (2) cost savings.
We treat a patient's number of active chronic illnesses as a proxy for the health improvement potential, denoted $h_i$, reflecting that the care resource helps patients manage chronic conditions. We apply min-max normalization to scale the chronic illness count from dataset to $[0,1]$.
Cost savings, denoted $s_i$, are estimated from avoidable healthcare costs from program enrollment. The original avoidable costs range from \$0 to \$642{,}700. To reduce skewness, we apply the log transformation $\log(1+s_i)$ adopted in \cite{doi:10.1126/science.aax2342} followed by min-max normalization to scale the values to $[0,1]$.
The unscaled benefit score is computed as a weighted average, $b_i = 0.5 h_i + 0.5 s_i$. To ensure numerical stability, we rescale these values to lie in the interval $[2, 101]$ as the ground-truth benefit score $r_i = \max\{100\, b_i, 1 \}+1$. The allocation cost $c_i$ is computed from $c_i = \max \{10\,\tilde c_i, 1 \}$, where $\tilde c_i$ is the min-max normalized healthcare spending for patient $i$ from the dataset.

\subsubsection{Multiple Resource Allocation.}
For each stakeholder $i$ with features $\mathbf{x}_i \sim \mathcal{N}(\mathbf{0}_d, I_d)$ and group label $g_i \in \{0,\ldots,K-1\}$, we generate ground-truth benefits by combining a noiseless feature-to-benefit signal, group bias, and per-group noise. The noiseless benefit signal is a degree-$P$ polynomial of the features, $f(x) = \sum_{p=1}^{P} \bigl(\mathbf{x}^{\odot p}\bigr)\,W_p$ with $W_p \in \mathbb{R}^{d \times R}$, $(W_p)_{kj} \sim \mathcal{N}(0, 1/p^2)$ and $P=2$, where $\mathbf{x}^{\odot p}$ denotes the element-wise $p$-th power. This noiseless prediction target is nonlinear in $\mathbf{x}$. For resource type $j$, we fix its standard deviation in the associated benefit signal $\sigma_j$ and set the baseline benefit-noise scale as $\eta_j = \frac{\sigma_j}{\sqrt{\mathrm{SNR}}}$ with $\mathrm{SNR}{=}5$. The benefit is constructed as:
\[
  r_{ij} = \operatorname{softplus}\bigl( f_j(\mathbf{x}_i) + \rho(g_i)\,\beta_b + \varepsilon_{ij} \bigr) + 0.05,
  \quad \varepsilon_{ij} \sim \mathcal{N}\bigl(0, (\eta_j\, s_b(g_i))^2\bigr),
\]
where $ \rho(g_i)\,\beta_b$ is the additive benefit bias for group $i$, which applies a scaling function $\rho(g_i)$ to the benefit bias factor $\beta_b$, and $\varepsilon_{ij}$ is the noise accounting for group-specific scaling factor $s_b(g_i)$. We define $\rho(g) = 1 - 2g/(K-1)$ and $s_b(g) = 1 + (\nu_b-1)\,g/(K-1)$. The parameters $\beta_b$ and $\nu_b$ are explained below with the imbalance parameter $\ell$.

We generate costs in a similar structure by combining a noiseless signal with group-specific bias and noise. The noiseless cost signal is feature-independent and fixed as $\mu_c$. We generate the cost $c_{ij}$ as:
\[
  c_{ij} = \max\!\bigl( \mu_c + \rho(g_i)\,\beta_c + \zeta_{ij}\, s_c(g_i),\; 10^{-3} \bigr),
  \quad \zeta_{ij} \sim \mathcal{N}(0, \sigma_c^2),
\]
with $\mu_c{=}1.0$, $\sigma_c{=}0.2$, cost group-bias $\beta_c$, and per-group noise scale $s_c(g) = 1 + (\nu_c-1)\,g/(K-1)$. The parameters $\beta_c$ and $\nu_c$ depend on the imbalance parameter $\ell$.

We use one parameter $\ell \in [0,1]$ to adjust the per-group scaling on benefit and cost bias and noise. The scaling is controlled via the four imbalance factors $\beta_b = \beta_c = 0.9\,\ell$ and $\nu_b = \nu_c = 1+\ell$. At $\ell=0$, the groups are exchangeable; increasing $\ell$ shifts the group means apart and inflates the disadvantaged group's noise on both benefit and cost. Because the signal-to-noise ratio is fixed independently of $\ell$, imbalance is decoupled from overall signal strength.

\subsection{Additional Methods in the Method Pool} \label{app:method-details}
This appendix describes the methods not specified in the main text. The training objective gradients are $\mathbf g_{\mathrm{dec}}$, $\mathbf g_{\mathrm{pred}}$, $\mathbf g_{\mathrm{fair}}$.

\paragraph{FDFL-FPLG.} This method adapts the prediction-loss guided (PLG) DFL of \citet{jeon2025plg} to include a prediction fairness objective. We combine $\mathbf{g}_{\mathrm{pred}}$ and $\mathbf{g}_{\mathrm{dec}}$ dynamically using the rule from \citet{jeon2025plg} and includes $\mathbf{g}_{\mathrm{fair}}$ with a fixed weight as in FDFL-Scal. Let $\hat{\mathbf u}_\bullet:=\mathbf g_\bullet/\lVert\mathbf g_\bullet\rVert$, the update direction is:
$\mathbf g_{\mathrm{FPLG}} \;=\; \sqrt{\lVert\mathbf g_{\mathrm{dec}}\rVert\,\lVert\mathbf g_{\mathrm{pred}}\rVert}\;\frac{\hat{\mathbf u}_{\mathrm{dec}}+\gamma_t\,\hat{\mathbf u}_{\mathrm{pred}}}{\lVert\hat{\mathbf u}_{\mathrm{dec}}+\gamma_t\,\hat{\mathbf u}_{\mathrm{pred}}\rVert} \;+\; \lambda\,\mathbf g_{\mathrm{fair}}$,
where the guidance weight $\gamma_t=\kappa_0/(1+\kappa\,t)$ decays over training, so the prediction loss stabilizes the early iterates and the decision objective dominates later ones. 

\paragraph{FDFL-MGDA.} This method applies the multiple gradient descent algorithm (MGDA) from \citet{desideri2012multiple}, where each iteration chooses the minimum-norm combination from the convex hull of individual objective gradients. The update direction is solved from: 
$\mathbf{g}_{MGDA} = \arg\min ~ \left\{\lVert \omega_1 \mathbf{g}_{\mathrm{pred}} + \omega_2 \mathbf{g}_{\mathrm{fair}} + \omega_3\mathbf{g}_{\mathrm{dec}} \rVert^2: {\sum_{i=1}^3 \omega_i = 1, \omega_i \geq 0~\forall i}\right\}$.
The solution $\mathbf{g}$ is either a nonzero common descent direction or equal to $\mathbf{0}$, which occurs exactly at Pareto-stationary points.

\paragraph{WDRO.} The variation-regularized Wasserstein-1 objective is
\begin{equation} \label{eq:wdro}
\min_\theta\ \tfrac{1}{n}\textstyle\sum_i \ell_\theta(\mathbf x_i, r_i) \;+\; \varepsilon\,\tfrac{1}{n}\textstyle\sum_i \big\lVert \nabla_{\mathbf x}\,\ell_\theta(\mathbf x_i, r_i) \big\rVert,
\end{equation}
the MSE loss augmented with an input-gradient-norm penalty. We adopt this regularized form rather than the exact worst-case formulation because the latter is tractable only for restricted model classes: with a nonconvex neural network predictor, the inner supremum over the Wasserstein ball admits no exact finite reformulation, whereas for smooth losses the worst-case objective agrees with the empirical loss plus the gradient-norm variation penalty of \eqref{eq:wdro} to first order in the ball radius $\varepsilon$ \citep{gao2024wasserstein}; \citet{sinha2018certifying} give certified smoothed relaxations of the same form. WDRO thus serves as a robustness-regularized prediction baseline rather than a full distributionally robust program.

\section{Additional Results and Findings} \label{app:extended}
This appendix reports the full experiment results:
all methods (including WDRO, FDFL-FPLG, FDFL-MGDA, described in Appendix~\ref{app:method-details}) and the complete grid over predictor capacity, the decision fairness parameter $\alpha$, the imbalance level $\ell$, the number of training instances $N$, the number of groups $K$, and the prediction-accuracy weight $\mu$ (carried by the Regret-and-MSE and FDFL-Scal rows). The aggregation is identical to the main-text tables: 5 seeds, mean $\pm$ standard deviation, best column mean highlighted in bold.

\input{Tables/tab_app_hc_capacity}
\input{Tables/tab_app_hc_alpha}

\input{Tables/tab_app_md_capacity}
\input{Tables/tab_app_md_alpha}
\input{Tables/tab_app_md_imbalance}

\subsection{Robustness and Additional Analyses} \label{app:robustness}
We report complementary analyses that support and extend the main findings. First, we discuss the effect of the weight $\mu$ on the prediction accuracy objective. In single resource allocation, Table \ref{tab:app-hc-capacity} shows that with MLP-64 predictor, relative to DFL, increasing $\mu$ in Regret-and-MSE steadily lowers prediction MSE while maintaining stable regret at $0.129$, which fits the expected gain from targeting the prediction accuracy metric. Because the prediction disparity metric (MAD) compares prediction errors, improving accuracy is aligned with reducing disparity: over the same range of $\mu$, Regret-and-MSE attains lower MAD with a larger $\mu$ even though the training does not include the prediction fairness objective ($\lambda = 0$). In multiple resource allocation results, Table~\ref{tab:app-md-capacity} shows consistent patterns that support targeting the prediction accuracy objective to enhance prediction quality. 

Next, we compare the gradient-combination rules in training with the full FDFL objective. In both allocation setups, the combination handlers, including scalarization, PCGrad, and NashMTL, attain comparable performances under an MLP-64 predictor (Table \ref{tab:app-hc-capacity} and Table \ref{tab:app-md-capacity}). The handlers separate more at limited predictor capacity. For example, in the multiple resource allocation, Table \ref{tab:app-md-capacity} shows that with the MLP-16 predictor, FDFL-PCGrad attains a superior balance over the other FDFL methods, yielding low regrets, MAD and MSE simultaneously.

Finally, our findings are robust to several design choices. The imbalance-driven prediction fairness trend in Section~\ref{exp-subsec:imbalance} is not specific to $\alpha = 2$: Table \ref{tab:app-md-imb-a05} demonstrates similar trends from varying imbalance levels at $\alpha = 0.5$, with prediction disparity worsening with $\ell$ increase in methods without the prediction fairness objective and effectively reduced from including prediction fairness. In single resource allocation, Table \ref{tab:app-hc-ntrain} validates that all results remain stable as the number of training instances grows from $N = 10$ to $50$. In the multiple resource allocation, Table \ref{tab:app-md-groups} shows that increasing the number of groups from $K=2$ to $K=4$ leaves the regret comparison unchanged and the prediction fairness objective continues to reduce disparity within each setting, though the disparity magnitudes are not directly comparable across a different number of groups.

\input{Tables/tab_app_hc_ntrain}
\input{Tables/tab_app_md_groups}
\input{Tables/tab_app_md_imb_a05}

\section{Proof of Generalization Bound} \label{app:proofs-gen}
This appendix contains the proof of Theorem \ref{thm:gen-main} stated in Section~\ref{sec:generalization}, together with verification remarks for the Lipschitzness assumptions in our experiment settings. 
Recall the notation of Section~\ref{sec:generalization}: $z=(X,\mathbf{a},\mathbf{r},D_\mathcal{S})$ collects the instance features, group memberships, true impacts, and feasible-set data (we suppress $\mathbf{a}$ and $D_\mathcal{S}$ where they play no active role); 
$\mathcal{D}=\{z_s\}_{s=1}^N$ is the set of i.i.d.\ training samples, $\Theta\subset\mathbb{R}^q$ is the parameter space, the per-instance
composite loss $\ell^{\mathrm{full}}(\theta; z)$ from Section~\ref{sec:generalization}, and we abbreviate the instance-dependent 
oracle $\mathbf{d}^*_z$ as $\mathbf{d}^*$ when the instance is clear from context. Because our predictor feeds \emph{point predictions} to the decision oracle, instead of 
a parametric distribution family as in 
\citet{elmachtoub2025dissecting}, 
we control the loss classes by a different set of standard tools, including Lipschitz transfer from the parameter space, metric-entropy bounds, and Dudley chaining. 
\subsection{Supporting Lemmas}

\begin{lemma} [The regret class is Lipschitz in $\theta$] 
\label{lem:gen-gtheta-lip}
Under Assumption~\ref{assp:gen-oracle-lip}, the map $\theta\mapsto \ell^{\mathrm{reg}}(\theta;z)$ is $L_W L_{\mathrm{or}}$-Lipschitz uniformly in $z$ with $L_W$ and $L_{\mathrm{or}}$ the Lipschitz constants of Assumption~\ref{assp:gen-oracle-lip}.:
$$|\ell^{\mathrm{reg}}(\theta_1;z)-\ell^{\mathrm{reg}}(\theta_2;z)| \leq L_W\,L_{\mathrm{or}}\,\|\theta_1-\theta_2\|,$$
\end{lemma}

\emph{Proof.}
By definition $\ell^{\mathrm{reg}}(\theta;z)=W(\mathbf{d}^*(\mathbf{r});\mathbf{r})-W(\mathbf{d}^*(f_\theta(X));\mathbf{r})$. The first term is independent of $\theta$, so the $\theta$-dependence enters only through $W(\mathbf{d}^*(f_\theta(X));\mathbf{r})$. Applying the 
two Lipschitz properties of Assumption~\ref{assp:gen-oracle-lip} in turn,
$|\ell^{\mathrm{reg}}(\theta_1;z)-\ell^{\mathrm{reg}}(\theta_2;z)| \leq L_W\,\|\mathbf{d}^*(f_{\theta_1}(X))-\mathbf{d}^*(f_{\theta_2}(X))\| \leq L_W\,L_{\mathrm{or}}\,\|\theta_1-\theta_2\|$ follows.
$\square$

\begin{lemma}[Parametric Rademacher bound for uniformly Lipschitz classes] \label{lem:gen-param-rad}
Let $\mathcal{F}:=\{\varphi_\theta:\theta\in\Theta\}$ be a class of real-valued functions on $\mathcal{Z}$ indexed by $\Theta$ such that, for some $L>0$,
$\sup_{z\in\mathcal{Z}}|\varphi_{\theta_1}(z)-\varphi_{\theta_2}(z)|\leq L\,\|\theta_1-\theta_2\|$ for all $\theta_1,\theta_2\in\Theta$. Under Assumption~\ref{assp:gen-geometry}, there exists an absolute constant $C_{\mathrm{abs}}>0$, independent of $L$, $\Theta$, and $N$, such that
$$\mathfrak{R}_N(\mathcal{F}) \leq C_{\mathrm{abs}}\,L\,E_\Theta\,\sqrt{q/N}.$$
\end{lemma}

\emph{Proof.}
Fix a sample $\mathcal{D}=\{z_s\}_{s=1}^N$, let $\sigma_1,\ldots,\sigma_N$ be i.i.d.\ Rademacher variables independent of $\mathcal{D}$, and define on $\Theta$ the scaled Euclidean metric $d_N(\theta_1,\theta_2) := \frac{L}{\sqrt{N}}\,\|\theta_1-\theta_2\|$. For the centered process $X_\theta:=\frac{1}{N}\sum_{s=1}^N\sigma_s\, \varphi_\theta(z_s)$, each increment $X_{\theta_1}-X_{\theta_2}=\sum_{s=1}^N \sigma_s a_s$ is a weighted Rademacher sum with weights $a_s:=(\varphi_{\theta_1}(z_s)-\varphi_{\theta_2}(z_s))/N$, whose norm the assumed Lipschitzness controls by
$\|a\|^2 = \frac{1}{N^2}\sum_{s=1}^N \bigl(\varphi_{\theta_1}(z_s)-\varphi_{\theta_2}(z_s)\bigr)^2 \leq \frac{L^2\|\theta_1-\theta_2\|^2}{N} = d_N(\theta_1,\theta_2)^2$,
so Hoeffding's lemma \citep[Lemma~D.1]{mohri2018foundations} gives, for all $s\in\mathbb{R}$,
\[
\mathbb{E}_\sigma\bigl[e^{s(X_{\theta_1}-X_{\theta_2})}\bigr]\leq e^{s^2\|a\|^2/2}\leq e^{s^2 d_N(\theta_1,\theta_2)^2/2},
\]
and hence, by a Chernoff bound, $\Pr\bigl[|X_{\theta_1}-X_{\theta_2}|\geq u\bigr]\leq 2\exp\bigl(-u^2/(2\,d_N(\theta_1,\theta_2)^2)\bigr)$ for all $u\geq 0$: the increments of $(X_\theta)_{\theta\in\Theta}$ are sub-Gaussian with respect to $d_N$. The process is also separable in the sense of \citet[Definition~4.4]{sen2018gentle}, since its sample paths are Lipschitz in $\theta$ and $\Theta\subseteq\mathbb{R}^q$ is separable.

Let $D(\epsilon,T,d)$ denote the $\epsilon$-packing number of a metric space $(T,d)$, the maximum number of points of $T$ that are pairwise more than $\epsilon$ apart \citep[Definitions~2.4--2.5]{sen2018gentle}. Covering the ball of radius $E_\Theta$ that encloses $\Theta$ (Assumption~\ref{assp:gen-geometry}) by at most $(6E_\Theta/\eta)^q$ balls of radius $\eta/2$ \citep[Lemma~6.27]{mohri2018foundations}, each of which contains at most one point of any $\eta$-packing, yields $\log D(\eta,\Theta,\|\cdot\|) \leq q\log(6 E_\Theta/\eta)$ for $\eta\in(0,2E_\Theta]$; since $d_N$ rescales the Euclidean metric by $L/\sqrt{N}$, for $\epsilon\in(0,2L E_\Theta/\sqrt{N}]$,
\[
\log D(\epsilon,\Theta,d_N) = \log D\Bigl(\frac{\epsilon\sqrt{N}}{L},\Theta,\|\cdot\|\Bigr) \leq q\log\frac{6 L E_\Theta}{\epsilon\sqrt{N}},
\]
and the $d_N$-diameter of $\Theta$ is at most $\bar{D}:=2LE_\Theta/\sqrt{N}$. Dudley's entropy-integral bound for separable sub-Gaussian processes \citep[Theorem~4.5]{sen2018gentle} then provides an absolute constant $c_0>0$ such that, for any fixed $\theta_0\in\Theta$ (using $\mathbb{E}_\sigma[X_{\theta_0}]=0$),
\begin{align*}
\widehat{\mathfrak{R}}_{\mathcal{D}}(\mathcal{F}) =  \mathbb{E}_\sigma\Bigl[\sup_{\theta\in\Theta}X_\theta\Bigr]
&\leq \mathbb{E}_\sigma\Bigl[\sup_{\theta\in\Theta}\bigl|X_\theta - X_{\theta_0}\bigr|\Bigr] 
\leq c_0\int_0^{\bar{D}/4}\sqrt{\log D(\epsilon,\Theta,d_N)}\,d\epsilon \\
&\leq c_0\sqrt{q}\int_0^{2L E_\Theta/\sqrt{N}}\!\sqrt{\log\frac{6 L E_\Theta}{\epsilon\sqrt{N}}}\,d\epsilon
= c_0\sqrt{q}\,\frac{L E_\Theta}{\sqrt{N}}\int_0^2\sqrt{\log\frac{6}{u}} \\
& = c_0\sqrt{q}\,\frac{L E_\Theta}{\sqrt{N}}\cdot 6\int_{\log 3}^{\infty}\sqrt{v}\,e^{-v}\,dv 
\leq 6\,\Gamma(3/2)\,c_0\, L\, E_\Theta\sqrt{\frac{q}{N}}
= C_{\mathrm{abs}}\,L\,E_\Theta\sqrt{\frac{q}{N}},
\end{align*}
where the two equalities use the changes of variables $u=\epsilon\sqrt{N}/(L E_\Theta)$ and $v=\log(6/u)$, and $C_{\mathrm{abs}}:=3\sqrt{\pi}\,c_0$. The right-hand side is deterministic in $\mathcal{D}$, so taking $\mathbb{E}_{\mathcal{D}}$ yields the claimed bound.
$\square$

\begin{lemma}[Rademacher decomposition] \label{lem:gen-decomp}
For any $\lambda,\mu\geq 0$,
$\mathfrak{R}_N(\mathcal{F}_{\mathrm{full}}) \leq \mathfrak{R}_N(\mathcal{F}_{\mathrm{reg}}) + \lambda\,\mathfrak{R}_N(\mathcal{F}_{\mathrm{fair}}) + \mu\,\mathfrak{R}_N(\mathcal{F}_{\mathrm{pred}}).$
\end{lemma}

\emph{Proof.}
For any fixed sample $\mathcal D$,
$\widehat{\mathfrak R}_{\mathcal D}(\mathcal F_{\mathrm{full}})
= \mathbb E_\sigma\!\left[
\sup_{\theta\in\Theta}
\frac{1}{N}\sum_{s=1}^N \sigma_s
\bigl(
\ell^{\mathrm{reg}}(\theta;z_s)
+\lambda\ell^{\mathrm{fair}}(\theta;z_s)
+\mu\ell^{\mathrm{pred}}(\theta;z_s)
\bigr)
\right] 
\leq
\widehat{\mathfrak R}_{\mathcal D}(\mathcal F_{\mathrm{reg}})
+\lambda\widehat{\mathfrak R}_{\mathcal D}(\mathcal F_{\mathrm{fair}})
+\mu\widehat{\mathfrak R}_{\mathcal D}(\mathcal F_{\mathrm{pred}})$,
where the inequality uses subadditivity of the supremum and
$\lambda,\mu\geq0$. Taking expectation over $\mathcal D$ proves the claim.
\hfill$\square$
\begin{lemma}[Rademacher bound for $\mathcal{F}_{\mathrm{fair}}$ and $\mathcal{F}_{\mathrm{pred}}$] \label{lem:gen-h-rad}
Under Assumptions~\ref{assp:gen-geometry} and~\ref{assp:gen-predictor-lip},
$\mathfrak{R}_N(\mathcal{F}_{\mathrm{fair}})\leq C_{\mathrm{abs}}\,L_F\,E_\Theta\sqrt{q/N}$ and 
$\mathfrak{R}_N(\mathcal{F}_{\mathrm{pred}})\leq C_{\mathrm{abs}}\,L_L\,E_\Theta\sqrt{q/N}$,
with the same absolute constant $C_{\mathrm{abs}}$ as in Lemma~\ref{lem:gen-param-rad}.
\end{lemma}

\emph{Proof.}
By Assumption~\ref{assp:gen-predictor-lip}, $|\ell^{\mathrm{fair}}(\theta_1;z)-\ell^{\mathrm{fair}}(\theta_2;z)|\leq L_F\|\theta_1-\theta_2\|$ and $|\ell^{\mathrm{pred}}(\theta_1;z)-\ell^{\mathrm{pred}}(\theta_2;z)|\leq L_L\|\theta_1-\theta_2\|$ for every $z\in\mathcal{Z}$ and $\theta_1,\theta_2\in\Theta$. The classes $\mathcal{F}_{\mathrm{fair}}$ and $\mathcal{F}_{\mathrm{pred}}$ therefore satisfy the required conditions of Lemma~\ref{lem:gen-param-rad} with $L=L_F$ and $L=L_L$, respectively, and the claimed bounds follow.
$\square$

\begin{lemma}[Uniform deviation for the composite training objective] \label{lem:gen-uniform}
Under Assumptions~\ref{assp:gen-geometry}--\ref{assp:gen-predictor-lip}, for any $\delta\in(0,1)$, with probability at least $1-\delta$,
\[
\sup_{\theta\in\Theta}\bigl(v_0^{\mathrm{full}}(\theta)-\hat v_0^{\mathrm{full}}(\theta)\bigr)
\leq 2\,\mathfrak{R}_N(\mathcal{F}_{\mathrm{reg}}) + 2\lambda\,\mathfrak{R}_N(\mathcal{F}_{\mathrm{fair}}) + 2\mu\,\mathfrak{R}_N(\mathcal{F}_{\mathrm{pred}}) + (B_{\mathrm{reg}}+\lambda B_F+\mu B_L)\sqrt{\tfrac{2\log(1/\delta)}{N}}.
\]
\end{lemma}
\emph{Proof.}
Write $B:=B_{\mathrm{reg}}+\lambda B_F+\mu B_L$, so $|\ell^{\mathrm{full}}(\theta;z)| \leq |\ell^{\mathrm{reg}}(\theta;z)| + \lambda|\ell^{\mathrm{fair}}(\theta;z)| + \mu|\ell^{\mathrm{pred}}(\theta;z)| \leq B$ by Assumption~\ref{assp:gen-bounded}, and define $\Phi(\mathcal{D}) := \sup_{\theta\in\Theta}\bigl(v_0^{\mathrm{full}}(\theta)-\hat v_0^{\mathrm{full}}(\theta)\bigr)$, measurable by the same separability argument as in the proof of Lemma~\ref{lem:gen-param-rad}. Let $\mathcal{D}'=\{z_s'\}_{s=1}^N\sim\mathcal{P}^N$ be a ghost sample independent of $\mathcal{D}$, with empirical objective $\hat v_0^{\mathrm{full},\mathcal{D}'}(\theta):=\frac{1}{N}\sum_{s=1}^N \ell^{\mathrm{full}}(\theta;z_s')$ satisfying $v_0^{\mathrm{full}}(\theta) = \mathbb{E}_{\mathcal{D}'}[\hat v_0^{\mathrm{full},\mathcal{D}'}(\theta)]$, and let $\sigma_1,\ldots,\sigma_N$ be i.i.d.\ Rademacher variables independent of $(\mathcal{D},\mathcal{D}')$. Then
\begin{align*}
\mathbb{E}_{\mathcal{D}}[\Phi(\mathcal{D})]
&= \mathbb{E}_{\mathcal{D}}\Bigl[\sup_{\theta\in\Theta}\,\mathbb{E}_{\mathcal{D}'}\bigl[\hat v_0^{\mathrm{full},\mathcal{D}'}(\theta) - \hat v_0^{\mathrm{full}}(\theta)\bigr]\Bigr] \\
&\leq \mathbb{E}_{\mathcal{D},\mathcal{D}'}\Bigl[\sup_{\theta\in\Theta}\frac{1}{N}\sum_{s=1}^N \bigl(\ell^{\mathrm{full}}(\theta;z_s')-\ell^{\mathrm{full}}(\theta;z_s)\bigr)\Bigr]
&& \text{(Jensen; Fubini)} \\
&= \mathbb{E}_{\mathcal{D},\mathcal{D}',\sigma}\Bigl[\sup_{\theta\in\Theta}\frac{1}{N}\sum_{s=1}^N \sigma_s\bigl(\ell^{\mathrm{full}}(\theta;z_s')-\ell^{\mathrm{full}}(\theta;z_s)\bigr)\Bigr]
&& \text{(swap $z_s\leftrightarrow z_s'$)} \\
&\leq \mathbb{E}_{\mathcal{D}',\sigma}\Bigl[\sup_{\theta\in\Theta}\frac{1}{N}\sum_{s=1}^N \sigma_s\,\ell^{\mathrm{full}}(\theta;z_s')\Bigr]
+ \mathbb{E}_{\mathcal{D},\sigma}\Bigl[\sup_{\theta\in\Theta}\frac{1}{N}\sum_{s=1}^N (-\sigma_s)\,\ell^{\mathrm{full}}(\theta;z_s)\Bigr] \\
&= 2\,\mathfrak{R}_N(\mathcal{F}_{\mathrm{full}})
&& (-\boldsymbol{\sigma}\overset{d}{=}\boldsymbol{\sigma},\ \mathcal{D}'\overset{d}{=}\mathcal{D}) \\
&\leq 2\,\mathfrak{R}_N(\mathcal{F}_{\mathrm{reg}}) + 2\lambda\,\mathfrak{R}_N(\mathcal{F}_{\mathrm{fair}}) + 2\mu\,\mathfrak{R}_N(\mathcal{F}_{\mathrm{pred}}).
&& \text{(Lemma~\ref{lem:gen-decomp})}
\end{align*}
Moreover, replacing the $s$-th coordinate $z_s$ of $\mathcal{D}$ with an arbitrary $\tilde z\in\mathcal{Z}$ yields a sample $\mathcal{D}^{(s)}$ with
\[
|\Phi(\mathcal{D}) - \Phi(\mathcal{D}^{(s)})| \leq \sup_{\theta\in\Theta}\bigl|\hat v_0^{\mathrm{full},\mathcal{D}}(\theta) - \hat v_0^{\mathrm{full},\mathcal{D}^{(s)}}(\theta)\bigr| = \sup_{\theta\in\Theta}\frac{1}{N}\bigl|\ell^{\mathrm{full}}(\theta;z_s)-\ell^{\mathrm{full}}(\theta;\tilde z)\bigr| \leq \frac{2B}{N},
\]
so McDiarmid's bounded-difference inequality \citep[Theorem~D.8]{mohri2018foundations} gives $\Pr\bigl[\Phi(\mathcal{D})\geq\mathbb{E}_{\mathcal{D}}[\Phi(\mathcal{D})] + t\bigr] \leq \exp(-Nt^2/(2B^2))$ for all $t>0$; setting the right-hand side to $\delta$, i.e., $t = B\sqrt{2\log(1/\delta)/N}$, and combining with the display above proves the claim.
$\square$

\subsection{Proof of Theorem~\ref{thm:gen-main}}
Write $B:=B_{\mathrm{reg}}+\lambda B_F+\mu B_L$. Adding and subtracting $\hat v_0^{\mathrm{full}}(\hat\theta)$ and $\hat v_0^{\mathrm{full}}(\theta^\ast)$,
\[
R^{\mathrm{full}}(\hat\theta) := v_0^{\mathrm{full}}(\hat\theta) - v_0^{\mathrm{full}}(\theta^\ast)
= \underbrace{\bigl(v_0^{\mathrm{full}}(\hat\theta) - \hat v_0^{\mathrm{full}}(\hat\theta)\bigr)}_{(A)}
+ \underbrace{\bigl(\hat v_0^{\mathrm{full}}(\hat\theta) - \hat v_0^{\mathrm{full}}(\theta^\ast)\bigr)}_{(M)}
+ \underbrace{\bigl(\hat v_0^{\mathrm{full}}(\theta^\ast) - v_0^{\mathrm{full}}(\theta^\ast)\bigr)}_{(B)},
\]
where $(M)\leq 0$ deterministically since $\hat\theta\in\arg\min_{\theta\in\Theta}\hat v_0^{\mathrm{full}}(\theta)$, and $(A) \leq \sup_{\theta\in\Theta}\bigl(v_0^{\mathrm{full}}(\theta) - \hat v_0^{\mathrm{full}}(\theta)\bigr)$. Each of the following two bounds holds with probability at least $1-\delta/2$ over $\mathcal{D}\sim\mathcal{P}^N$:
\begin{align}
(A) &\leq 2\,\mathfrak{R}_N(\mathcal{F}_{\mathrm{reg}}) + 2\lambda\,\mathfrak{R}_N(\mathcal{F}_{\mathrm{fair}}) + 2\mu\,\mathfrak{R}_N(\mathcal{F}_{\mathrm{pred}}) + B\sqrt{\tfrac{2\log(2/\delta)}{N}}, \label{eq:proof-thm-A} \\
(B) &= \frac{1}{N}\sum_{s=1}^N \ell^{\mathrm{full}}(\theta^\ast;z_s) - \mathbb{E}\bigl[\ell^{\mathrm{full}}(\theta^\ast;z)\bigr]
\;\leq\; B\sqrt{\tfrac{2\log(2/\delta)}{N}}, \label{eq:proof-thm-B}
\end{align}
where \eqref{eq:proof-thm-A} applies Lemma~\ref{lem:gen-uniform} at confidence level $\delta/2$, and \eqref{eq:proof-thm-B} applies Hoeffding's inequality \citep[Theorem~D.2]{mohri2018foundations} to the i.i.d.\ variables $\ell^{\mathrm{full}}(\theta^\ast;z_s)\in[-B,B]$ (Assumption~\ref{assp:gen-bounded}; $\theta^\ast$ is non-random). On the intersection of the two events, whose probability is at least $1-\delta$ by the union bound,
\begin{align*}
R^{\mathrm{full}}(\hat\theta)
&\leq 2\,\mathfrak{R}_N(\mathcal{F}_{\mathrm{reg}}) + 2\lambda\,\mathfrak{R}_N(\mathcal{F}_{\mathrm{fair}}) + 2\mu\,\mathfrak{R}_N(\mathcal{F}_{\mathrm{pred}}) + 2B\sqrt{\tfrac{2\log(2/\delta)}{N}} \\
&\leq 2 C_{\mathrm{abs}}\bigl(L_W L_{\mathrm{or}} + \lambda\, L_F + \mu\, L_L\bigr) E_\Theta\sqrt{\tfrac{q}{N}} + 2B\sqrt{\tfrac{2\log(2/\delta)}{N}},
\end{align*}
where the last inequality substitutes the parametric Rademacher bounds of Lemma~\ref{lem:gen-param-rad}, applied to $\mathcal{F}_{\mathrm{reg}}$ with $L=L_WL_{\mathrm{or}}$ (Lemma~\ref{lem:gen-gtheta-lip}) and to $\mathcal{F}_{\mathrm{fair}}$ and $\mathcal{F}_{\mathrm{pred}}$ with $L=L_F$ and $L=L_L$ (Lemma~\ref{lem:gen-h-rad}). This is exactly the bound stated in~\eqref{eq:gen-main-bound}.
$\square$

\subsection{Verification remarks} \label{app:verify}
This subsection verifies Assumptions~\ref{assp:gen-oracle-lip} and~\ref{assp:gen-predictor-lip} for the settings of Section~\ref{sec:exp-setup}. 

\begin{remark}[Verification of Assumption~\ref{assp:gen-oracle-lip}] \label{rem:gen-oracle-verification}

We verify Assumption~\ref{assp:gen-oracle-lip} for the two allocation tasks via the following two lemmas.

\begin{lemma}[Concavity of $\alpha$-fairness]\label{lem:alpha-fair-concavity}
Suppose utilities are all positive. (i) For every $\alpha>0$, the group-based measure $W_\alpha^g$ in \eqref{eq:alpha-group-W} is strictly concave in $\mathbf{u}$.
(ii) If the utility map $\mathbf{d}\mapsto\mathbf{u}(\mathbf{d})$ is affine, $\mathbf{d}\mapsto W_\alpha^g(\mathbf{u}(\mathbf{d}))$ is concave in $\mathbf{d}$; if the map is injective on the feasible set, $W_\alpha^g(\mathbf{u}(\mathbf{d}))$ is strictly concave in $\mathbf{d}$ with a negative-definite Hessian.
\end{lemma}
\emph{Proof.}
\emph{Part (i), case $\alpha>1$.} Write $\beta:=\alpha-1>0$ and $S_k:=\sum_{i\in G_k}u_i^{-\beta}$. The intra-group score is $h_k=\beta/S_k$, so the $k$th group term of \eqref{eq:alpha-group-W} equals $h_k^{1-\alpha}/(1-\alpha)=-\beta^{-\beta-1}S_k^{\beta}$; therefore, it suffices to show that $S_k^{\beta}$ is strictly convex in the group's coordinates. With $w_i:=u_i^{-\beta-1}$, the Hessian of $S_k^{\beta}$ is:
\[
\nabla^2 S_k^{\beta}=\beta^3(\beta-1)S_k^{\beta-2}\,\mathbf{w}\mathbf{w}^\top+\beta^2(\beta+1)S_k^{\beta-1}\operatorname{diag}\!\bigl(u_i^{-\beta-2}\bigr).
\]
Fix a direction $v\neq 0$ and set $T:=\sum_{i\in G_k}u_i^{-\beta-2}v_i^2>0$. If $\beta\geq1$, the rank-one term $\beta^3(\beta-1)S_k^{\beta-2}\,\mathbf{w}\mathbf{w}^\top$ is positive semidefinite, so $v^\top\nabla^2 S_k^{\beta}\,v\geq\beta^2(\beta+1)S_k^{\beta-1}T>0$. If $0<\beta<1$, the rank-one term is negative; the Cauchy--Schwarz bound $\bigl(\sum_{i\in G_k}u_i^{-\beta-1}v_i\bigr)^2\leq S_k\,T$ then gives $v^\top\nabla^2 S_k^{\beta}\,v \geq \beta^2(\beta^2+1)S_k^{\beta-1}T>0$. In both cases $S_k^{\beta}$ is strictly convex, so each group term is strictly concave in its group's coordinates.

\emph{Part (i), case $0<\alpha<1$.} Write $\gamma:=1-\alpha\in(0,1)$ and $T_k:=\sum_{i\in G_k}u_i^{\gamma}$. The intra-group score is $h_k=T_k/\gamma$, so the $k$th group term of \eqref{eq:alpha-group-W} equals $h_k^{\gamma}/\gamma=\gamma^{-\gamma-1}T_k^{\gamma}$, and it suffices to show that $T_k^{\gamma}$ is strictly concave in the group's coordinates. Its Hessian is
\[
\nabla^2 T_k^{\gamma}=\gamma^3(\gamma-1)T_k^{\gamma-2}\,\mathbf{w}\mathbf{w}^\top+\gamma^2(\gamma-1)T_k^{\gamma-1}\operatorname{diag}\!\bigl(u_i^{\gamma-2}\bigr),
\qquad w_i:=u_i^{\gamma-1}.
\]
Since $\gamma\in(0,1)$, both coefficients are negative, so the Hessian is negative definite and the group term is strictly concave in the group's coordinates.

In either case, the groups partition $[m]$, so the Hessian of $W_\alpha^g$ is block-diagonal with negative-definite blocks and $W_\alpha^g$ is strictly concave in $\mathbf{u}$. (At $\alpha=1$, $h_k=\prod_{i\in G_k}u_i$ gives $W_1^g(\mathbf{u})=\sum_{i=1}^m\log u_i$ by Remark~\ref{rem:alpha-coincide}, whose Hessian $-\operatorname{diag}(u_i^{-2})$ is negative definite, so the claim holds there as well.)

\emph{Part (ii).} Concavity is preserved by composition with an affine map. If the map is injective on the feasible set, its linear part $A$ has trivial kernel, so the composition's Hessian $A^\top(\nabla^2_{\mathbf{u}}W_\alpha^g)A$ is negative definite and the composition is strictly concave; on a compact set of decisions with strictly positive utilities, the Hessian's eigenvalues are bounded away from zero by continuity, so $-W_\alpha^g(\mathbf{u}(\cdot))$ is strongly convex there.
$\square$

\begin{lemma}[Decision oracle Lipschitzness from strong convexity]\label{lem:gen-oracle-from-sc}
Suppose, in all instances (i)~$-W(\cdot;\hat{\mathbf{r}})$ is $\rho_c$-strongly convex in $\mathbf{d}$ on a convex set containing the oracle-reachable allocations;
(ii)~$\nabla_{\mathbf{d}}W(\mathbf{d};\hat{\mathbf{r}})$ is $L_{W,r}$-Lipschitz in $\hat{\mathbf{r}}$; and
(iii)~the prediction map $\theta\mapsto f_\theta(X)\in\mathbb{R}^m$ is $L_f$-Lipschitz in $\theta$.
Then the decision map $\theta\mapsto\mathbf{d}^*(f_\theta(X))$ satisfies Assumption~\ref{assp:gen-oracle-lip} with $L_{\mathrm{or}}\leq (L_{W,r}/\rho_c)\,L_f$.
\end{lemma}
\emph{Proof.}
Denote $\psi(\cdot;\hat{\mathbf{r}}):=-W(\cdot;\hat{\mathbf{r}})$ and $\mathbf{d}_j:=\mathbf{d}^*(\hat{\mathbf{r}}_j)$ for $j=1,2$.
The first-order optimality conditions on the convex feasible set give $\langle\nabla\psi(\mathbf{d}_1;\hat{\mathbf{r}}_1),\mathbf{d}_2-\mathbf{d}_1\rangle\geq 0$ and $\langle\nabla\psi(\mathbf{d}_2;\hat{\mathbf{r}}_2),\mathbf{d}_2-\mathbf{d}_1\rangle\leq 0$. Combining them with the strong monotonicity of $\nabla\psi(\cdot;\hat{\mathbf{r}}_1)$ and the Cauchy--Schwarz inequality,
$\rho_c\|\mathbf{d}_2-\mathbf{d}_1\|^2
\leq \langle\nabla\psi(\mathbf{d}_2;\hat{\mathbf{r}}_1)-\nabla\psi(\mathbf{d}_1;\hat{\mathbf{r}}_1),\,\mathbf{d}_2-\mathbf{d}_1\rangle 
\leq \langle\nabla\psi(\mathbf{d}_2;\hat{\mathbf{r}}_1)-\nabla\psi(\mathbf{d}_2;\hat{\mathbf{r}}_2),\,\mathbf{d}_2-\mathbf{d}_1\rangle
\leq L_{W,r}\,\|\hat{\mathbf{r}}_1-\hat{\mathbf{r}}_2\|\,\|\mathbf{d}_2-\mathbf{d}_1\|$,
so $\|\mathbf{d}^*(\hat{\mathbf{r}}_1)-\mathbf{d}^*(\hat{\mathbf{r}}_2)\|\leq(L_{W,r}/\rho_c)\,\|\hat{\mathbf{r}}_1-\hat{\mathbf{r}}_2\|$; composing with the $L_f$-Lipschitz map $\theta\mapsto f_\theta(X)$ gives the bound.
$\square$

In the single resource allocation task, the predicted benefits are positive and bounded, with $\hat{r}_i \in [2,101]$ for all $i$. By Proposition~\ref{prop:closed_form_solution_group-1}, $\mathbf{d}^*(\hat{\mathbf{r}})$ is a continuous and strictly positive function of $\hat{\mathbf{r}}$, so the oracle-reachable allocations admit uniform bounds $0<d_{\min}\le d_i^*\le d_{\max}=Q/c_{\min}$. In addition, $W$, $\nabla_{\mathbf{d}}W$, and $\nabla^2_{\mathbf{d}}W$ are continuous with negative-definite $\nabla^2_{\mathbf{d}}W$ (Lemma~\ref{lem:alpha-fair-concavity}); hence finite $B_{\mathrm{reg}}$, $L_W$, $L_{W,r}$ and a strong-concavity modulus $\rho_c>0$ exist, and Lemma~\ref{lem:gen-oracle-from-sc} gives a uniform $L_{\mathrm{or}}\le (L_{W,r}/\rho_c)L_f$ on $\Theta$.

Next, we examine the multiple resource allocation setting. The utility function $u_i=\sum_{j=1}^D\hat r_{ij}d_{ij}$ has Jacobian $A=\partial\mathbf{u}/\partial\mathbf{d}\in\mathbb{R}^{m\times mD}$ of rank $m$, so $\nabla_{\mathbf{d}}^2 W=A^\top(\nabla_{\mathbf{u}}^2 W)A$ is rank-deficient by $m(D-1)$: the $\alpha$-fairness objective function is strictly concave in $\mathbf{u}$ but only concave in $\mathbf{d}$. This has two consequences. First, Lemma~\ref{lem:gen-oracle-from-sc} does not apply, since the required strong-convexity does not hold. Second, because $W$ is constant along $\ker A$ within the feasible set, the optimal utility $\mathbf{u}^*$ is unique but the optimal allocation $\mathbf{d}^*$ is not guaranteed to be unique. To obtain decision-oracle Lipschitzness, we require an additional nondegeneracy assumption as follows.

\begin{assumption}[Uniform conic nondegeneracy for non-injective utility maps]\label{assp:gen-conic}
At every $\hat{\mathbf{r}}$, the conic reformulation of \eqref{eq:opt_prob} admits a unique primal--dual optimal solution satisfying strict complementarity; by \citet{agrawal2019cone,busseti2019solution}, this solution is a regular point at which the normalized residual map $\mathcal{N}$ is differentiable with invertible derivative $D\mathcal{N}$. We assume in addition that the smallest singular value of $D\mathcal{N}$ at these solutions is bounded below by a uniform constant $\sigma_{\min}>0$.
\end{assumption}

Decision-oracle Lipschitzness then follows from the differentiable cone-program framework of \citet{agrawal2019differentiable} underlying our \texttt{cvxpylayers} backend: the $\alpha$-fair allocation admits a power-cone representation,
and under Assumption~\ref{assp:gen-conic} the implicit function theorem applied to the residual map of the homogeneous self-dual embedding \citep[\S 2]{agrawal2019cone,busseti2019solution} gives $\|D_{\hat{\mathbf{r}}}\mathbf{d}^*\|\le C_{\mathrm{data}}/\sigma_{\min}$ throughout the prediction range, where $C_{\mathrm{data}}$ bounds the canonicalization and extraction maps under the bounded problem data. The mean-value inequality over the convex, compact prediction range lifts this pointwise bound to a Lipschitz constant for $\hat{\mathbf{r}}\mapsto\mathbf{d}^*$, which composes with the $L_f$-Lipschitz predictor to yield $L_{\mathrm{or}}$ on $\Theta$.
\end{remark}

\begin{remark} [Verification of $L_F$ for the MAD-based prediction disparity] \label{rem:gen-mad-cohort}
Let $\mathcal{I}$ index the predicted coordinates of an instance ($\mathcal{I}=[m]$ in the
single-resource task; $\mathcal{I}=[m]\times[R]$ in the multiple-resource task), and let
$\mathcal{I}=\bigcup_{k=1}^{K}\mathcal{I}_k$ be the partition induced by the group memberships $\mathbf{a}$. The MAD-based prediction disparity measure is a function of the squared prediction errors
$\mathbf{e}(\theta;z)\in\mathbb{R}^{\mathcal{I}}$ with $e_i(\theta;z):=\bigl(r_i-f_\theta(\mathbf{x}_i)\bigr)^2$.

Since true impacts $\mathbf{r}$ and predictions lie in a common bounded interval, the prediction
errors obey $|r_i-f_\theta(\mathbf{x}_i)|\le B_e$ for all $i\in\mathcal{I}$, $\theta\in\Theta$, and $z\in\mathcal{Z}$, so $\mathbf{e}(\theta;z)\in[0,B_e^2]^{\mathcal{I}}$.
The prediction map $\theta \mapsto f_\theta(X)$ is $L_f$-Lipschitz with respect to $\|\cdot\|_\infty$ (coordinate-wise projection onto a fixed interval is $1$-Lipschitz, so the projected predictor retains $L_f$). Let $\hat y_i^{\,j}:=f_{\theta_j}(\mathbf{x}_i)$, $\bigl|e_i(\theta_1;z)-e_i(\theta_2;z)\bigr|
=\bigl|\hat y_i^{\,2}-\hat y_i^{\,1}\bigr|\cdot
 \bigl|(r_i-\hat y_i^{\,1})+(r_i-\hat y_i^{\,2})\bigr|
\le 2B_e\bigl|\hat y_i^{\,1}-\hat y_i^{\,2}\bigr|$.
Taking the maximum over $i\in\mathcal{I}$ gives
\begin{equation}\label{eq:gen-error-stability}
\bigl\|\mathbf{e}(\theta_1;z)-\mathbf{e}(\theta_2;z)\bigr\|_\infty
\;\le\; 2B_e L_f\,\|\theta_1-\theta_2\|
\qquad\text{for all }\theta_1,\theta_2\in\Theta.
\end{equation}

The group-based disparity passes this vector through a group average and then a
MAD function,
$\ell^{\mathrm{fair}}(\theta;z)=p\bigl(\mathbf{m}(\mathbf{e}(\theta;z))\bigr)$, where
$m_k(\mathbf{e}):=\frac{1}{|\mathcal{I}_k|}\sum_{i\in\mathcal{I}_k}e_i$, $p(\mathbf{m}):=\frac{1}{K}\sum_{k=1}^{K}\bigl|m_k-\bar m\bigr|$, 
and $\bar m$ is the mean of $\mathbf{m}$; the loss is thus a single scalar per instance and is
per-instance decomposable. Each of the two maps is $1$-Lipschitz with respect to
$\|\cdot\|_\infty$. For the group average,
$|m_k(\mathbf{e})-m_k(\mathbf{e}')|\le|\mathcal{I}_k|^{-1}\sum_{i\in\mathcal{I}_k}|e_i-e_i'|
\le\|\mathbf{e}-\mathbf{e}'\|_\infty$ for every $k$. For the MAD map, linearity of
$\mathbf{m}\mapsto\bar m$ makes $p$ a seminorm, so the reverse triangle inequality gives
$|p(\mathbf{m})-p(\mathbf{m}')|\le p(\mathbf{m}-\mathbf{m}')
\le\bigl[\sup_{\|\mathbf{x}\|_\infty\le1}p(\mathbf{x})\bigr]\|\mathbf{m}-\mathbf{m}'\|_\infty$;
by convexity the supremum is attained at a vertex of the cube, which yields
$\max_{K_1+K_2=K}4K_1K_2/K^2\le1$, so $p$ is $1$-Lipschitz in $\|\cdot\|_\infty$. The same
vertex argument over $[0,B_e^2]^{K}$ bounds the range of $p$ by $B_e^2/2$. Combining with
\eqref{eq:gen-error-stability} gives $B_F\le B_e^2/2$ and $L_F\le 2B_eL_f$.
\end{remark}

\begin{remark}[Verification of $L_L$ for the MSE prediction loss]\label{rem:gen-mse-cohort}
The prediction loss shares the structure analyzed in Remark~\ref{rem:gen-mad-cohort}, applying
a simpler outer map to the same vector of squared prediction errors: it is the plain average
$\ell^{\mathrm{pred}}(\theta;z)=|\mathcal{I}|^{-1}\sum_{i\in\mathcal{I}}e_i(\theta;z)$, without
the group-averaging and centering steps, and is likewise per-instance decomposable. Averaging
is $1$-Lipschitz with respect to $\|\cdot\|_\infty$ by the same estimate used there for the
group average, and it maps $[0,B_e^2]^{\mathcal{I}}$ into $[0,B_e^2]$. Combining with
\eqref{eq:gen-error-stability} gives $B_L\le B_e^2$ and $L_L\le 2B_eL_f$.
\end{remark}

Both prediction-side losses therefore satisfy Assumptions~\ref{assp:gen-bounded}
and~\ref{assp:gen-predictor-lip} with the same Lipschitz constant $2B_eL_f$, differing only in
their range bounds. 

\end{document}

%% file: Tables/tab_sec53_hc_main.tex
\begin{table}[htbp]
\centering
\caption{Single resource allocation results across predictor capacity ($\alpha{=}2$, $N{=}50$). Mean $\pm$ standard deviation over 5 seeds; bold marks the lowest mean in each column. 
}
\label{tab:sec53-hc-main}
\footnotesize
\setlength{\tabcolsep}{3pt}
\resizebox{\textwidth}{!}{%
\begin{tabular}{l cc ccc ccc ccc}
\toprule
 & & & \multicolumn{3}{c}{Linear} & \multicolumn{3}{c}{MLP-16} & \multicolumn{3}{c}{MLP-64} \\
\cmidrule(lr){4-6}\cmidrule(lr){7-9}\cmidrule(lr){10-12}
Method & $\lambda$ & $\mu$ & Regret & MAD & MSE & Regret & MAD & MSE & Regret & MAD & MSE \\
\midrule
PTO & 0 & -- & $0.193 \pm 0.005$ & $29.5 \pm 5.1$ & $136.5 \pm 1.7$ & $0.155 \pm 0.035$ & $40.1 \pm 41.3$ & $167.8 \pm 109.3$ & $0.139 \pm 0.003$ & $21.2 \pm 3.0$ & $\mathbf{118.9 \pm 1.9}$ \\
SAA & -- & -- & $0.279 \pm 0.003$ & $59.2 \pm 10.8$ & $249.3 \pm 4.5$ & $0.279 \pm 0.003$ & $59.2 \pm 10.8$ & $249.3 \pm 4.5$ & $0.279 \pm 0.003$ & $59.2 \pm 10.8$ & $249.3 \pm 4.5$ \\
DFL & 0 & 0 & $0.229 \pm 0.007$ & $101.3 \pm 11.8$ & $346.5 \pm 7.3$ & $0.129 \pm 0.003$ & $96.3 \pm 14.7$ & $342.9 \pm 25.2$ & $0.129 \pm 0.003$ & $71.0 \pm 11.0$ & $270.3 \pm 18.5$ \\
\midrule
FPTO & 1 & -- & $0.179 \pm 0.016$ & $\mathbf{27.8 \pm 5.6}$ & $\mathbf{135.2 \pm 2.8}$ & $0.160 \pm 0.034$ & $39.9 \pm 42.3$ & $171.2 \pm 108.7$ & $0.142 \pm 0.002$ & $\mathbf{20.6 \pm 3.5}$ & $122.6 \pm 1.9$ \\
Regret-and-MAD & 1 & 0 & $0.132 \pm 0.002$ & $37.1 \pm 3.0$ & $199.7 \pm 15.6$ & $0.129 \pm 0.002$ & $25.3 \pm 5.5$ & $136.4 \pm 4.4$ & $0.129 \pm 0.003$ & $27.2 \pm 5.5$ & $144.5 \pm 11.7$ \\
Regret-and-MSE & 0 & 0.5 & $\mathbf{0.130 \pm 0.002}$ & $41.7 \pm 6.1$ & $191.7 \pm 13.5$ & $\mathbf{0.129 \pm 0.003}$ & $28.2 \pm 2.6$ & $125.8 \pm 2.2$ & $\mathbf{0.129 \pm 0.003}$ & $27.3 \pm 5.1$ & $129.5 \pm 2.7$ \\
\midrule
FDFL-Scal & 1 & 1 & $0.131 \pm 0.002$ & $34.1 \pm 2.7$ & $175.5 \pm 4.4$ & $0.131 \pm 0.002$ & $24.3 \pm 3.0$ & $\mathbf{123.9 \pm 1.8}$ & $0.130 \pm 0.003$ & $25.3 \pm 3.4$ & $126.8 \pm 4.6$ \\
FDFL-PCGrad & -- & -- & $0.216 \pm 0.002$ & $63.7 \pm 9.8$ & $306.6 \pm 4.8$ & $0.134 \pm 0.002$ & $\mathbf{22.4 \pm 4.3}$ & $126.0 \pm 6.4$ & $0.130 \pm 0.003$ & $24.8 \pm 4.7$ & $136.0 \pm 10.1$ \\
FDFL-NashMTL & -- & -- & $0.172 \pm 0.001$ & $76.7 \pm 8.4$ & $298.4 \pm 7.5$ & $0.130 \pm 0.002$ & $37.7 \pm 21.3$ & $163.0 \pm 64.6$ & $0.129 \pm 0.003$ & $25.8 \pm 4.4$ & $126.4 \pm 2.8$ \\
\bottomrule
\end{tabular}
}
\end{table}

%% file: Tables/tab_sec53_md_main.tex
\begin{table}[htbp]
\centering
\caption{Multiple resource allocation results across group imbalance $\ell$ ($\alpha{=}2$, $N=50$, MLP-64, $K{=}2$). Mean $\pm$ standard deviation over 5 seeds; bold marks the lowest mean in each column. 
}
\label{tab:sec53-md-main}
\footnotesize
\setlength{\tabcolsep}{3pt}
\resizebox{\textwidth}{!}{%
\begin{tabular}{l cc ccccc ccccc}
\toprule
 & & & \multicolumn{5}{c}{Normalized regret} & \multicolumn{5}{c}{MAD} \\
\cmidrule(lr){4-8}\cmidrule(lr){9-13}
Method & $\lambda$ & $\mu$ & $\ell{=}0$ & $.2$ & $.4$ & $.6$ & $.8$ & $\ell{=}0$ & $.2$ & $.4$ & $.6$ & $.8$ \\
\midrule
PTO & 0 & -- & $0.135 \pm 0.038$ & $0.164 \pm 0.042$ & $0.193 \pm 0.039$ & $0.215 \pm 0.046$ & $0.216 \pm 0.044$ & $0.042 \pm 0.012$ & $0.078 \pm 0.040$ & $0.149 \pm 0.079$ & $0.228 \pm 0.117$ & $0.311 \pm 0.158$ \\
SAA & -- & -- & $0.664 \pm 0.084$ & $0.663 \pm 0.082$ & $0.660 \pm 0.079$ & $0.656 \pm 0.080$ & $0.656 \pm 0.086$ & $0.187 \pm 0.085$ & $0.207 \pm 0.079$ & $0.233 \pm 0.075$ & $0.260 \pm 0.072$ & $0.286 \pm 0.072$ \\
DFL & 0 & 0 & $0.139 \pm 0.039$ & $0.165 \pm 0.036$ & $0.196 \pm 0.045$ & $0.216 \pm 0.058$ & $0.219 \pm 0.056$ & $0.068 \pm 0.037$ & $0.114 \pm 0.035$ & $0.177 \pm 0.042$ & $0.176 \pm 0.035$ & $0.393 \pm 0.153$ \\
\midrule
FPTO & 1 & -- & $0.138 \pm 0.038$ & $0.167 \pm 0.038$ & $0.196 \pm 0.037$ & $0.228 \pm 0.038$ & $0.225 \pm 0.039$ & $0.042 \pm 0.013$ & $0.072 \pm 0.036$ & $0.117 \pm 0.064$ & $0.142 \pm 0.082$ & $0.174 \pm 0.104$ \\
Regret-and-MAD & 1 & 0 & $0.137 \pm 0.038$ & $0.164 \pm 0.036$ & $0.198 \pm 0.045$ & $0.217 \pm 0.053$ & $0.218 \pm 0.055$ & $0.058 \pm 0.028$ & $0.095 \pm 0.032$ & $0.109 \pm 0.027$ & $\mathbf{0.102 \pm 0.019}$ & $\mathbf{0.132 \pm 0.056}$ \\
Regret-and-MSE & 0 & 0.5 & $0.137 \pm 0.038$ & $0.166 \pm 0.041$ & $0.195 \pm 0.043$ & $0.214 \pm 0.057$ & $0.219 \pm 0.057$ & $0.049 \pm 0.012$ & $0.065 \pm 0.021$ & $0.103 \pm 0.026$ & $0.150 \pm 0.047$ & $0.160 \pm 0.089$ \\
\midrule
FDFL-Scal & 1 & 1 & $0.137 \pm 0.039$ & $0.164 \pm 0.039$ & $0.194 \pm 0.044$ & $0.215 \pm 0.054$ & $0.215 \pm 0.057$ & $0.046 \pm 0.012$ & $\mathbf{0.059 \pm 0.024}$ & $\mathbf{0.084 \pm 0.031}$ & $0.113 \pm 0.057$ & $0.142 \pm 0.083$ \\
FDFL-PCGrad & -- & -- & $0.136 \pm 0.038$ & $0.162 \pm 0.039$ & $\mathbf{0.190 \pm 0.040}$ & $0.211 \pm 0.051$ & $0.215 \pm 0.057$ & $0.042 \pm 0.013$ & $0.072 \pm 0.034$ & $0.105 \pm 0.059$ & $0.126 \pm 0.064$ & $0.144 \pm 0.080$ \\
FDFL-NashMTL & -- & -- & $\mathbf{0.132 \pm 0.036}$ & $\mathbf{0.156 \pm 0.037}$ & $0.193 \pm 0.048$ & $\mathbf{0.209 \pm 0.058}$ & $\mathbf{0.213 \pm 0.055}$ & $\mathbf{0.041 \pm 0.012}$ & $0.074 \pm 0.038$ & $0.150 \pm 0.087$ & $0.230 \pm 0.126$ & $0.302 \pm 0.155$ \\
\bottomrule
\end{tabular}
}
\end{table}

%% file: Tables/tab_app_hc_capacity.tex
\begin{table}[htbp]
\centering
\caption{Single resource allocation, full method pool across predictor capacity ($\alpha{=}2$, $N{=}50$).}
\label{tab:app-hc-capacity}
\scriptsize
\setlength{\tabcolsep}{3pt}
\resizebox{\textwidth}{!}{%
\begin{tabular}{l cc ccc ccc ccc}
\toprule
 & & & \multicolumn{3}{c}{Linear} & \multicolumn{3}{c}{MLP-16} & \multicolumn{3}{c}{MLP-64} \\
\cmidrule(lr){4-6}\cmidrule(lr){7-9}\cmidrule(lr){10-12}
Method & $\lambda$ & $\mu$ & Reg & MAD & MSE & Reg & MAD & MSE & Reg & MAD & MSE \\
\midrule
PTO & 0 & -- & $0.193 \pm 0.005$ & $29.5 \pm 5.1$ & $136.5 \pm 1.7$ & $0.155 \pm 0.035$ & $40.1 \pm 41.3$ & $167.8 \pm 109.3$ & $0.139 \pm 0.003$ & $21.2 \pm 3.0$ & $118.9 \pm 1.9$ \\
SAA & -- & -- & $0.279 \pm 0.003$ & $59.2 \pm 10.8$ & $249.3 \pm 4.5$ & $0.279 \pm 0.003$ & $59.2 \pm 10.8$ & $249.3 \pm 4.5$ & $0.279 \pm 0.003$ & $59.2 \pm 10.8$ & $249.3 \pm 4.5$ \\
WDRO & -- & -- & $0.172 \pm 0.004$ & $29.2 \pm 4.0$ & $138.1 \pm 1.8$ & $0.241 \pm 0.033$ & $105.9 \pm 9.1$ & $376.0 \pm 27.1$ & $0.138 \pm 0.002$ & $22.1 \pm 2.5$ & $\mathbf{118.2 \pm 1.5}$ \\
DFL & 0 & 0 & $0.229 \pm 0.007$ & $101.3 \pm 11.8$ & $346.5 \pm 7.3$ & $0.129 \pm 0.003$ & $96.3 \pm 14.7$ & $342.9 \pm 25.2$ & $0.129 \pm 0.003$ & $71.0 \pm 11.0$ & $270.3 \pm 18.5$ \\
\midrule
FPTO & 1 & -- & $0.179 \pm 0.016$ & $\mathbf{27.8 \pm 5.6}$ & $\mathbf{135.2 \pm 2.8}$ & $0.160 \pm 0.034$ & $39.9 \pm 42.3$ & $171.2 \pm 108.7$ & $0.142 \pm 0.002$ & $\mathbf{20.6 \pm 3.5}$ & $122.6 \pm 1.9$ \\
Regret-and-MAD & 1 & 0 & $0.132 \pm 0.002$ & $37.1 \pm 3.0$ & $199.7 \pm 15.6$ & $0.129 \pm 0.002$ & $25.3 \pm 5.5$ & $136.4 \pm 4.4$ & $0.129 \pm 0.003$ & $27.2 \pm 5.5$ & $144.5 \pm 11.7$ \\
Regret-and-MSE & 0 & 0.1 & $0.171 \pm 0.009$ & $82.2 \pm 9.4$ & $304.5 \pm 7.4$ & $0.129 \pm 0.002$ & $55.9 \pm 27.4$ & $209.0 \pm 102.6$ & $0.129 \pm 0.003$ & $40.4 \pm 16.7$ & $166.7 \pm 52.3$ \\
Regret-and-MSE & 0 & 0.5 & $\mathbf{0.130 \pm 0.002}$ & $41.7 \pm 6.1$ & $191.7 \pm 13.5$ & $\mathbf{0.129 \pm 0.003}$ & $28.2 \pm 2.6$ & $125.8 \pm 2.2$ & $\mathbf{0.129 \pm 0.003}$ & $27.3 \pm 5.1$ & $129.5 \pm 2.7$ \\
Regret-and-MSE & 0 & 1 & $0.131 \pm 0.002$ & $35.7 \pm 2.7$ & $168.9 \pm 3.7$ & $0.129 \pm 0.003$ & $24.5 \pm 3.6$ & $\mathbf{123.1 \pm 2.0}$ & $0.129 \pm 0.003$ & $24.8 \pm 4.4$ & $126.0 \pm 3.0$ \\
FDFL-FPLG & 0 & -- & $0.205 \pm 0.001$ & $94.5 \pm 11.6$ & $332.1 \pm 6.2$ & $0.129 \pm 0.002$ & $47.6 \pm 30.0$ & $194.7 \pm 77.8$ & $0.130 \pm 0.003$ & $26.4 \pm 4.7$ & $127.8 \pm 3.5$ \\
\midrule
FDFL-Scal & 1 & 0.1 & $0.131 \pm 0.002$ & $36.0 \pm 4.5$ & $193.0 \pm 5.8$ & $0.130 \pm 0.001$ & $25.7 \pm 3.5$ & $139.7 \pm 10.2$ & $0.129 \pm 0.002$ & $26.3 \pm 6.9$ & $145.9 \pm 9.4$ \\
FDFL-Scal & 1 & 0.5 & $0.131 \pm 0.002$ & $34.3 \pm 3.7$ & $179.2 \pm 5.0$ & $0.130 \pm 0.002$ & $25.6 \pm 3.5$ & $127.8 \pm 2.9$ & $0.130 \pm 0.003$ & $25.8 \pm 6.1$ & $135.0 \pm 8.1$ \\
FDFL-Scal & 1 & 1 & $0.131 \pm 0.002$ & $34.1 \pm 2.7$ & $175.5 \pm 4.4$ & $0.131 \pm 0.002$ & $24.3 \pm 3.0$ & $123.9 \pm 1.8$ & $0.130 \pm 0.003$ & $25.3 \pm 3.4$ & $126.8 \pm 4.6$ \\
FDFL-FPLG & 1 & -- & $0.188 \pm 0.003$ & $77.3 \pm 10.2$ & $312.6 \pm 6.5$ & $0.129 \pm 0.002$ & $26.2 \pm 5.8$ & $144.5 \pm 14.8$ & $0.130 \pm 0.002$ & $24.6 \pm 4.1$ & $129.3 \pm 4.0$ \\
FDFL-PCGrad & -- & -- & $0.216 \pm 0.002$ & $63.7 \pm 9.8$ & $306.6 \pm 4.8$ & $0.134 \pm 0.002$ & $\mathbf{22.4 \pm 4.3}$ & $126.0 \pm 6.4$ & $0.130 \pm 0.003$ & $24.8 \pm 4.7$ & $136.0 \pm 10.1$ \\
FDFL-MGDA & -- & -- & $0.218 \pm 0.001$ & $67.5 \pm 8.8$ & $313.6 \pm 3.9$ & $0.141 \pm 0.006$ & $32.7 \pm 3.3$ & $181.9 \pm 24.5$ & $0.134 \pm 0.002$ & $28.1 \pm 5.1$ & $134.2 \pm 5.1$ \\
FDFL-NashMTL & -- & -- & $0.172 \pm 0.001$ & $76.7 \pm 8.4$ & $298.4 \pm 7.5$ & $0.130 \pm 0.002$ & $37.7 \pm 21.3$ & $163.0 \pm 64.6$ & $0.129 \pm 0.003$ & $25.8 \pm 4.4$ & $126.4 \pm 2.8$ \\
\bottomrule
\end{tabular}
}
\end{table}

%% file: Tables/tab_app_hc_alpha.tex
\begin{table}[htbp]
\centering
\caption{Single resource allocation, full method pool across decision fairness parameter $\alpha$ (MLP-64, $N{=}50$).}
\label{tab:app-hc-alpha}
\scriptsize
\setlength{\tabcolsep}{3pt}
\resizebox{\textwidth}{!}{%
\begin{tabular}{l cc ccc ccc ccc ccc}
\toprule
 & & & \multicolumn{3}{c}{$\alpha{=}0.5$} & \multicolumn{3}{c}{$\alpha{=}1.5$} & \multicolumn{3}{c}{$\alpha{=}2$} & \multicolumn{3}{c}{$\alpha{=}4$} \\
\cmidrule(lr){4-6}\cmidrule(lr){7-9}\cmidrule(lr){10-12}\cmidrule(lr){13-15}
Method & $\lambda$ & $\mu$ & Reg & MAD & MSE & Reg & MAD & MSE & Reg & MAD & MSE & Reg & MAD & MSE \\
\midrule
PTO & 0 & -- & $0.044 \pm 0.000$ & $21.1 \pm 3.4$ & $119.0 \pm 1.8$ & $0.013 \pm 0.000$ & $21.8 \pm 3.4$ & $120.8 \pm 1.4$ & $0.139 \pm 0.003$ & $21.2 \pm 3.0$ & $118.9 \pm 1.9$ & $40.127 \pm 7.164$ & $23.9 \pm 2.2$ & $137.6 \pm 3.3$ \\
SAA & -- & -- & $0.075 \pm 0.001$ & $59.2 \pm 10.8$ & $249.3 \pm 4.5$ & $0.026 \pm 0.000$ & $59.2 \pm 10.8$ & $249.3 \pm 4.5$ & $0.279 \pm 0.003$ & $59.2 \pm 10.8$ & $249.3 \pm 4.5$ & $153.145 \pm 6.777$ & $59.2 \pm 10.8$ & $249.3 \pm 4.5$ \\
WDRO & -- & -- & $\mathbf{0.044 \pm 0.001}$ & $21.2 \pm 2.5$ & $\mathbf{116.6 \pm 2.9}$ & $0.012 \pm 0.000$ & $22.2 \pm 2.8$ & $\mathbf{118.2 \pm 1.5}$ & $0.138 \pm 0.002$ & $22.1 \pm 2.5$ & $\mathbf{118.2 \pm 1.5}$ & $41.325 \pm 7.755$ & $24.2 \pm 2.1$ & $\mathbf{136.9 \pm 3.2}$ \\
DFL & 0 & 0 & $0.053 \pm 0.001$ & $105.9 \pm 10.9$ & $380.5 \pm 10.6$ & $0.012 \pm 0.000$ & $62.3 \pm 16.2$ & $233.0 \pm 50.6$ & $0.129 \pm 0.003$ & $71.0 \pm 11.0$ & $270.3 \pm 18.5$ & $32.410 \pm 14.715$ & $109.1 \pm 11.7$ & $377.9 \pm 8.1$ \\
\midrule
FPTO & 1 & -- & $0.045 \pm 0.001$ & $20.2 \pm 2.8$ & $120.4 \pm 3.5$ & $0.013 \pm 0.000$ & $21.9 \pm 3.3$ & $129.9 \pm 5.8$ & $0.142 \pm 0.002$ & $\mathbf{20.6 \pm 3.5}$ & $122.6 \pm 1.9$ & $39.491 \pm 6.286$ & $\mathbf{22.0 \pm 1.9}$ & $144.8 \pm 2.8$ \\
Regret-and-MAD & 1 & 0 & $0.061 \pm 0.001$ & $103.7 \pm 10.5$ & $371.6 \pm 8.1$ & $0.017 \pm 0.001$ & $20.8 \pm 4.5$ & $183.4 \pm 28.2$ & $0.129 \pm 0.003$ & $27.2 \pm 5.5$ & $144.5 \pm 11.7$ & $32.410 \pm 14.715$ & $109.1 \pm 11.7$ & $377.9 \pm 8.1$ \\
Regret-and-MSE & 0 & 0.1 & $0.044 \pm 0.001$ & $20.2 \pm 3.6$ & $118.5 \pm 2.8$ & $0.012 \pm 0.000$ & $21.5 \pm 4.0$ & $118.5 \pm 1.7$ & $0.129 \pm 0.003$ & $40.4 \pm 16.7$ & $166.7 \pm 52.3$ & $32.410 \pm 14.715$ & $109.1 \pm 11.7$ & $377.9 \pm 8.1$ \\
Regret-and-MSE & 0 & 0.5 & $0.044 \pm 0.000$ & $20.6 \pm 3.4$ & $119.5 \pm 1.9$ & $0.013 \pm 0.000$ & $21.4 \pm 3.5$ & $119.3 \pm 1.5$ & $\mathbf{0.129 \pm 0.003}$ & $27.3 \pm 5.1$ & $129.5 \pm 2.7$ & $32.410 \pm 14.715$ & $109.1 \pm 11.7$ & $377.9 \pm 8.1$ \\
Regret-and-MSE & 0 & 1 & $0.044 \pm 0.000$ & $21.1 \pm 3.5$ & $119.0 \pm 1.8$ & $0.012 \pm 0.000$ & $21.1 \pm 3.5$ & $119.1 \pm 1.8$ & $0.129 \pm 0.003$ & $24.8 \pm 4.4$ & $126.0 \pm 3.0$ & $32.410 \pm 14.716$ & $109.1 \pm 11.7$ & $377.9 \pm 8.1$ \\
FDFL-FPLG & 0 & -- & $0.053 \pm 0.001$ & $94.1 \pm 9.9$ & $356.7 \pm 17.7$ & $0.012 \pm 0.000$ & $21.6 \pm 3.5$ & $118.7 \pm 2.0$ & $0.130 \pm 0.003$ & $26.4 \pm 4.7$ & $127.8 \pm 3.5$ & $27.557 \pm 4.904$ & $24.7 \pm 2.7$ & $139.3 \pm 4.6$ \\
\midrule
FDFL-Scal & 1 & 0.1 & $0.048 \pm 0.001$ & $20.6 \pm 2.5$ & $132.2 \pm 3.9$ & $0.014 \pm 0.000$ & $24.2 \pm 3.2$ & $136.6 \pm 10.2$ & $0.129 \pm 0.002$ & $26.3 \pm 6.9$ & $145.9 \pm 9.4$ & $32.410 \pm 14.715$ & $109.1 \pm 11.7$ & $377.9 \pm 8.1$ \\
FDFL-Scal & 1 & 0.5 & $0.046 \pm 0.001$ & $21.4 \pm 3.4$ & $124.2 \pm 3.6$ & $0.013 \pm 0.000$ & $21.8 \pm 2.9$ & $124.1 \pm 2.1$ & $0.130 \pm 0.003$ & $25.8 \pm 6.1$ & $135.0 \pm 8.1$ & $32.410 \pm 14.716$ & $109.1 \pm 11.7$ & $377.9 \pm 8.1$ \\
FDFL-Scal & 1 & 1 & $0.045 \pm 0.001$ & $\mathbf{20.1 \pm 2.8}$ & $120.3 \pm 3.4$ & $0.013 \pm 0.000$ & $\mathbf{20.6 \pm 3.9}$ & $122.0 \pm 2.0$ & $0.130 \pm 0.003$ & $25.3 \pm 3.4$ & $126.8 \pm 4.6$ & $32.410 \pm 14.716$ & $109.1 \pm 11.7$ & $377.9 \pm 8.1$ \\
FDFL-FPLG & 1 & -- & $0.061 \pm 0.001$ & $103.7 \pm 10.4$ & $371.5 \pm 8.1$ & $0.012 \pm 0.000$ & $24.5 \pm 4.3$ & $141.1 \pm 2.7$ & $0.130 \pm 0.002$ & $24.6 \pm 4.1$ & $129.3 \pm 4.0$ & $27.452 \pm 4.967$ & $25.0 \pm 2.4$ & $141.7 \pm 6.8$ \\
FDFL-PCGrad & -- & -- & $0.046 \pm 0.000$ & $21.4 \pm 3.6$ & $121.4 \pm 1.8$ & $0.012 \pm 0.000$ & $23.8 \pm 5.1$ & $129.7 \pm 5.7$ & $0.130 \pm 0.003$ & $24.8 \pm 4.7$ & $136.0 \pm 10.1$ & $29.039 \pm 3.943$ & $23.8 \pm 5.4$ & $143.3 \pm 10.0$ \\
FDFL-MGDA & -- & -- & $0.053 \pm 0.001$ & $105.8 \pm 11.0$ & $380.1 \pm 9.9$ & $\mathbf{0.012 \pm 0.000}$ & $34.9 \pm 8.2$ & $147.0 \pm 9.3$ & $0.134 \pm 0.002$ & $28.1 \pm 5.1$ & $134.2 \pm 5.1$ & $33.158 \pm 15.415$ & $106.8 \pm 10.7$ & $373.5 \pm 10.0$ \\
FDFL-NashMTL & -- & -- & $0.045 \pm 0.001$ & $21.7 \pm 6.6$ & $123.1 \pm 2.5$ & $0.012 \pm 0.000$ & $25.4 \pm 6.3$ & $122.9 \pm 4.0$ & $0.129 \pm 0.003$ & $25.8 \pm 4.4$ & $126.4 \pm 2.8$ & $\mathbf{27.406 \pm 4.945}$ & $24.4 \pm 2.2$ & $139.1 \pm 4.6$ \\
\bottomrule
\end{tabular}
}
\end{table}

%% file: Tables/tab_app_md_capacity.tex
\begin{table}[htbp]
\centering
\caption{Multiple resource allocation, full method pool across predictor capacity ($\alpha{=}2$, imbalance $0.6$, $K{=}2$).}
\label{tab:app-md-capacity}
\scriptsize
\setlength{\tabcolsep}{3pt}
\resizebox{\textwidth}{!}{%
\begin{tabular}{l cc ccc ccc ccc}
\toprule
 & & & \multicolumn{3}{c}{Linear} & \multicolumn{3}{c}{MLP-16} & \multicolumn{3}{c}{MLP-64} \\
\cmidrule(lr){4-6}\cmidrule(lr){7-9}\cmidrule(lr){10-12}
Method & $\lambda$ & $\mu$ & Reg & MAD & MSE & Reg & MAD & MSE & Reg & MAD & MSE \\
\midrule
PTO & 0 & -- & $0.345 \pm 0.063$ & $0.189 \pm 0.110$ & $\mathbf{1.707 \pm 0.761}$ & $0.248 \pm 0.057$ & $0.222 \pm 0.128$ & $1.063 \pm 0.428$ & $0.215 \pm 0.046$ & $0.228 \pm 0.117$ & $\mathbf{0.953 \pm 0.368}$ \\
SAA & -- & -- & $0.656 \pm 0.080$ & $0.260 \pm 0.072$ & $2.772 \pm 1.169$ & $0.656 \pm 0.080$ & $0.260 \pm 0.072$ & $2.772 \pm 1.169$ & $0.656 \pm 0.080$ & $0.260 \pm 0.072$ & $2.772 \pm 1.169$ \\
WDRO & -- & -- & $0.343 \pm 0.061$ & $0.197 \pm 0.131$ & $1.731 \pm 0.760$ & $0.255 \pm 0.054$ & $0.179 \pm 0.114$ & $1.090 \pm 0.445$ & $0.215 \pm 0.049$ & $0.194 \pm 0.115$ & $0.961 \pm 0.375$ \\
DFL & 0 & 0 & $0.287 \pm 0.067$ & $0.420 \pm 0.269$ & $2.535 \pm 1.250$ & $0.224 \pm 0.049$ & $0.405 \pm 0.191$ & $2.380 \pm 1.018$ & $0.216 \pm 0.058$ & $0.176 \pm 0.035$ & $1.567 \pm 0.763$ \\
\midrule
FPTO & 1 & -- & $0.351 \pm 0.063$ & $0.181 \pm 0.100$ & $1.711 \pm 0.762$ & $0.257 \pm 0.054$ & $0.144 \pm 0.082$ & $1.111 \pm 0.448$ & $0.228 \pm 0.038$ & $0.142 \pm 0.082$ & $0.998 \pm 0.402$ \\
Regret-and-MAD & 1 & 0 & $0.287 \pm 0.067$ & $0.412 \pm 0.264$ & $2.509 \pm 1.239$ & $0.222 \pm 0.054$ & $0.318 \pm 0.137$ & $2.075 \pm 0.924$ & $0.217 \pm 0.053$ & $0.102 \pm 0.019$ & $1.351 \pm 0.652$ \\
Regret-and-MSE & 0 & 0.1 & $0.287 \pm 0.067$ & $0.417 \pm 0.266$ & $2.524 \pm 1.244$ & $0.224 \pm 0.050$ & $0.367 \pm 0.152$ & $2.246 \pm 0.972$ & $0.215 \pm 0.055$ & $0.202 \pm 0.021$ & $1.722 \pm 0.900$ \\
Regret-and-MSE & 0 & 0.5 & $0.287 \pm 0.067$ & $0.406 \pm 0.257$ & $2.482 \pm 1.220$ & $0.223 \pm 0.052$ & $0.249 \pm 0.086$ & $1.837 \pm 0.784$ & $0.214 \pm 0.057$ & $0.150 \pm 0.047$ & $1.334 \pm 0.638$ \\
Regret-and-MSE & 0 & 1 & $0.288 \pm 0.067$ & $0.391 \pm 0.246$ & $2.431 \pm 1.193$ & $0.225 \pm 0.052$ & $0.188 \pm 0.075$ & $1.547 \pm 0.653$ & $0.215 \pm 0.057$ & $0.143 \pm 0.070$ & $1.122 \pm 0.495$ \\
FDFL-FPLG & 0 & -- & $0.296 \pm 0.065$ & $0.283 \pm 0.203$ & $2.074 \pm 0.937$ & $0.221 \pm 0.054$ & $0.120 \pm 0.067$ & $1.242 \pm 0.497$ & $0.214 \pm 0.051$ & $0.131 \pm 0.086$ & $1.066 \pm 0.435$ \\
\midrule
FDFL-Scal & 1 & 0.1 & $\mathbf{0.287 \pm 0.067}$ & $0.409 \pm 0.262$ & $2.499 \pm 1.233$ & $0.223 \pm 0.051$ & $0.283 \pm 0.119$ & $1.971 \pm 0.888$ & $0.216 \pm 0.057$ & $0.124 \pm 0.038$ & $1.467 \pm 0.787$ \\
FDFL-Scal & 1 & 0.5 & $0.287 \pm 0.067$ & $0.398 \pm 0.253$ & $2.457 \pm 1.210$ & $0.226 \pm 0.050$ & $0.185 \pm 0.080$ & $1.652 \pm 0.762$ & $0.213 \pm 0.058$ & $0.110 \pm 0.029$ & $1.285 \pm 0.591$ \\
FDFL-Scal & 1 & 1 & $0.288 \pm 0.067$ & $0.384 \pm 0.242$ & $2.406 \pm 1.183$ & $0.228 \pm 0.056$ & $0.171 \pm 0.086$ & $1.530 \pm 0.703$ & $0.215 \pm 0.054$ & $0.113 \pm 0.057$ & $1.140 \pm 0.497$ \\
FDFL-FPLG & 1 & -- & $0.295 \pm 0.066$ & $0.268 \pm 0.196$ & $2.040 \pm 0.923$ & $0.219 \pm 0.052$ & $\mathbf{0.108 \pm 0.041}$ & $1.329 \pm 0.582$ & $0.217 \pm 0.051$ & $\mathbf{0.089 \pm 0.036}$ & $1.174 \pm 0.531$ \\
FDFL-PCGrad & -- & -- & $0.305 \pm 0.068$ & $\mathbf{0.173 \pm 0.102}$ & $1.747 \pm 0.771$ & $0.216 \pm 0.050$ & $0.120 \pm 0.056$ & $1.082 \pm 0.465$ & $0.211 \pm 0.051$ & $0.126 \pm 0.064$ & $1.040 \pm 0.437$ \\
FDFL-MGDA & -- & -- & $0.475 \pm 0.077$ & $0.270 \pm 0.156$ & $2.593 \pm 1.076$ & $0.431 \pm 0.079$ & $0.129 \pm 0.061$ & $1.864 \pm 0.775$ & $0.238 \pm 0.045$ & $0.098 \pm 0.048$ & $1.122 \pm 0.504$ \\
FDFL-NashMTL & -- & -- & $0.296 \pm 0.070$ & $0.210 \pm 0.159$ & $1.831 \pm 0.818$ & $\mathbf{0.216 \pm 0.049}$ & $0.225 \pm 0.121$ & $\mathbf{0.999 \pm 0.390}$ & $\mathbf{0.209 \pm 0.058}$ & $0.230 \pm 0.126$ & $0.963 \pm 0.382$ \\
\bottomrule
\end{tabular}
}
\end{table}

%% file: Tables/tab_app_md_alpha.tex
\begin{table}[htbp]
\centering
\caption{Multiple resource allocation, full method pool across decision fairness parameter $\alpha$ (MLP-64, imbalance $0.6$, $K{=}2$).}
\label{tab:app-md-alpha}
\scriptsize
\setlength{\tabcolsep}{3pt}
\resizebox{\textwidth}{!}{%
\begin{tabular}{l cc ccc ccc ccc}
\toprule
 & & & \multicolumn{3}{c}{$\alpha{=}0.5$} & \multicolumn{3}{c}{$\alpha{=}1.5$} & \multicolumn{3}{c}{$\alpha{=}2$} \\
\cmidrule(lr){4-6}\cmidrule(lr){7-9}\cmidrule(lr){10-12}
Method & $\lambda$ & $\mu$ & Reg & MAD & MSE & Reg & MAD & MSE & Reg & MAD & MSE \\
\midrule
PTO & 0 & -- & $\mathbf{0.016 \pm 0.004}$ & $0.222 \pm 0.118$ & $\mathbf{0.960 \pm 0.369}$ & $0.035 \pm 0.005$ & $0.228 \pm 0.117$ & $\mathbf{0.953 \pm 0.368}$ & $0.215 \pm 0.046$ & $0.228 \pm 0.117$ & $\mathbf{0.953 \pm 0.368}$ \\
SAA & -- & -- & $0.053 \pm 0.014$ & $0.260 \pm 0.072$ & $2.772 \pm 1.169$ & $0.150 \pm 0.030$ & $0.260 \pm 0.072$ & $2.772 \pm 1.169$ & $0.656 \pm 0.080$ & $0.260 \pm 0.072$ & $2.772 \pm 1.169$ \\
WDRO & -- & -- & $0.016 \pm 0.004$ & $0.194 \pm 0.115$ & $0.961 \pm 0.375$ & $0.035 \pm 0.006$ & $0.194 \pm 0.115$ & $0.961 \pm 0.375$ & $0.215 \pm 0.049$ & $0.194 \pm 0.115$ & $0.961 \pm 0.375$ \\
DFL & 0 & 0 & $0.021 \pm 0.007$ & $0.562 \pm 0.247$ & $3.088 \pm 1.453$ & $0.038 \pm 0.009$ & $0.406 \pm 0.223$ & $2.424 \pm 1.271$ & $0.216 \pm 0.058$ & $0.176 \pm 0.035$ & $1.567 \pm 0.763$ \\
\midrule
FPTO & 1 & -- & $0.017 \pm 0.004$ & $0.135 \pm 0.076$ & $1.017 \pm 0.405$ & $0.038 \pm 0.004$ & $0.142 \pm 0.082$ & $0.998 \pm 0.402$ & $0.228 \pm 0.038$ & $0.142 \pm 0.082$ & $0.998 \pm 0.402$ \\
Regret-and-MAD & 1 & 0 & $0.021 \pm 0.007$ & $\mathbf{0.090 \pm 0.029}$ & $1.368 \pm 0.593$ & $0.037 \pm 0.008$ & $0.102 \pm 0.035$ & $1.288 \pm 0.603$ & $0.217 \pm 0.053$ & $0.102 \pm 0.019$ & $1.351 \pm 0.652$ \\
Regret-and-MSE & 0 & 0.1 & $0.018 \pm 0.006$ & $0.193 \pm 0.062$ & $1.619 \pm 0.738$ & $0.037 \pm 0.008$ & $0.166 \pm 0.095$ & $1.092 \pm 0.449$ & $0.215 \pm 0.055$ & $0.202 \pm 0.021$ & $1.722 \pm 0.900$ \\
Regret-and-MSE & 0 & 0.5 & $0.019 \pm 0.006$ & $0.147 \pm 0.088$ & $1.183 \pm 0.440$ & $0.036 \pm 0.008$ & $0.200 \pm 0.112$ & $1.012 \pm 0.395$ & $0.214 \pm 0.057$ & $0.150 \pm 0.047$ & $1.334 \pm 0.638$ \\
Regret-and-MSE & 0 & 1 & $0.018 \pm 0.005$ & $0.166 \pm 0.095$ & $1.079 \pm 0.403$ & $0.036 \pm 0.008$ & $0.214 \pm 0.117$ & $0.989 \pm 0.385$ & $0.215 \pm 0.057$ & $0.143 \pm 0.070$ & $1.122 \pm 0.495$ \\
FDFL-FPLG & 0 & -- & $0.020 \pm 0.006$ & $0.308 \pm 0.221$ & $1.972 \pm 0.829$ & $0.036 \pm 0.007$ & $0.140 \pm 0.075$ & $1.055 \pm 0.439$ & $0.214 \pm 0.051$ & $0.131 \pm 0.086$ & $1.066 \pm 0.435$ \\
\midrule
FDFL-Scal & 1 & 0.1 & $0.019 \pm 0.006$ & $0.092 \pm 0.030$ & $1.266 \pm 0.541$ & $0.037 \pm 0.008$ & $0.093 \pm 0.033$ & $1.186 \pm 0.527$ & $0.216 \pm 0.057$ & $0.124 \pm 0.038$ & $1.467 \pm 0.787$ \\
FDFL-Scal & 1 & 0.5 & $0.019 \pm 0.006$ & $0.116 \pm 0.058$ & $1.216 \pm 0.466$ & $0.036 \pm 0.008$ & $0.108 \pm 0.049$ & $1.081 \pm 0.447$ & $0.213 \pm 0.058$ & $0.110 \pm 0.029$ & $1.285 \pm 0.591$ \\
FDFL-Scal & 1 & 1 & $0.017 \pm 0.006$ & $0.119 \pm 0.064$ & $1.106 \pm 0.436$ & $0.036 \pm 0.007$ & $0.135 \pm 0.076$ & $1.022 \pm 0.409$ & $0.215 \pm 0.054$ & $0.113 \pm 0.057$ & $1.140 \pm 0.497$ \\
FDFL-FPLG & 1 & -- & $0.019 \pm 0.006$ & $0.099 \pm 0.034$ & $1.384 \pm 0.646$ & $0.036 \pm 0.007$ & $\mathbf{0.086 \pm 0.035}$ & $1.156 \pm 0.532$ & $0.217 \pm 0.051$ & $\mathbf{0.089 \pm 0.036}$ & $1.174 \pm 0.531$ \\
FDFL-PCGrad & -- & -- & $0.019 \pm 0.005$ & $0.129 \pm 0.072$ & $1.174 \pm 0.491$ & $0.037 \pm 0.008$ & $0.118 \pm 0.056$ & $1.062 \pm 0.462$ & $0.211 \pm 0.051$ & $0.126 \pm 0.064$ & $1.040 \pm 0.437$ \\
FDFL-MGDA & -- & -- & $0.018 \pm 0.006$ & $0.119 \pm 0.072$ & $1.091 \pm 0.452$ & $0.038 \pm 0.006$ & $0.118 \pm 0.066$ & $1.081 \pm 0.476$ & $0.238 \pm 0.045$ & $0.098 \pm 0.048$ & $1.122 \pm 0.504$ \\
FDFL-NashMTL & -- & -- & $0.017 \pm 0.005$ & $0.212 \pm 0.111$ & $0.999 \pm 0.376$ & $\mathbf{0.035 \pm 0.007}$ & $0.223 \pm 0.118$ & $0.955 \pm 0.379$ & $\mathbf{0.209 \pm 0.058}$ & $0.230 \pm 0.126$ & $0.963 \pm 0.382$ \\
\bottomrule
\end{tabular}
}
\end{table}

%% file: Tables/tab_app_md_imbalance.tex
\begin{table}[htbp]
\centering
\caption{Multiple resource allocation, full method pool across group imbalance $\ell$ ($\alpha{=}2$, MLP-64, $K{=}2$).}
\label{tab:app-md-imbalance}
\scriptsize
\setlength{\tabcolsep}{3pt}
\resizebox{\textwidth}{!}{%
\begin{tabular}{l cc ccc ccc ccc ccc ccc}
\toprule
 & & & \multicolumn{3}{c}{$\ell{=}0$} & \multicolumn{3}{c}{$\ell{=}0.2$} & \multicolumn{3}{c}{$\ell{=}0.4$} & \multicolumn{3}{c}{$\ell{=}0.6$} & \multicolumn{3}{c}{$\ell{=}0.8$} \\
\cmidrule(lr){4-6}\cmidrule(lr){7-9}\cmidrule(lr){10-12}\cmidrule(lr){13-15}\cmidrule(lr){16-18}
Method & $\lambda$ & $\mu$ & Reg & MAD & MSE & Reg & MAD & MSE & Reg & MAD & MSE & Reg & MAD & MSE & Reg & MAD & MSE \\
\midrule
PTO & 0 & -- & $0.135 \pm 0.038$ & $0.042 \pm 0.012$ & $\mathbf{0.545 \pm 0.231}$ & $0.164 \pm 0.042$ & $0.078 \pm 0.040$ & $\mathbf{0.660 \pm 0.275}$ & $0.193 \pm 0.039$ & $0.149 \pm 0.079$ & $\mathbf{0.797 \pm 0.321}$ & $0.215 \pm 0.046$ & $0.228 \pm 0.117$ & $\mathbf{0.953 \pm 0.368}$ & $0.216 \pm 0.044$ & $0.311 \pm 0.158$ & $\mathbf{1.130 \pm 0.419}$ \\
SAA & -- & -- & $0.664 \pm 0.084$ & $0.187 \pm 0.085$ & $2.408 \pm 1.055$ & $0.663 \pm 0.082$ & $0.207 \pm 0.079$ & $2.506 \pm 1.088$ & $0.660 \pm 0.079$ & $0.233 \pm 0.075$ & $2.628 \pm 1.127$ & $0.656 \pm 0.080$ & $0.260 \pm 0.072$ & $2.772 \pm 1.169$ & $0.656 \pm 0.086$ & $0.286 \pm 0.072$ & $2.936 \pm 1.215$ \\
WDRO & -- & -- & $0.133 \pm 0.034$ & $0.043 \pm 0.014$ & $0.550 \pm 0.233$ & $0.166 \pm 0.041$ & $0.070 \pm 0.042$ & $0.664 \pm 0.278$ & $0.196 \pm 0.045$ & $0.127 \pm 0.077$ & $0.802 \pm 0.326$ & $0.215 \pm 0.049$ & $0.194 \pm 0.115$ & $0.961 \pm 0.375$ & $0.227 \pm 0.043$ & $0.257 \pm 0.149$ & $1.149 \pm 0.432$ \\
DFL & 0 & 0 & $0.139 \pm 0.039$ & $0.068 \pm 0.037$ & $0.952 \pm 0.538$ & $0.165 \pm 0.036$ & $0.114 \pm 0.035$ & $1.399 \pm 0.849$ & $0.196 \pm 0.045$ & $0.177 \pm 0.042$ & $1.577 \pm 0.863$ & $0.216 \pm 0.058$ & $0.176 \pm 0.035$ & $1.567 \pm 0.763$ & $0.219 \pm 0.056$ & $0.393 \pm 0.153$ & $2.197 \pm 0.988$ \\
\midrule
FPTO & 1 & -- & $0.138 \pm 0.038$ & $0.042 \pm 0.013$ & $0.546 \pm 0.230$ & $0.167 \pm 0.038$ & $0.072 \pm 0.036$ & $0.665 \pm 0.278$ & $0.196 \pm 0.037$ & $0.117 \pm 0.064$ & $0.821 \pm 0.338$ & $0.228 \pm 0.038$ & $0.142 \pm 0.082$ & $0.998 \pm 0.402$ & $0.225 \pm 0.039$ & $0.174 \pm 0.104$ & $1.189 \pm 0.467$ \\
Regret-and-MAD & 1 & 0 & $0.137 \pm 0.038$ & $0.058 \pm 0.028$ & $0.788 \pm 0.404$ & $0.164 \pm 0.036$ & $0.095 \pm 0.032$ & $1.237 \pm 0.750$ & $0.198 \pm 0.045$ & $0.109 \pm 0.027$ & $1.318 \pm 0.756$ & $0.217 \pm 0.053$ & $0.102 \pm 0.019$ & $1.351 \pm 0.652$ & $0.218 \pm 0.055$ & $0.132 \pm 0.056$ & $1.634 \pm 0.774$ \\
Regret-and-MSE & 0 & 0.1 & $0.137 \pm 0.033$ & $0.056 \pm 0.024$ & $0.809 \pm 0.418$ & $0.168 \pm 0.045$ & $0.075 \pm 0.024$ & $0.977 \pm 0.514$ & $0.199 \pm 0.041$ & $0.100 \pm 0.020$ & $1.178 \pm 0.623$ & $0.215 \pm 0.055$ & $0.202 \pm 0.021$ & $1.722 \pm 0.900$ & $0.218 \pm 0.056$ & $0.243 \pm 0.037$ & $1.864 \pm 0.840$ \\
Regret-and-MSE & 0 & 0.5 & $0.137 \pm 0.038$ & $0.049 \pm 0.012$ & $0.674 \pm 0.301$ & $0.166 \pm 0.041$ & $0.065 \pm 0.021$ & $0.836 \pm 0.381$ & $0.195 \pm 0.043$ & $0.103 \pm 0.026$ & $1.109 \pm 0.569$ & $0.214 \pm 0.057$ & $0.150 \pm 0.047$ & $1.334 \pm 0.638$ & $0.219 \pm 0.057$ & $0.160 \pm 0.089$ & $1.327 \pm 0.539$ \\
Regret-and-MSE & 0 & 1 & $0.135 \pm 0.037$ & $0.046 \pm 0.012$ & $0.650 \pm 0.295$ & $0.165 \pm 0.040$ & $0.063 \pm 0.028$ & $0.771 \pm 0.342$ & $0.194 \pm 0.044$ & $0.103 \pm 0.050$ & $0.997 \pm 0.459$ & $0.215 \pm 0.057$ & $0.143 \pm 0.070$ & $1.122 \pm 0.495$ & $0.214 \pm 0.057$ & $0.212 \pm 0.116$ & $1.284 \pm 0.518$ \\
FDFL-FPLG & 0 & -- & $0.136 \pm 0.035$ & $0.047 \pm 0.013$ & $0.606 \pm 0.257$ & $0.160 \pm 0.040$ & $0.061 \pm 0.029$ & $0.726 \pm 0.309$ & $0.197 \pm 0.044$ & $0.093 \pm 0.050$ & $0.891 \pm 0.381$ & $0.214 \pm 0.051$ & $0.131 \pm 0.086$ & $1.066 \pm 0.435$ & $0.217 \pm 0.058$ & $0.158 \pm 0.117$ & $1.272 \pm 0.492$ \\
\midrule
FDFL-Scal & 1 & 0.1 & $0.136 \pm 0.034$ & $0.053 \pm 0.021$ & $0.767 \pm 0.381$ & $0.164 \pm 0.041$ & $0.069 \pm 0.020$ & $0.927 \pm 0.452$ & $0.194 \pm 0.043$ & $0.087 \pm 0.014$ & $1.109 \pm 0.541$ & $0.216 \pm 0.057$ & $0.124 \pm 0.038$ & $1.467 \pm 0.787$ & $0.220 \pm 0.054$ & $0.127 \pm 0.048$ & $1.586 \pm 0.724$ \\
FDFL-Scal & 1 & 0.5 & $0.134 \pm 0.036$ & $0.048 \pm 0.013$ & $0.656 \pm 0.285$ & $0.166 \pm 0.042$ & $0.064 \pm 0.022$ & $0.809 \pm 0.357$ & $0.194 \pm 0.042$ & $0.089 \pm 0.024$ & $1.081 \pm 0.538$ & $0.213 \pm 0.058$ & $0.110 \pm 0.029$ & $1.285 \pm 0.591$ & $0.216 \pm 0.057$ & $0.109 \pm 0.052$ & $1.370 \pm 0.581$ \\
FDFL-Scal & 1 & 1 & $0.137 \pm 0.039$ & $0.046 \pm 0.012$ & $0.630 \pm 0.269$ & $0.164 \pm 0.039$ & $\mathbf{0.059 \pm 0.024}$ & $0.772 \pm 0.339$ & $0.194 \pm 0.044$ & $0.084 \pm 0.031$ & $1.003 \pm 0.468$ & $0.215 \pm 0.054$ & $0.113 \pm 0.057$ & $1.140 \pm 0.497$ & $0.215 \pm 0.057$ & $0.142 \pm 0.083$ & $1.304 \pm 0.535$ \\
FDFL-FPLG & 1 & -- & $0.135 \pm 0.034$ & $0.046 \pm 0.013$ & $0.590 \pm 0.247$ & $0.165 \pm 0.040$ & $0.061 \pm 0.029$ & $0.759 \pm 0.349$ & $0.198 \pm 0.046$ & $\mathbf{0.073 \pm 0.032}$ & $0.960 \pm 0.448$ & $0.217 \pm 0.051$ & $\mathbf{0.089 \pm 0.036}$ & $1.174 \pm 0.531$ & $0.217 \pm 0.055$ & $\mathbf{0.103 \pm 0.041}$ & $1.405 \pm 0.638$ \\
FDFL-PCGrad & -- & -- & $0.136 \pm 0.038$ & $0.042 \pm 0.013$ & $0.555 \pm 0.229$ & $0.162 \pm 0.039$ & $0.072 \pm 0.034$ & $0.686 \pm 0.286$ & $\mathbf{0.190 \pm 0.040}$ & $0.105 \pm 0.059$ & $0.849 \pm 0.370$ & $0.211 \pm 0.051$ & $0.126 \pm 0.064$ & $1.040 \pm 0.437$ & $0.215 \pm 0.057$ & $0.144 \pm 0.080$ & $1.255 \pm 0.528$ \\
FDFL-MGDA & -- & -- & $0.181 \pm 0.066$ & $0.044 \pm 0.016$ & $0.765 \pm 0.366$ & $0.210 \pm 0.081$ & $0.075 \pm 0.029$ & $0.838 \pm 0.410$ & $0.223 \pm 0.049$ & $0.105 \pm 0.043$ & $0.939 \pm 0.443$ & $0.238 \pm 0.045$ & $0.098 \pm 0.048$ & $1.122 \pm 0.504$ & $0.238 \pm 0.060$ & $0.143 \pm 0.064$ & $1.285 \pm 0.564$ \\
FDFL-NashMTL & -- & -- & $\mathbf{0.132 \pm 0.036}$ & $\mathbf{0.041 \pm 0.012}$ & $0.545 \pm 0.233$ & $\mathbf{0.156 \pm 0.037}$ & $0.074 \pm 0.038$ & $0.663 \pm 0.280$ & $0.193 \pm 0.048$ & $0.150 \pm 0.087$ & $0.803 \pm 0.329$ & $\mathbf{0.209 \pm 0.058}$ & $0.230 \pm 0.126$ & $0.963 \pm 0.382$ & $\mathbf{0.213 \pm 0.055}$ & $0.302 \pm 0.155$ & $1.136 \pm 0.427$ \\
\bottomrule
\end{tabular}
}
\end{table}

%% file: Tables/tab_app_hc_ntrain.tex
\begin{table}[htbp]
\centering
\caption{Single resource allocation, full method pool across number of training instances $N$ (MLP-64, $\alpha{=}2$).}
\label{tab:app-hc-ntrain}
\scriptsize
\setlength{\tabcolsep}{3pt}
\resizebox{\textwidth}{!}{%
\begin{tabular}{l cc ccc ccc ccc}
\toprule
 & & & \multicolumn{3}{c}{$N{=}10$} & \multicolumn{3}{c}{$N{=}20$} & \multicolumn{3}{c}{$N{=}50$} \\
\cmidrule(lr){4-6}\cmidrule(lr){7-9}\cmidrule(lr){10-12}
Method & $\lambda$ & $\mu$ & Reg & MAD & MSE & Reg & MAD & MSE & Reg & MAD & MSE \\
\midrule
PTO & 0 & -- & $0.139 \pm 0.003$ & $19.9 \pm 3.3$ & $119.7 \pm 1.3$ & $0.139 \pm 0.003$ & $19.9 \pm 3.0$ & $119.1 \pm 1.2$ & $0.139 \pm 0.003$ & $21.2 \pm 3.0$ & $118.9 \pm 1.9$ \\
SAA & -- & -- & $0.279 \pm 0.002$ & $57.2 \pm 11.2$ & $248.4 \pm 4.7$ & $0.279 \pm 0.002$ & $57.2 \pm 11.4$ & $248.4 \pm 4.8$ & $0.279 \pm 0.003$ & $59.2 \pm 10.8$ & $249.3 \pm 4.5$ \\
WDRO & -- & -- & $0.139 \pm 0.003$ & $20.5 \pm 3.2$ & $\mathbf{118.8 \pm 1.5}$ & $0.138 \pm 0.003$ & $20.6 \pm 2.7$ & $\mathbf{118.2 \pm 1.4}$ & $0.138 \pm 0.002$ & $22.1 \pm 2.5$ & $\mathbf{118.2 \pm 1.5}$ \\
DFL & 0 & 0 & $0.131 \pm 0.003$ & $65.4 \pm 11.0$ & $260.3 \pm 25.5$ & $0.130 \pm 0.004$ & $68.9 \pm 6.3$ & $267.6 \pm 23.1$ & $0.129 \pm 0.003$ & $71.0 \pm 11.0$ & $270.3 \pm 18.5$ \\
\midrule
FPTO & 1 & -- & $0.143 \pm 0.002$ & $\mathbf{18.1 \pm 6.0}$ & $126.2 \pm 7.6$ & $0.142 \pm 0.002$ & $\mathbf{17.9 \pm 6.0}$ & $127.0 \pm 9.6$ & $0.142 \pm 0.002$ & $\mathbf{20.6 \pm 3.5}$ & $122.6 \pm 1.9$ \\
Regret-and-MAD & 1 & 0 & $0.130 \pm 0.004$ & $29.8 \pm 6.8$ & $139.7 \pm 10.7$ & $0.130 \pm 0.003$ & $27.7 \pm 4.4$ & $141.5 \pm 7.4$ & $0.129 \pm 0.003$ & $27.2 \pm 5.5$ & $144.5 \pm 11.7$ \\
Regret-and-MSE & 0 & 0.1 & $0.147 \pm 0.038$ & $41.2 \pm 14.3$ & $179.2 \pm 45.3$ & $0.226 \pm 0.216$ & $44.5 \pm 17.2$ & $189.4 \pm 61.6$ & $0.129 \pm 0.003$ & $40.4 \pm 16.7$ & $166.7 \pm 52.3$ \\
Regret-and-MSE & 0 & 0.5 & $\mathbf{0.130 \pm 0.003}$ & $27.0 \pm 5.7$ & $130.1 \pm 4.7$ & $0.136 \pm 0.016$ & $26.3 \pm 5.5$ & $127.9 \pm 3.3$ & $\mathbf{0.129 \pm 0.003}$ & $27.3 \pm 5.1$ & $129.5 \pm 2.7$ \\
Regret-and-MSE & 0 & 1 & $0.130 \pm 0.003$ & $24.7 \pm 4.4$ & $126.4 \pm 2.9$ & $0.129 \pm 0.003$ & $24.1 \pm 4.0$ & $125.4 \pm 3.4$ & $0.129 \pm 0.003$ & $24.8 \pm 4.4$ & $126.0 \pm 3.0$ \\
FDFL-FPLG & 0 & -- & $0.130 \pm 0.002$ & $24.8 \pm 4.9$ & $127.0 \pm 3.9$ & $\mathbf{0.129 \pm 0.003}$ & $25.8 \pm 5.5$ & $128.3 \pm 2.8$ & $0.130 \pm 0.003$ & $26.4 \pm 4.7$ & $127.8 \pm 3.5$ \\
\midrule
FDFL-Scal & 1 & 0.1 & $0.130 \pm 0.003$ & $26.1 \pm 6.4$ & $137.6 \pm 6.4$ & $0.143 \pm 0.030$ & $25.9 \pm 5.3$ & $138.3 \pm 8.3$ & $0.129 \pm 0.002$ & $26.3 \pm 6.9$ & $145.9 \pm 9.4$ \\
FDFL-Scal & 1 & 0.5 & $0.131 \pm 0.003$ & $23.4 \pm 4.0$ & $129.8 \pm 4.5$ & $0.131 \pm 0.003$ & $23.9 \pm 3.0$ & $129.7 \pm 4.6$ & $0.130 \pm 0.003$ & $25.8 \pm 6.1$ & $135.0 \pm 8.1$ \\
FDFL-Scal & 1 & 1 & $0.131 \pm 0.004$ & $23.6 \pm 5.6$ & $129.4 \pm 7.4$ & $0.130 \pm 0.003$ & $22.4 \pm 3.0$ & $126.1 \pm 2.4$ & $0.130 \pm 0.003$ & $25.3 \pm 3.4$ & $126.8 \pm 4.6$ \\
FDFL-FPLG & 1 & -- & $0.131 \pm 0.003$ & $24.0 \pm 3.4$ & $130.1 \pm 7.5$ & $0.130 \pm 0.003$ & $24.5 \pm 3.6$ & $131.1 \pm 7.3$ & $0.130 \pm 0.002$ & $24.6 \pm 4.1$ & $129.3 \pm 4.0$ \\
FDFL-PCGrad & -- & -- & $0.131 \pm 0.004$ & $24.3 \pm 4.5$ & $129.8 \pm 3.9$ & $0.131 \pm 0.003$ & $22.5 \pm 4.6$ & $132.1 \pm 3.9$ & $0.130 \pm 0.003$ & $24.8 \pm 4.7$ & $136.0 \pm 10.1$ \\
FDFL-MGDA & -- & -- & $0.133 \pm 0.002$ & $26.1 \pm 6.3$ & $131.9 \pm 6.7$ & $0.134 \pm 0.005$ & $22.8 \pm 3.7$ & $131.9 \pm 9.1$ & $0.134 \pm 0.002$ & $28.1 \pm 5.1$ & $134.2 \pm 5.1$ \\
FDFL-NashMTL & -- & -- & $0.130 \pm 0.003$ & $24.7 \pm 2.6$ & $126.7 \pm 3.0$ & $0.129 \pm 0.003$ & $24.0 \pm 3.5$ & $126.8 \pm 3.9$ & $0.129 \pm 0.003$ & $25.8 \pm 4.4$ & $126.4 \pm 2.8$ \\
\bottomrule
\end{tabular}
}
\end{table}

%% file: Tables/tab_app_md_groups.tex
\begin{table}[htbp]
\centering
\caption{Multiple resource allocation, full method pool across the number of groups $K$ (MLP-64, $\alpha{=}2$, imbalance $0.6$; MAD values are not comparable across $K$).}
\label{tab:app-md-groups}
\scriptsize
\setlength{\tabcolsep}{3pt}
\begin{tabular}{l cc ccc ccc}
\toprule
 & & & \multicolumn{3}{c}{$K{=}2$} & \multicolumn{3}{c}{$K{=}4$} \\
\cmidrule(lr){4-6}\cmidrule(lr){7-9}
Method & $\lambda$ & $\mu$ & Reg & MAD & MSE & Reg & MAD & MSE \\
\midrule
PTO & 0 & -- & $0.215 \pm 0.046$ & $0.228 \pm 0.117$ & $\mathbf{0.953 \pm 0.368}$ & $0.217 \pm 0.045$ & $0.194 \pm 0.083$ & $\mathbf{0.929 \pm 0.364}$ \\
SAA & -- & -- & $0.656 \pm 0.080$ & $0.260 \pm 0.072$ & $2.772 \pm 1.169$ & $0.662 \pm 0.076$ & $0.404 \pm 0.119$ & $2.743 \pm 1.166$ \\
WDRO & -- & -- & $0.215 \pm 0.049$ & $0.194 \pm 0.115$ & $0.961 \pm 0.375$ & $0.222 \pm 0.052$ & $0.171 \pm 0.080$ & $0.935 \pm 0.374$ \\
DFL & 0 & 0 & $0.216 \pm 0.058$ & $0.176 \pm 0.035$ & $1.567 \pm 0.763$ & $0.221 \pm 0.051$ & $0.228 \pm 0.104$ & $1.440 \pm 0.726$ \\
\midrule
FPTO & 1 & -- & $0.228 \pm 0.038$ & $0.142 \pm 0.082$ & $0.998 \pm 0.402$ & $0.222 \pm 0.044$ & $0.177 \pm 0.076$ & $0.942 \pm 0.375$ \\
Regret-and-MAD & 1 & 0 & $0.217 \pm 0.053$ & $0.102 \pm 0.019$ & $1.351 \pm 0.652$ & $0.218 \pm 0.056$ & $0.201 \pm 0.092$ & $1.317 \pm 0.623$ \\
Regret-and-MSE & 0 & 0.1 & $0.215 \pm 0.055$ & $0.202 \pm 0.021$ & $1.722 \pm 0.900$ & -- & -- & -- \\
Regret-and-MSE & 0 & 0.5 & $0.214 \pm 0.057$ & $0.150 \pm 0.047$ & $1.334 \pm 0.638$ & -- & -- & -- \\
Regret-and-MSE & 0 & 1 & $0.215 \pm 0.057$ & $0.143 \pm 0.070$ & $1.122 \pm 0.495$ & $0.219 \pm 0.053$ & $0.191 \pm 0.097$ & $1.124 \pm 0.512$ \\
FDFL-FPLG & 0 & -- & $0.214 \pm 0.051$ & $0.131 \pm 0.086$ & $1.066 \pm 0.435$ & $\mathbf{0.213 \pm 0.048}$ & $0.179 \pm 0.081$ & $1.025 \pm 0.427$ \\
\midrule
FDFL-Scal & 1 & 0.1 & $0.216 \pm 0.057$ & $0.124 \pm 0.038$ & $1.467 \pm 0.787$ & -- & -- & -- \\
FDFL-Scal & 1 & 0.5 & $0.213 \pm 0.058$ & $0.110 \pm 0.029$ & $1.285 \pm 0.591$ & -- & -- & -- \\
FDFL-Scal & 1 & 1 & $0.215 \pm 0.054$ & $0.113 \pm 0.057$ & $1.140 \pm 0.497$ & $0.218 \pm 0.047$ & $0.185 \pm 0.092$ & $1.135 \pm 0.524$ \\
FDFL-FPLG & 1 & -- & $0.217 \pm 0.051$ & $\mathbf{0.089 \pm 0.036}$ & $1.174 \pm 0.531$ & $0.217 \pm 0.050$ & $\mathbf{0.166 \pm 0.081}$ & $1.112 \pm 0.515$ \\
FDFL-PCGrad & -- & -- & $0.211 \pm 0.051$ & $0.126 \pm 0.064$ & $1.040 \pm 0.437$ & $0.220 \pm 0.043$ & $0.174 \pm 0.076$ & $0.981 \pm 0.396$ \\
FDFL-MGDA & -- & -- & $0.238 \pm 0.045$ & $0.098 \pm 0.048$ & $1.122 \pm 0.504$ & $0.249 \pm 0.072$ & $0.177 \pm 0.078$ & $1.112 \pm 0.465$ \\
FDFL-NashMTL & -- & -- & $\mathbf{0.209 \pm 0.058}$ & $0.230 \pm 0.126$ & $0.963 \pm 0.382$ & $0.214 \pm 0.048$ & $0.197 \pm 0.087$ & $0.937 \pm 0.377$ \\
\bottomrule
\end{tabular}
\end{table}

%% file: Tables/tab_app_md_imb_a05.tex
\begin{table}[htbp]
\centering
\caption{Multiple resource allocation, full method pool across group imbalance $\ell$ at $\alpha{=}0.5$ (MLP-64, $K{=}2$).}
\label{tab:app-md-imb-a05}
\scriptsize
\setlength{\tabcolsep}{3pt}
\resizebox{\textwidth}{!}{%
\begin{tabular}{l cc ccc ccc ccc ccc ccc}
\toprule
 & & & \multicolumn{3}{c}{$\ell{=}0$} & \multicolumn{3}{c}{$\ell{=}0.2$} & \multicolumn{3}{c}{$\ell{=}0.4$} & \multicolumn{3}{c}{$\ell{=}0.6$} & \multicolumn{3}{c}{$\ell{=}0.8$} \\
\cmidrule(lr){4-6}\cmidrule(lr){7-9}\cmidrule(lr){10-12}\cmidrule(lr){13-15}\cmidrule(lr){16-18}
Method & $\lambda$ & $\mu$ & Reg & MAD & MSE & Reg & MAD & MSE & Reg & MAD & MSE & Reg & MAD & MSE & Reg & MAD & MSE \\
\midrule
PTO & 0 & -- & $0.013 \pm 0.001$ & $0.043 \pm 0.012$ & $0.567 \pm 0.233$ & $0.014 \pm 0.001$ & $0.075 \pm 0.038$ & $0.677 \pm 0.276$ & $0.011 \pm 0.003$ & $0.144 \pm 0.078$ & $0.810 \pm 0.322$ & $\mathbf{0.016 \pm 0.004}$ & $0.222 \pm 0.118$ & $\mathbf{0.960 \pm 0.369}$ & $0.023 \pm 0.006$ & $0.303 \pm 0.157$ & $\mathbf{1.129 \pm 0.418}$ \\
SAA & -- & -- & $0.098 \pm 0.019$ & $0.187 \pm 0.085$ & $2.408 \pm 1.055$ & $0.089 \pm 0.013$ & $0.207 \pm 0.079$ & $2.506 \pm 1.088$ & $0.047 \pm 0.012$ & $0.233 \pm 0.075$ & $2.628 \pm 1.127$ & $0.053 \pm 0.014$ & $0.260 \pm 0.072$ & $2.772 \pm 1.169$ & $0.065 \pm 0.015$ & $0.286 \pm 0.072$ & $2.936 \pm 1.215$ \\
WDRO & -- & -- & $\mathbf{0.012 \pm 0.001}$ & $0.043 \pm 0.014$ & $0.550 \pm 0.233$ & $0.014 \pm 0.002$ & $0.070 \pm 0.042$ & $0.664 \pm 0.278$ & $\mathbf{0.011 \pm 0.003}$ & $0.127 \pm 0.077$ & $\mathbf{0.802 \pm 0.326}$ & $0.016 \pm 0.004$ & $0.194 \pm 0.115$ & $0.961 \pm 0.375$ & $0.023 \pm 0.006$ & $0.263 \pm 0.151$ & $1.140 \pm 0.430$ \\
DFL & 0 & 0 & $0.013 \pm 0.001$ & $0.077 \pm 0.032$ & $0.945 \pm 0.526$ & $0.014 \pm 0.002$ & $0.106 \pm 0.055$ & $1.094 \pm 0.658$ & $0.011 \pm 0.003$ & $0.116 \pm 0.050$ & $1.410 \pm 0.893$ & $0.021 \pm 0.007$ & $0.562 \pm 0.247$ & $3.088 \pm 1.453$ & $0.030 \pm 0.009$ & $0.904 \pm 0.297$ & $4.004 \pm 1.644$ \\
\midrule
FPTO & 1 & -- & $0.013 \pm 0.001$ & $0.043 \pm 0.013$ & $0.569 \pm 0.232$ & $0.014 \pm 0.001$ & $0.069 \pm 0.035$ & $0.685 \pm 0.282$ & $0.011 \pm 0.003$ & $0.105 \pm 0.056$ & $0.840 \pm 0.342$ & $0.017 \pm 0.004$ & $0.135 \pm 0.076$ & $1.017 \pm 0.405$ & $0.024 \pm 0.006$ & $0.158 \pm 0.092$ & $1.213 \pm 0.473$ \\
Regret-and-MAD & 1 & 0 & $0.013 \pm 0.002$ & $0.043 \pm 0.011$ & $0.597 \pm 0.259$ & $0.014 \pm 0.002$ & $\mathbf{0.059 \pm 0.023}$ & $0.773 \pm 0.353$ & $0.012 \pm 0.003$ & $\mathbf{0.074 \pm 0.029}$ & $1.001 \pm 0.471$ & $0.021 \pm 0.007$ & $\mathbf{0.090 \pm 0.029}$ & $1.368 \pm 0.593$ & $0.027 \pm 0.008$ & $0.121 \pm 0.034$ & $1.668 \pm 0.712$ \\
Regret-and-MSE & 0 & 0.1 & $0.013 \pm 0.001$ & $0.055 \pm 0.016$ & $0.663 \pm 0.299$ & $0.014 \pm 0.002$ & $0.072 \pm 0.035$ & $0.805 \pm 0.387$ & $0.012 \pm 0.004$ & $0.095 \pm 0.056$ & $1.008 \pm 0.469$ & $0.018 \pm 0.006$ & $0.193 \pm 0.062$ & $1.619 \pm 0.738$ & $0.024 \pm 0.006$ & $0.347 \pm 0.119$ & $2.063 \pm 0.922$ \\
Regret-and-MSE & 0 & 0.5 & $0.012 \pm 0.002$ & $0.042 \pm 0.011$ & $0.565 \pm 0.241$ & $0.014 \pm 0.002$ & $0.071 \pm 0.037$ & $0.684 \pm 0.291$ & $0.011 \pm 0.003$ & $0.131 \pm 0.074$ & $0.858 \pm 0.343$ & $0.019 \pm 0.006$ & $0.147 \pm 0.088$ & $1.183 \pm 0.440$ & $0.026 \pm 0.007$ & $0.149 \pm 0.058$ & $1.452 \pm 0.502$ \\
Regret-and-MSE & 0 & 1 & $0.013 \pm 0.002$ & $0.044 \pm 0.013$ & $0.585 \pm 0.245$ & $0.014 \pm 0.002$ & $0.074 \pm 0.039$ & $0.689 \pm 0.285$ & $0.012 \pm 0.004$ & $0.134 \pm 0.081$ & $0.854 \pm 0.343$ & $0.018 \pm 0.005$ & $0.166 \pm 0.095$ & $1.079 \pm 0.403$ & $0.024 \pm 0.006$ & $0.192 \pm 0.121$ & $1.237 \pm 0.448$ \\
FDFL-FPLG & 0 & -- & $0.013 \pm 0.002$ & $0.053 \pm 0.017$ & $0.640 \pm 0.289$ & $0.014 \pm 0.002$ & $0.069 \pm 0.038$ & $0.819 \pm 0.445$ & $0.012 \pm 0.004$ & $0.087 \pm 0.031$ & $1.146 \pm 0.657$ & $0.020 \pm 0.006$ & $0.308 \pm 0.221$ & $1.972 \pm 0.829$ & $0.025 \pm 0.007$ & $0.333 \pm 0.187$ & $1.985 \pm 0.803$ \\
\midrule
FDFL-Scal & 1 & 0.1 & $0.013 \pm 0.002$ & $0.047 \pm 0.012$ & $0.618 \pm 0.263$ & $0.014 \pm 0.002$ & $0.063 \pm 0.026$ & $0.758 \pm 0.339$ & $0.012 \pm 0.004$ & $0.082 \pm 0.034$ & $0.977 \pm 0.438$ & $0.019 \pm 0.006$ & $0.092 \pm 0.030$ & $1.266 \pm 0.541$ & $0.024 \pm 0.006$ & $\mathbf{0.104 \pm 0.025}$ & $1.467 \pm 0.614$ \\
FDFL-Scal & 1 & 0.5 & $0.013 \pm 0.001$ & $0.042 \pm 0.011$ & $0.562 \pm 0.238$ & $0.014 \pm 0.001$ & $0.063 \pm 0.029$ & $0.696 \pm 0.298$ & $0.011 \pm 0.003$ & $0.087 \pm 0.049$ & $0.885 \pm 0.364$ & $0.019 \pm 0.006$ & $0.116 \pm 0.058$ & $1.216 \pm 0.466$ & $0.026 \pm 0.008$ & $0.119 \pm 0.041$ & $1.458 \pm 0.565$ \\
FDFL-Scal & 1 & 1 & $0.013 \pm 0.002$ & $0.044 \pm 0.013$ & $0.582 \pm 0.246$ & $0.014 \pm 0.002$ & $0.070 \pm 0.037$ & $0.694 \pm 0.291$ & $0.012 \pm 0.004$ & $0.103 \pm 0.062$ & $0.871 \pm 0.357$ & $0.017 \pm 0.006$ & $0.119 \pm 0.064$ & $1.106 \pm 0.436$ & $0.024 \pm 0.006$ & $0.126 \pm 0.069$ & $1.283 \pm 0.501$ \\
FDFL-FPLG & 1 & -- & $0.013 \pm 0.001$ & $0.044 \pm 0.011$ & $0.599 \pm 0.255$ & $0.014 \pm 0.002$ & $0.064 \pm 0.031$ & $0.795 \pm 0.387$ & $0.012 \pm 0.004$ & $0.074 \pm 0.030$ & $1.036 \pm 0.518$ & $0.019 \pm 0.006$ & $0.099 \pm 0.034$ & $1.384 \pm 0.646$ & $0.025 \pm 0.006$ & $0.125 \pm 0.045$ & $1.574 \pm 0.718$ \\
FDFL-PCGrad & -- & -- & $0.012 \pm 0.002$ & $0.042 \pm 0.013$ & $0.553 \pm 0.233$ & $0.014 \pm 0.002$ & $0.072 \pm 0.038$ & $0.686 \pm 0.295$ & $0.012 \pm 0.004$ & $0.089 \pm 0.052$ & $0.891 \pm 0.377$ & $0.019 \pm 0.005$ & $0.129 \pm 0.072$ & $1.174 \pm 0.491$ & $0.026 \pm 0.009$ & $0.154 \pm 0.095$ & $1.355 \pm 0.569$ \\
FDFL-MGDA & -- & -- & $0.013 \pm 0.002$ & $0.042 \pm 0.017$ & $0.635 \pm 0.287$ & $0.015 \pm 0.002$ & $0.070 \pm 0.036$ & $0.715 \pm 0.313$ & $0.012 \pm 0.004$ & $0.098 \pm 0.049$ & $0.905 \pm 0.386$ & $0.018 \pm 0.006$ & $0.119 \pm 0.072$ & $1.091 \pm 0.452$ & $0.024 \pm 0.007$ & $0.146 \pm 0.082$ & $1.273 \pm 0.516$ \\
FDFL-NashMTL & -- & -- & $0.012 \pm 0.001$ & $\mathbf{0.041 \pm 0.012}$ & $\mathbf{0.548 \pm 0.235}$ & $\mathbf{0.014 \pm 0.002}$ & $0.075 \pm 0.039$ & $\mathbf{0.661 \pm 0.282}$ & $0.011 \pm 0.003$ & $0.141 \pm 0.075$ & $0.814 \pm 0.335$ & $0.017 \pm 0.005$ & $0.212 \pm 0.111$ & $0.999 \pm 0.376$ & $\mathbf{0.023 \pm 0.006}$ & $0.288 \pm 0.151$ & $1.154 \pm 0.432$ \\
\bottomrule
\end{tabular}
}
\end{table}

%% file: sample-base.bib
@String{Computing = "Computing" }

@String{Computer = "{IEEE} Computer" }

@String{Chelsea = "Chelsea" }

@String{Springer = "Springer-Verlag" }

@inproceedings{wilder2019melding,
  title={Melding the data-decisions pipeline: Decision-focused learning for combinatorial optimization},
  author={Wilder, Bryan and Dilkina, Bistra and Tambe, Milind},
  booktitle={Proceedings of the AAAI Conference on Artificial Intelligence},
  volume={33},
  pages={1658--1665},
  year={2019}
}

@article{donti2017task,
  title={Task-based end-to-end model learning in stochastic optimization},
  author={Donti, Priya and Amos, Brandon and Kolter, J Zico},
  journal={Advances in neural information processing systems},
  volume={30},
  year={2017}
}

@inproceedings{amos2017optnet,
  title={Optnet: Differentiable optimization as a layer in neural networks},
  author={Amos, Brandon and Kolter, J Zico},
  booktitle={International conference on machine learning},
  pages={136--145},
  year={2017},
  organization={PMLR}
}

@article{kuppler2022fair,
  title={From fair predictions to just decisions? Conceptualizing algorithmic fairness and distributive justice in the context of data-driven decision-making},
  author={Kuppler, Matthias and Kern, Christoph and Bach, Ruben L and Kreuter, Frauke},
  journal={Frontiers in sociology},
  volume={7},
  pages={883999},
  year={2022},
  publisher={Frontiers Media SA}
}

@inproceedings{kotary2022end,
  title={End-to-end learning for fair ranking systems},
  author={Kotary, James and Fioretto, Ferdinando and Van Hentenryck, Pascal and Zhu, Ziwei},
  booktitle={Proceedings of the ACM Web Conference 2022},
  pages={3520--3530},
  year={2022}
}

@article{mehrabi2021survey,
  title={A survey on bias and fairness in machine learning},
  author={Mehrabi, Ninareh and Morstatter, Fred and Saxena, Nripsuta and Lerman, Kristina and Galstyan, Aram},
  journal={ACM computing surveys (CSUR)},
  volume={54},
  number={6},
  pages={1--35},
  year={2021},
  publisher={ACM New York, NY, USA}
}

@Article{Chen2023,
  author={Violet Xinying Chen and J. N. Hooker},
  title={{A guide to formulating fairness in an optimization model}},
  journal={Annals of Operations Research},
  year=2023,
  volume={326},
  number={1},
  pages={581-619},
  month={July},
}

@article{spo2020,
  title={Smart “predict, then optimize”},
  author={Elmachtoub, Adam N and Grigas, Paul},
  journal={Management Science},
  volume={68},
  number={1},
  pages={9--26},
  year={2022},
  publisher={INFORMS}
}

@article{mo2000fair,
  title={Fair end-to-end window-based congestion control},
  author={Mo, Jeonghoon and Walrand, Jean},
  journal={IEEE/ACM Transactions on networking},
  volume={8},
  number={5},
  pages={556--567},
  year={2000},
  publisher={IEEE}
}

@article{elmachtoub2023estimate,
  title={Estimate-Then-Optimize versus Integrated-Estimation-Optimization versus Sample Average Approximation: A Stochastic Dominance Perspective},
  author={Elmachtoub, Adam N and Lam, Henry and Zhang, Haofeng and Zhao, Yunfan},
  journal={arXiv preprint arXiv:2304.06833},
  year={2023}
}

@article{huang2024decision,
  title={Decision-focused learning with directional gradients},
  author={Huang, Michael and Gupta, Vishal},
  journal={Advances in Neural Information Processing Systems},
  volume={37},
  pages={79194--79220},
  year={2024}
}

@article{scantamburlo2024prediction,
  title={On prediction-modelers and decision-makers: why fairness requires more than a fair prediction model},
  author={Scantamburlo, Teresa and Baumann, Joachim and Heitz, Christoph},
  journal={AI \& SOCIETY},
  pages={1--17},
  year={2024},
  publisher={Springer}
}

@inproceedings{kotaryfoldopt2023,
author = {Kotary, James and Dinh, My H and Fioretto, Ferdinando},
title = {Backpropagation of unrolled solvers with folded optimization},
year = {2023},
isbn = {978-1-956792-03-4},
booktitle = {Proceedings of the Thirty-Second International Joint Conference on Artificial Intelligence},
articleno = {218},
numpages = {8},
location = {Macao, P.R.China},
series = {IJCAI '23}
}

@article{corbett2023measure,
  title={The measure and mismeasure of fairness},
  author={Corbett-Davies, Sam and Gaebler, Johann D and Nilforoshan, Hamed and Shroff, Ravi and Goel, Sharad},
  journal={The Journal of Machine Learning Research},
  volume={24},
  number={1},
  pages={14730--14846},
  year={2023},
  publisher={JMLRORG}
}

@article{shah2022decision,
  title={Decision-focused learning without decision-making: Learning locally optimized decision losses},
  author={Shah, Sanket and Wang, Kai and Wilder, Bryan and Perrault, Andrew and Tambe, Milind},
  journal={Advances in Neural Information Processing Systems},
  volume={35},
  pages={1320--1332},
  year={2022}
}

@article{agrawal2019differentiable,
  title={Differentiable convex optimization layers},
  author={Agrawal, Akshay and Amos, Brandon and Barratt, Shane and Boyd, Stephen and Diamond, Steven and Kolter, J Zico},
  journal={Advances in neural information processing systems},
  volume={32},
  year={2019}
}

@article{zharmagambetov2023landscape,
  title={Landscape surrogate: Learning decision losses for mathematical optimization under partial information},
  author={Zharmagambetov, Arman and Amos, Brandon and Ferber, Aaron and Huang, Taoan and Dilkina, Bistra and Tian, Yuandong},
  journal={Advances in Neural Information Processing Systems},
  volume={36},
  pages={27332--27350},
  year={2023}
}

@article{
doi:10.1126/science.aax2342,
author = {Ziad Obermeyer  and Brian Powers  and Christine Vogeli  and Sendhil Mullainathan },
title = {Dissecting racial bias in an algorithm used to manage the health of populations},
journal = {Science},
volume = {366},
number = {6464},
pages = {447-453},
year = {2019},
eprint = {https://www.science.org/doi/pdf/10.1126/science.aax2342}}

@article{berk2017convex,
  title={A convex framework for fair regression},
  author={Berk, Richard and Heidari, Hoda and Jabbari, Shahin and Joseph, Matthew and Kearns, Michael and Morgenstern, Jamie and Neel, Seth and Roth, Aaron},
  journal={arXiv preprint arXiv:1706.02409},
  year={2017}
}

@inproceedings{agarwal2019fair,
  title={Fair regression: Quantitative definitions and reduction-based algorithms},
  author={Agarwal, Alekh and Dud{\'\i}k, Miroslav and Wu, Zhiwei Steven},
  booktitle={International Conference on Machine Learning},
  pages={120--129},
  year={2019},
  organization={PMLR}
}

@inproceedings{liu2018delayed,
  title={Delayed impact of fair machine learning},
  author={Liu, Lydia T and Dean, Sarah and Rolf, Esther and Simchowitz, Max and Hardt, Moritz},
  booktitle={International Conference on Machine Learning},
  pages={3150--3158},
  year={2018},
  organization={PMLR}
}

@article{chohlas2024learning,
  title={Learning to Be Fair: A Consequentialist Approach to Equitable Decision Making},
  author={Chohlas-Wood, Alex and Coots, Madison and Zhu, Henry and Brunskill, Emma and Goel, Sharad},
  journal={Management Science},
  year={2024},
  publisher={INFORMS}
}

@inproceedings{dinh2024learning,
  title={Learning Fair Ranking Policies via Differentiable Optimization of Ordered Weighted Averages},
  author={Dinh, My H and Kotary, James and Fioretto, Ferdinando},
  booktitle={The 2024 ACM Conference on Fairness, Accountability, and Transparency},
  pages={2508--2517},
  year={2024}
}

@inproceedings{dinh2024end,
  title={End-to-End Learning for Fair Multiobjective Optimization Under Uncertainty},
  author={Dinh, My H and Kotary, James and Fioretto, Ferdinando},
  booktitle={Uncertainty in Artificial Intelligence},
  pages={1129--1145},
  year={2024},
  organization={PMLR}
}

@article{bertsimas2020predictive,
  title={From predictive to prescriptive analytics},
  author={Bertsimas, Dimitris and Kallus, Nathan},
  journal={Management Science},
  volume={66},
  number={3},
  pages={1025--1044},
  year={2020},
  publisher={INFORMS}
}

@article{qi2023practical,
  title={A practical end-to-end inventory management model with deep learning},
  author={Qi, Meng and Shi, Yuanyuan and Qi, Yongzhi and Ma, Chenxin and Yuan, Rong and Wu, Di and Shen, Zuo-Jun},
  journal={Management Science},
  volume={69},
  number={2},
  pages={759--773},
  year={2023},
  publisher={INFORMS}
}

@article{paulus2020predictably,
  title={Predictably unequal: understanding and addressing concerns that algorithmic clinical prediction may increase health disparities},
  author={Paulus, Jessica K and Kent, David M},
  journal={NPJ digital medicine},
  volume={3},
  number={1},
  pages={99},
  year={2020},
  publisher={Nature Publishing Group UK London}
}

@article{chung2022decision,
  title={Decision-aware learning for optimizing health supply chains},
  author={Chung, Tsai-Hsuan and Rostami, Vahid and Bastani, Hamsa and Bastani, Osbert},
  journal={arXiv preprint arXiv:2211.08507},
  year={2022}
}

@inproceedings{chouldechova2018case,
  title={A case study of algorithm-assisted decision making in child maltreatment hotline screening decisions},
  author={Chouldechova, Alexandra and Benavides-Prado, Diana and Fialko, Oleksandr and Vaithianathan, Rhema},
  booktitle={Conference on fairness, accountability and transparency},
  pages={134--148},
  year={2018},
  organization={PMLR}
}

@article{mandi2024decision,
  title={Decision-focused learning: Foundations, state of the art, benchmark and future opportunities},
  author={Mandi, Jayanta and Kotary, James and Berden, Senne and Mulamba, Maxime and Bucarey, Victor and Guns, Tias and Fioretto, Ferdinando},
  journal={Journal of Artificial Intelligence Research},
  volume={80},
  pages={1623--1701},
  year={2024}
}

@article{sadana2025survey,
  title={A survey of contextual optimization methods for decision-making under uncertainty},
  author={Sadana, Utsav and Chenreddy, Abhilash and Delage, Erick and Forel, Alexandre and Frejinger, Emma and Vidal, Thibaut},
  journal={European Journal of Operational Research},
  volume={320},
  number={2},
  pages={271--289},
  year={2025},
  publisher={Elsevier}
}

@inproceedings{yu2020gradient,
  title={Gradient surgery for multi-task learning},
  author={Yu, Tianhe and Kumar, Saurabh and Gupta, Abhishek and Levine, Sergey and Hausman, Karol and Finn, Chelsea},
  booktitle={Advances in Neural Information Processing Systems},
  volume={33},
  pages={5824--5836},
  year={2020}
}

@inproceedings{sener2018multi,
  title={Multi-task learning as multi-objective optimization},
  author={Sener, Ozan and Koltun, Vladlen},
  booktitle={Advances in Neural Information Processing Systems},
  volume={31},
  year={2018}
}

@inproceedings{navon2022multi,
  title={Multi-task learning as a bargaining game},
  author={Navon, Aviv and Shamsian, Aviv and Achituve, Idan and Maron, Haggai and Kawaguchi, Kenji and Chechik, Gal and Fetaya, Ethan},
  booktitle={International Conference on Machine Learning},
  pages={16428--16446},
  year={2022},
  organization={PMLR}
}

@article{gao2024wasserstein,
  title={Wasserstein distributionally robust optimization and variation regularization},
  author={Gao, Rui and Chen, Xi and Kleywegt, Anton J.},
  journal={Operations Research},
  volume={72},
  number={3},
  pages={1177--1191},
  year={2024},
  publisher={INFORMS}
}

@inproceedings{sinha2018certifying,
  title={Certifying some distributional robustness with principled adversarial training},
  author={Sinha, Aman and Namkoong, Hongseok and Duchi, John},
  booktitle={International Conference on Learning Representations},
  year={2018}
}

@article{jeon2025plg,
  title={Prediction loss guided decision-focused learning},
  author={Jeon, Haeun and Bae, Hyunglip and Kim, Chanyeong and Lee, Yongjae and Kim, Woo Chang},
  journal={arXiv preprint arXiv:2509.08359},
  year={2025}
}

@book{mohri2018foundations,
  author    = {Mohri, Mehryar and Rostamizadeh, Afshin and Talwalkar, Ameet},
  title     = {Foundations of Machine Learning},
  edition   = {2nd},
  publisher = {MIT Press},
  year      = {2018}
}

@article{kannan2024residuals,
  title={Residuals-based distributionally robust optimization with covariate information},
  author={Kannan, Rohit and Bayraksan, G{\"u}zin and Luedtke, James R},
  journal={Mathematical Programming},
  volume={207},
  number={1-2},
  pages={369--425},
  year={2024},
  publisher={Springer Berlin Heidelberg Berlin/Heidelberg}
}

@inproceedings{elmachtoub2025dissecting,
  title={Dissecting the Impact of Model Misspecification in Data-Driven Optimization},
  author={Elmachtoub, Adam N and Lam, Henry and Lan, Haixiang and Zhang, Haofeng},
  booktitle={International Conference on Artificial Intelligence and Statistics},
  pages={1594--1602},
  year={2025},
  organization={PMLR}
}

@inproceedings{kendall2018multi,
  title={Multi-task learning using uncertainty to weigh losses for scene geometry and semantics},
  author={Kendall, Alex and Gal, Yarin and Cipolla, Roberto},
  booktitle={Proceedings of the IEEE conference on computer vision and pattern recognition},
  pages={7482--7491},
  year={2018}
}

@article{hu2023revisiting,
  title={Revisiting scalarization in multi-task learning: A theoretical perspective},
  author={Hu, Yuzheng and Xian, Ruicheng and Wu, Qilong and Fan, Qiuling and Yin, Lang and Zhao, Han},
  journal={Advances in Neural Information Processing Systems},
  volume={36},
  pages={48510--48533},
  year={2023}
}

@article{desideri2012multiple,
  title={Multiple-gradient descent algorithm (MGDA) for multiobjective optimization},
  author={D{\'e}sid{\'e}ri, Jean-Antoine},
  journal={Comptes Rendus. Math{\'e}matique},
  volume={350},
  number={5-6},
  pages={313--318},
  year={2012}
}

@article{kurin2022defense,
  title={In defense of the unitary scalarization for deep multi-task learning},
  author={Kurin, Vitaly and De Palma, Alessandro and Kostrikov, Ilya and Whiteson, Shimon and Mudigonda, Pawan K},
  journal={Advances in Neural Information Processing Systems},
  volume={35},
  pages={12169--12183},
  year={2022}
}

@article{xin2022current,
  title={Do current multi-task optimization methods in deep learning even help?},
  author={Xin, Derrick and Ghorbani, Behrooz and Gilmer, Justin and Garg, Ankush and Firat, Orhan},
  journal={Advances in neural information processing systems},
  volume={35},
  pages={13597--13609},
  year={2022}
}

@article{chen2024three,
  title={Three-way trade-off in multi-objective learning: Optimization, generalization and conflict-avoidance},
  author={Chen, Lisha and Fernando, Heshan and Ying, Yiming and Chen, Tianyi},
  journal={Journal of Machine Learning Research},
  volume={25},
  number={193},
  pages={1--53},
  year={2024}
}

@article{agrawal2019cone,
  author  = {Agrawal, Akshay and Barratt, Shane and Boyd, Stephen and Busseti, Enzo and Moursi, Walaa M.},
  title   = {Differentiating through a cone program},
  journal = {Journal of Applied and Numerical Optimization},
  volume  = {1},
  number  = {2},
  pages   = {107--115},
  year    = {2019}
}

@article{busseti2019solution,
  author  = {Busseti, Enzo and Moursi, Walaa M. and Boyd, Stephen},
  title   = {Solution refinement at regular points of conic problems},
  journal = {Computational Optimization and Applications},
  volume  = {74},
  pages   = {627--643},
  year    = {2019}
}

@misc{sen2018gentle,
  author = {Sen, Bodhisattva},
  title  = {A Gentle Introduction to Empirical Process Theory and Applications},
  year   = {2018},
  note   = {Lecture notes, Columbia University.}
}
